\documentclass[11pt]{article}
\usepackage[utf8]{inputenc}
\usepackage[margin=1in]{geometry}
\usepackage{graphicx}
\usepackage{gensymb}
\usepackage{amsfonts}
\usepackage{amsmath}
\usepackage{bigints}
\usepackage{mathtools}
\usepackage{enumerate}
\usepackage[shortlabels]{enumitem}
\usepackage{multirow}
\usepackage{float}
\usepackage{algorithm}
\usepackage{algpseudocode}

\usepackage{setspace}
\usepackage{hyperref}
\usepackage{tikz}
\usepackage{amsmath}
\DeclareMathOperator*{\argmin}{argmin}
\AtBeginEnvironment{algorithm}{
  \singlespacing
}
\usepackage{bbm}
\usepackage{subfigure}
\usepackage{array}
\newcolumntype{P}[1]{>{\centering\arraybackslash}m{#1}}
\usepackage[usestackEOL]{stackengine}
\newcommand{\RomanNumeral}[1]{\MakeUppercase{\romannumeral #1}}
\usepackage{amsthm}
\theoremstyle{definition}
\newtheorem{remark}{\textbf{Remark}}
\newtheorem{theorem}{Theorem}[section]
\newtheorem{assumption}{Assumption}

\newtheorem*{problem}{Problem}
\makeatletter
\def\th@plain{
  \thm@notefont{}
  \itshape 
}
\def\th@definition{
  \thm@notefont{}
  \normalfont 
}
\makeatother
\newcounter{proof}

\usepackage{placeins}

\title{Regularized Random Fourier Features and Finite Element Reconstruction for Operator Learning in Sobolev Space}
\author{Xinyue Yu and Hayden Schaeffer}
\date{}

\begin{document}

\maketitle

\vspace{-1cm}
\begin{center}
    Department of Mathematics\\
    University of California, Los Angeles\\
    Los Angeles, California 90095, USA
\end{center}

\begin{abstract}
Operator learning is a data-driven approximation of mappings between infinite-dimensional function spaces, such as the solution operators of partial differential equations. Kernel-based operator learning can offer accurate, theoretically justified approximations that require less training than standard methods. However, they can become computationally prohibitive for large training sets and can be sensitive to noise. We propose a regularized random Fourier feature (RRFF) approach, coupled with a finite element reconstruction map (RRFF-FEM), for learning operators from noisy data. The method uses random features drawn from multivariate Student's $t$ distributions, together with frequency-weighted Tikhonov regularization that suppresses high-frequency noise. We establish high-probability bounds on the extreme singular values of the associated random feature matrix and show that when the number of features $N$ scales like $m \log m$ with the number of training samples $m$, the system is well-conditioned, which yields estimation and generalization guarantees. Detailed numerical experiments on benchmark PDE problems, including advection, Burgers', Darcy flow, Helmholtz, Navier-Stokes, and structural mechanics, demonstrate that RRFF and RRFF-FEM are robust to noise and achieve improved performance with reduced training time compared to the unregularized random feature model, while maintaining competitive accuracy relative to kernel and neural operator tests.
\end{abstract}


\section{Introduction}
The development of efficient numerical methods to approximate solutions of partial differential equations (PDEs) is an important computational task in scientific and economic disciplines, including physics, engineering, and finance. For example, classical methods, such as finite difference methods, finite element methods (FEM), finite volume methods, and spectral methods, are fundamental approaches for designing numerical solvers to PDEs. In practice, all of these methods require knowledge of the governing equation of the PDE in order to construct the solver, which may not be accessible when the model or underlying dynamics are unknown. Additionally, these methods can be sensitive to noisy inputs or boundary data, since they assume the solution has some degree of smoothness. As a result, deep learning techniques have gained popularity as surrogate solvers for PDEs in the settings of incomplete models and noisy data. An example is physics-informed neural networks (PINNs) \cite{raissi2019pinn}, which train a deep neural network to approximate solutions of a PDE by minimizing a mean squared error loss corresponding to the training data and the PDE at a finite set of collocation points. One potential drawback, however, is that PINNs are designed to encode one PDE and initial or boundary condition at a time and thus must be retrained from scratch for any problem changes. This can incur high computational costs for multi-query applications. Operator learning (OL) trains a neural network to approximate the mapping between two function spaces\cite{chen1995universal,lu2021learning, li2021fourier}. In the context of approximating the solution of a PDE, OL is used to directly learn the solution operator between the input functions (initial or boundary data) and the output functions (the solution to the PDE). This eliminates the need for retraining the neural network for new data. 

Using neural networks to approximate solution operators was initially proposed in \cite{chen1993approximations, chen1995universal}, where they approximated functionals and operators using shallow neural networks formed from a linear combination of the product of subnetworks. For high-dimensional problems related to solving PDEs, popular neural operator architectures utilize deeper networks; examples include the deep operator network (DeepONet) \cite{lu2021learning}, Fourier neural operator (FNO) \cite{li2021fourier}, BelNet \cite{zhang2023belnet}, DeepGreen \cite{gin2021deepgreen}, and many more. Theoretical guarantees for neural operators \cite{chen1995universal} rely on the universal approximation theorem for functions, which states that a (shallow) neural network can accurately approximate any nonlinear continuous function. Although this property shows the existence of a network of sufficient size that approximates an operator up to a specified error tolerance, it does not guarantee that the network is computable in practice \cite{chen1995universal, deng2022approximationratesdeeponet, kovachki2021universalapproximationerrorboundsfno}. Additionally, training can be computationally expensive, in terms of time and computing resources. Recent work on scaling laws for operators  \cite{liu2024neural,weihs2025deep} have quantified the approximation rates for general Lipschitz operators and may provide an approach for estimating the accuracy for specific tasks.

Other methods for operator learning were introduced to overcome some of the shortcomings of neural operators. In \cite{mora2025operatorlearninggaussianprocesses}, a hybrid Gaussian process (GP) and neural network (NN) based framework was developed. It was shown that zero-mean GPs obtained accurate results with proper initialization of kernel parameters and no training, yielding a zero-shot model for operator learning. A kernel-based framework was developed in \cite{batlle2023kernelmethodscompetitiveoperator} with convergence results and a priori error bounds. An advantage is that the training of kernel-based methods are direct. Numerical results showed that kernel-based operator learning achieved comparable or superior performance to neural operators in both test accuracy and computational efficiency across a series of benchmark PDE examples. However, directly applying kernel methods to large datasets can be computationally expensive, since the method operates on the kernel matrix of size $m \times m$, where $m$ is the number of training samples.

The random feature method (RFM) can be viewed as an approximation to a kernel method, where randomization helps to reduce the computing cost. Instead of computing the kernel matrix directly, RFM maps the data to a low-dimensional feature space using a randomized feature map \cite{rahimi2007advances, rahimi2008uniformapproximationfunctionsrandombases}. Consequently, evaluations during training and testing  are faster than in kernel method. Random features were utilized in several approaches for PDE or high-dimensional systems including \cite{jingrun2022bridging,nelsen2021randomfeaturemodel,nelsen2024operatorlearningrandomfeatures,liu2023random,liao2025cauchyrandomfeaturesoperator}. In \cite{jingrun2022bridging}, random feature functions were employed as a randomized basis in a Galerkin-type method, after which collocation and penalty methods were used to enforce the PDE and boundary conditions. The works \cite{nelsen2021randomfeaturemodel, nelsen2024operatorlearningrandomfeatures, liao2025cauchyrandomfeaturesoperator} separately demonstrated that random features can be applied in an operator learning approach for generating the solutions to PDEs. In \cite{nelsen2021randomfeaturemodel, nelsen2024operatorlearningrandomfeatures}, various numerical experiments were conducted that showed RFM produces accurate results. A high-probability non-asymptotic error bound was also established under the framework of vector-valued reproducing kernel Hilbert spaces (RKHS) as a result of \cite{lanthaler2023errorbounds}. This result, however, requires that the target operator lies within an RKHS with an operator-valued kernel, which may not necessarily be the case in practice. Additionally, to train this model, samples in the frequency domain are needed, which may not always be accessible. The authors of \cite{liu2023random} used random features for learning the interaction kernels of multi-agent systems. 
In \cite{liao2025cauchyrandomfeaturesoperator}, generalization error bounds for operator learning with Cauchy random features were derived using the RFM error for function approximation. In some ways, this is similar to the operator learning theoretical approaches.  In the RFM setting, the analysis involves the condition number of the random feature matrix, similar to prior risk and generalization bounds from \cite{chen2024conditioning, chen2022concentration, liao2024differentially, saha2023harfe, hashemi2023generalization, liao2025cauchyrandomfeaturesoperator}. Since the condition number can be used to bound the error, obtaining theoretical properties of the singular values are often sufficient. In numerical experiments from \cite{liao2025cauchyrandomfeaturesoperator}, it was shown that the RFM can reduce training time while maintaining competitive test accuracy as compared to kernel method and neural operator for smooth problems. However, since the approximation is obtained as a mininum-norm interpolation, it is sensitive to noise and outliers. 

We propose a regularized random Fourier feature method for operator learning in the regime of noisy data. The approach utilizes regularization in the RFM to handle noise and a finite element interpolation to allow for evaluation at new spatial points. We provide theoretical guarantees on the conditioning of the random feature matrix and extensive numerical results for this approach on a variety of benchmark PDEs used in \cite{lu2022comprehensive, dehoop2022costaccuracy,batlle2023kernelmethodscompetitiveoperator,liao2025cauchyrandomfeaturesoperator}. We summarize the contributions as follows:
\begin{itemize}
    \item Following the operator learning framework utilized in \cite{mhaskar2023localapproximation,batlle2023kernelmethodscompetitiveoperator,liu2025generalizationerror,liao2025cauchyrandomfeaturesoperator}, we propose RRFF-FEM, a regularized random Fourier feature (RRFF) method paired with a finite element recovery map, for operator learning in the noisy data setting. 

    \item We derive high-probability bounds on the condition number of the random feature matrix with weights sampled from the Student's $t$ distribution, similar to the theoretical results in \cite{chen2024conditioning, liao2025cauchyrandomfeaturesoperator}. This covers the range of potential sampling from Gaussian to Cauchy distributions and can be seen as a generalization of prior results. We show that if the complexity ratio $\frac{N}{m}$ scales like $\log m$, where $N$ is the number of random features and $m$ is the number of training samples, then the random feature matrix is well-conditioned with high probability. Estimation and generalization bounds follow from \cite{liao2025cauchyrandomfeaturesoperator}.

    \item We present extensive numerical results on benchmark examples found in \cite{lu2022comprehensive, dehoop2022costaccuracy,batlle2023kernelmethodscompetitiveoperator,liao2025cauchyrandomfeaturesoperator}. Our results show that for function approximation with noisy data, RRFF outperforms the unregularized random Fourier feature (RFF) method and achieves faster training times and lower test errors. Furthermore, for operator learning in the regime of noisy data, RRFF-FEM attains lower test errors than unregularized RFF coupled with FEM, which we denote by RFF-FEM. The implementation of our RFM-based operator learning algorithm is also simple and does not require costly computational resources, such as GPU.
\end{itemize}

We organize our paper as follows. In Section \ref{problem_statement}, we define our proposed method for operator learning in the regime of noisy data, which combines a regularized random feature model and a finite element recovery map. Then in Section \ref{conditioning}, we prove a high-probability bound on the conditioning of the random feature matrix associated with the Student's $t$ distribution. Section \ref{numerical_experiments} provides our numerical results and comparisons to the unregularized random feature model. The discussion and concluding remarks are in Section \ref{discussion}.

\section{Problem Statement}
\label{problem_statement}
We follow a construction of the operator learning problem similar to that in \cite{ mhaskar2023localapproximation,bartolucci2023representationequivalentneuraloperators,batlle2023kernelmethodscompetitiveoperator,liu2025generalizationerror,liao2025cauchyrandomfeaturesoperator}; however, in this work, we consider the setting where the data is corrupted by random noise. We assume that the functions live on compact subspaces of $L^2$. Specifically, the input and output function spaces are denoted by:
\begin{align*}
    \mathcal{U} &= C^0\big(D_\mathcal{U},\mathbb{R}^{d_1}\big), \hspace{0.25cm}D_\mathcal{U}\subset \mathbb{R}^{d_x}\\
    \mathcal{V} &= C^0\big(D_\mathcal{V},\mathbb{R}^{d_2}\big), \hspace{0.25cm}D_\mathcal{V}\subset \mathbb{R}^{d_y},
\end{align*}
where $D_\mathcal{U}$ and $D_\mathcal{V}$ are the corresponding spatial domains for the input and output functions, respectively. In the noise-free setting, given a set of training samples $\{(u_j, v_j)\}_{j \in [M]} \subset \mathcal{U} \times \mathcal{V}$, the goal is to learn an operator $G: \mathcal{U} \to \mathcal{V}$ such that $G(u_j)=v_j$ for all $j\in[M]$. We assume $G$ is a Lipschitz continuous operator, i.e.,
there exists a constant $L_G > 0$ such that
\begin{equation*}
    \|G(u_1)-G(u_2) \|_{L^2(D_\mathcal{V})} \leq L_G\|u_2-u_2\|_{L^2(D_\mathcal{U})}
\end{equation*}
for all $u_1, u_2 \in \mathcal{U}$.
To make the problem tractable, we consider a finite dimensional representation of $G$ and assume our training data is observed on a finite number of function evaluations at prescribed collocation points. Given collocation points $\{\mathbf{x}_j\}_{j\in[n]} \subset D_\mathcal{U}$ and $\{\mathbf{y}_j\}_{j\in[m]}\subset D_\mathcal{V}$, define the sampling operators $S_\mathcal{U} : \mathcal{U}\to\mathbb{R}^n$ and $S_\mathcal{V}:\mathcal{V}\to\mathbb{R}^m$:
\begin{align*}
    S_\mathcal{U}(u) &= [u(\mathbf{x}_1),\dots,u(\mathbf{x}_n)]^\top \in \mathbb{R}^n\\
    S_\mathcal{V}(v) &= [v(\mathbf{y}_1),\dots,v(\mathbf{y}_m)]^\top \in \mathbb{R}^m,
\end{align*}
i.e., the evaluation of the functions on the collocation points. 

In the noisy case, given ``true'' input-output pairs $\{(\mathbf{u}_j,\mathbf{v}_j)\}_{j\in[M]} = \{(S_\mathcal{U}(u_j),S_\mathcal{V}(v_j))\}_{j\in[M]}$, we observe noisy data:
\begin{equation*}
    \tilde{\mathbf{u}}_j = \mathbf{u}_j + \boldsymbol{\epsilon}_{\mathbf{u}_j} \ \ \text{ and } \ \ \tilde{\mathbf{v}}_j = \mathbf{v}_j + \boldsymbol{\epsilon}_{\mathbf{v}_j},
\end{equation*}
where $\boldsymbol{\epsilon}_{\mathbf{u}_j} \sim N(0, \sigma_{u_j} \mathbf{I}_n)$ and $\boldsymbol{\epsilon}_{\mathbf{v}_j} \sim N(0, \sigma_{v_j} \mathbf{I}_m)$ for all $j \in [M]$. The noise is assumed to be i.i.d. and isotropic in space (i.e., for each fixed $j$) but can vary between functions.\footnote{See Appendix \hyperref[appendix:noise]{A} for additional details on how noise is added for the numerical examples in Section \ref{numerical_experiments}.} The aim is to recover the operator $G$ from the training data $\{(\tilde{\mathbf{u}}_j,\tilde{\mathbf{v}}_j)\}_{j\in[M]}$, but since the data is noisy, the condition $G(u_j)=v_j$ for all $j\in[M]$ is relaxed to $G(u_j)\approx v_j$ for all $j\in[M]$.
Formally, our problem statement is defined below.
\begin{problem}
    We aim to learn an operator $G$ from the noisy training data $\{\tilde{\mathbf{u}}_j, \tilde{\mathbf{v}}_j\}_{j\in[M]}$, such that the functions $u_j\in\mathcal{U},v_j\in\mathcal{V}$ satisfy $G(u_j)\approx v_j$ for all $j\in[M]$.
\end{problem}

This problem can be expressed in a commutative diagram, as depicted in Figure \ref{fig:diagram}. The map $f: \mathbb{R}^n \to \mathbb{R}^m$, which is the finite dimensional representation of the operator $G$, is defined explicitly as:
\begin{equation*}
    f = S_\mathcal{V} \circ G \circ R_\mathcal{U},
\end{equation*}
where $R_\mathcal{U}: \mathbb{R}^n \to \mathcal{U}$ is a recovery map that takes vectorized function values in $\mathbb{R}^n$ as input and outputs a function in $\mathcal{U}$. We assume $f$ has the form $f(\mathbf{u}) = [f_1(\mathbf{u}),\dots,f_m(\mathbf{u})]^\top \in \mathbb{R}^m$ for every $\mathbf{u} \in \mathbb{R}^n$ and each component $f_j: \mathbb{R}^n \to \mathbb{R}$ is Lipschitz, i.e., there exists a constant $L_j >0$ such that 
\begin{equation*}
    |f_j(\mathbf{u}_1) - f_j(\mathbf{u}_2)|\leq L_j \|\mathbf{u}_1 - \mathbf{u}_2\|_2
\end{equation*}
for all $\mathbf{u}_1$, $\mathbf{u}_2 \in \mathbb{R}^n$. After learning the recovery map $R_\mathcal{V}: \mathbb{R}^m \to \mathcal{V}$, which maps a vector of function evaluations in $\mathbb{R}^m$ to a function in $\mathcal{V}$, and the approximation $\hat{f}:\mathbb{R}^n \to \mathbb{R}^m$ of $f$, we obtain an estimator: 
\begin{equation*}
    \hat{G} = R_\mathcal{V} \circ \hat{f} \circ S_\mathcal{U}
\end{equation*}
for the operator $G$. In the noise-free setting, the recovery map $R_\mathcal{V}$ can be viewed as an interpolation operator.

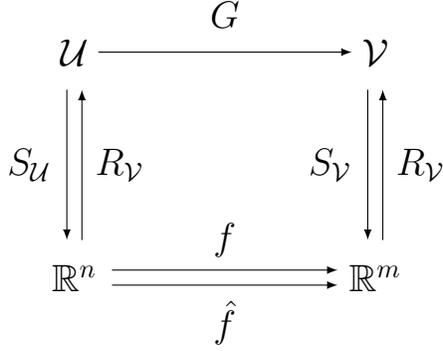
\begin{figure}[!t]
\centering
\begin{tikzpicture}
\tikzset{edge/.style = {->,> = latex}}
\tikzstyle{every node}=[font=\Large]
\node [font=\Large] at (8,15)(A) {$\mathcal{U}$};
\node [font=\Large] at (12,15)(B) {$\mathcal{V}$};
\node [font=\Large] at (8,12)(C) {$\mathbb{R}^n$};
\node [font=\Large] at (12,12)(D) {$\mathbb{R}^m$};
\draw [edge] (A) -- (B);
\draw [edge] (7.9,14.5) -- (7.9,12.5);
\draw [edge] (8.1,12.5) -- (8.1,14.5);
\draw [edge] (11.9,14.5) -- (11.9,12.5);
\draw [edge] (12.1,12.5) -- (12.1,14.5);
\draw [edge] (8.5,12.1) -- (11.5,12.1);
\draw [edge] (8.5,11.9) -- (11.5,11.9);
\node [font=\Large] at (10,15.5) {$G$};
\node [font=\Large] at (7.4,13.5) {$S_\mathcal{U}$};
\node [font=\Large] at (8.6,13.5) {$R_\mathcal{V}$};
\node [font=\Large] at (10,12.5) {$f$};
\node [font=\Large] at (10,11.4) {$\hat{f}$};
\node [font=\Large] at (12.6,13.5) {$R_\mathcal{V}$};
\node [font=\Large] at (11.4,13.5) {$S_\mathcal{V}$};
\end{tikzpicture}
\caption{Commutative diagram of the operator learning framework following \cite{ mhaskar2023localapproximation,bartolucci2023representationequivalentneuraloperators,batlle2023kernelmethodscompetitiveoperator,liu2025generalizationerror,liao2025cauchyrandomfeaturesoperator}.}
\label{fig:diagram}
\end{figure}

In \cite{bartolucci2023representationequivalentneuraloperators,batlle2023kernelmethodscompetitiveoperator, liu2025generalizationerror, liao2025cauchyrandomfeaturesoperator}, this operator learning framework was proposed alongside various function approximation techniques. The works \cite{bartolucci2023representationequivalentneuraloperators, liu2025generalizationerror} both approximated $f$ using neural networks, with \cite{liu2025generalizationerror} using an auto-encoder-based approach. In \cite{batlle2023kernelmethodscompetitiveoperator}, $\mathcal{U}$ and $\mathcal{V}$ were endowed with RKHS structures, and optimal recovery maps in the RKHS were employed to construct $R_\mathcal{U}$ and $R_\mathcal{V}$. Furthermore, the optimal recovery map in a vector-valued RKHS served as an approximation of $f$. Applying the standard representer formulae in the theory of optimal recovery, these maps admit closed-form expressions for kernel interpolation \cite{scholkopf2001generalizedrepresentertheorem, foucart2022learningnonrandomdata}. Lastly, the authors in \cite{liao2025cauchyrandomfeaturesoperator} constructed a RFM-based approach to approximate $f$ and provided theoretical guarantees for the error. In the prior work, the measurements were assumed to be noise-free, and thus, minimum-norm interpolation could be utilized within each component of the commutative diagram. In this work, we construct a random feature approximation $\hat{f}$ of $f$ with regularized least squares (accounting for noise) and a finite element interpolation for the recovery maps $R_\mathcal{U}$ and $R_\mathcal{V}$ to allow for flexibility in meshes and geometry.

\subsection{Random Feature Model} \label{rfm}
\noindent To construct a random Fourier feature approximation of some function $g$, we draw N i.i.d. samples $\{\boldsymbol{\omega}_k\}_{k\in[N]} \subset \mathbb{R}^d$ from a probability distribution $\rho(\boldsymbol{\omega})$ and define the random feature map:
\begin{equation*}
    g^{\#}(\mathbf{x})=\sum_{k=1}^N c_k^{\#} \exp{(i\langle \boldsymbol{\omega}_k,\mathbf{x})\rangle}.
\end{equation*}
The probability distribution is chosen by the user and can be utilized to encode additional regularity into the function approximation. From the commutative diagram, the map $f: \mathbb{R}^n \to \mathbb{R}^m$ has the form $f(\mathbf{u}) = [f_1(\mathbf{u}),\dots,f_m(\mathbf{u})]^\top \in \mathbb{R}^m$. Let $\hat{f_j}: \mathbb{R}^n \to \mathbb{R}$ be a random feature approximation of $f_j:\mathbb{R}^n \to \mathbb{R}$:
\begin{equation}\label{f_j}
    \hat{f}_j(\mathbf{u}) = \sum\limits_{k=1}^N c_k^{(j)} \exp(i\langle \boldsymbol{\omega}_k, \mathbf{u}\rangle)
\end{equation}
so that the vector-valued random feature map of $f$ is: 
\begin{equation}\label{f_hat}
    \hat{f}(\mathbf{u}) = \big[\hat{f}_1(\mathbf{u}),\dots,\hat{f}_m(\mathbf{u})\big]^\top \in \mathbb{R}^m.
\end{equation}

\noindent We compute $\mathbf{c}^{(j)} \in \mathbb{R}^N$ for all $j\in[m]$ by solving the minimization problem:
    \begin{equation}\label{min_f}
    \mathbf{c}^{(j)}= \argmin\limits_{\mathbf{x}\in \mathbb{R}^N} \left\|\mathbf{A} \mathbf{x} - \mathbf{v}^{(j)}\right\|_2^2 +\alpha\sum\limits_{k=1}^N \left\|\boldsymbol{\omega}_k\right\|_2^p \left|\mathbf{x}_k\right|^2
    \end{equation}
where $\mathbf{A} \in \mathbb{C}^{M\times N}$ is defined component-wise by $\mathbf{A}_{\ell,k} = \exp(i\langle \boldsymbol{\omega}_k, \mathbf{u}_\ell\rangle )$ and 
\begin{equation*}
    \mathbf{v}^{(j)} =[v_1(\mathbf{y}_j),\dots,v_M(\mathbf{y}_j)]^\top \in \mathbb{R}^M.
\end{equation*}
The second term in (\ref{min_f}) is a regularization term that handles noisy data by penalizing the coefficients $c_k^{(j)}$ associated with the high frequency random feature weights $\boldsymbol{\omega}_k$. 

We denote the unregularized random Fourier feature model with $\alpha = 0$ as RFF and the regularized random Fourier feature model with $\alpha> 0$ as RRFF. In our numerical examples for RFF and RRFF in Section \ref{numerical_experiments}, we use feature weights generated from the multivariate Student's $t$ distribution with varying degrees of freedom $\nu$ and refer to the RFF method as RFF-$\nu$ and the RRFF method as RRFF-$\nu$, where $\nu$ is replaced by the chosen degree of freedom.

\subsection{Finite Element Recovery Map} \label{fem}
We define the recovery maps $R_\mathcal{U}: \mathbb{R}^n \to \mathcal{U}$, which maps a vector $\mathbf{u} \in \mathbb{R}^n$ to a function $R_\mathcal{U}[\mathbf{u}]\in \mathcal{U}$, and $R_\mathcal{V}: \mathbb{R}^m \to \mathcal{V}$, which maps a vector $\mathbf{v} \in \mathbb{R}^m$ to a function $R_\mathcal{V}[\mathbf{v}]\in \mathcal{V}$. We construct $R_\mathcal{U}$ and $R_\mathcal{V}$ using finite element interpolation. For a given $\mathbf{u} = [u_1,\dots,u_n]^\top\in\mathbb{R}^n$, we form the function $R_\mathcal{U}[\mathbf{u}]\in\mathcal{U}$ that interpolates the values $\{u_j\}_{j\in[n]}$ on the grid points $\{\mathbf{x}_j\}_{j\in[n]} \subset D_\mathcal{U}$ as follows. Let $\mathcal{T}_\mathcal{U}$ be a triangulation of $D_\mathcal{U}$ with nodes $\{\mathbf{x}_j\}_{j\in[n]}$. Define the approximation space using Lagrange finite elements, which will be utilized in the examples:
\begin{equation*}
    V_{\mathcal{T}_\mathcal{U}} := \{v\in C^0(\overline{D_\mathcal{U}}): v|_K\in\mathcal{P}_k(K) \text{ for all } K\in\mathcal{T}_\mathcal{U}\} \subset \mathcal{U}.
\end{equation*}
Let $\{\phi_j\}_{j\in[n]}$ be a basis of $V_{\mathcal{T}_\mathcal{U}}$ such that: \begin{equation*}
    \phi_j(\mathbf{x}_i) = \delta_{ij} = \begin{cases}
    0 \text{ if } i\neq j\\
    1 \text{ if } i = j
\end{cases}
\end{equation*}
The finite element interpolant $R_\mathcal{U}[\mathbf{u}]\in V_{\mathcal{T}_\mathcal{U}} \subset \mathcal{U}$ is thus defined as:
\begin{equation*}
    R_\mathcal{U}[\mathbf{u}](\mathbf{x}) = \sum_{j=1}^n u_j\phi_j(\mathbf{x}).
\end{equation*}
Similarly, for any vector $\mathbf{v}=[v_1,\dots,v_m]^\top\in\mathbb{R}^m$, we define $R_\mathcal{V}[\mathbf{v}]\in\mathcal{V}$ as the function  that interpolates the values $\{v_j\}_{j\in[m]}$ at the grid points $\{\mathbf{y}_j\}_{j\in[m]} \subset D_\mathcal{V}$. Let $\mathcal{T}_\mathcal{V}$ denote a triangulation of $D_\mathcal{V}$ with nodes $\{\mathbf{y}_j\}_{j\in[m]}$. As before, we use Lagrange finite elements and define the approximation space: 
\begin{equation*}
    V_{\mathcal{T}_\mathcal{V}} := \{v\in C^0(\overline{D_\mathcal{V}}): v|_K\in\mathcal{P}_k(K) \text{ for all } K\in\mathcal{T}_\mathcal{V}\} \subset \mathcal{V}.
\end{equation*}
Let $\{\psi_j\}_{j\in[m]}$ be a basis of $V_{\mathcal{T}_\mathcal{V}}$ satisfying: \begin{equation*}
    \psi_j(\mathbf{y}_i) = \delta_{ij} = \begin{cases}
    0 \text{ if } i\neq j\\
    1 \text{ if } i = j 
\end{cases}
\end{equation*}
The finite element interpolant $R_\mathcal{V}[\mathbf{v}]\in V_{\mathcal{T}_\mathcal{V}} \subset \mathcal{V}$ is given by:
\begin{equation}\label{R_v}
    R_\mathcal{V}[\mathbf{v}](\mathbf{y}) = \sum_{j=1}^m v_j\psi_j(\mathbf{y}).
\end{equation}

For operator learning, we combine RFF and RRFF with the finite element recovery map and refer to these methods as RFF-FEM and RRFF-FEM, respectively. This enables the output to be a function (through interpolation), which can be evaluated throughout the spatial domain. In particular, we use the RFF and RRFF methods for $\hat{f}$ and the finite element recovery map for $R_\mathcal{V}$. The numerical experiments for RFF-FEM and RRFF-FEM in Section \ref{numerical_experiments} use feature weights drawn from the multivariate Student's $t$ distribution with different degrees of freedom $\nu$. In these cases, we denote the RFF-FEM method as RFF-FEM-$\nu$ and the RRFF-FEM method as RRFF-FEM-$\nu$, where $\nu$ is substituted with the value we chose for the degree of freedom.

\begin{algorithm}[!htbp]
    \caption{RRFF - Training of $\hat{f}$}
    \label{f_hat_algorithm}
    \begin{algorithmic}[1]
        \Require Noisy training data $\{(\tilde{\mathbf{u}}_j,\tilde{\mathbf{v}}_j)\}_{j\in[M]}$, number of random features $N$, and probability distribution $\rho(\boldsymbol{\omega})$
        \Ensure Random feature approximation $\hat{f}$
        \State Sample $N$ i.i.d. random feature weights $\{\boldsymbol{\omega}_k\}_{k\in[N]}$ from $\rho(\boldsymbol{\omega})$.
        \For{$j =1:m$} \State Compute $\mathbf{c}^{(j)}$ by solving (\ref{min_f}) using Cholesky decomposition.
        \EndFor
        \State Formulate each component $\hat{f}_j$ of  $\hat{f}$ by substituting trained coefficient vectors $\mathbf{c}^{(j)}$ and random feature weights $\{\boldsymbol{\omega}_k\}_{k\in[N]}$ into (\ref{f_j}).
        \State Return $\hat{f}$ by concatenating the $\hat{f}_j$'s as in (\ref{f_hat}).
    \end{algorithmic}
\end{algorithm}

\begin{algorithm}[!htbp]
    \caption{RRFF-FEM - Inference}
    \label{RRFF-FEM-inference}
    \begin{algorithmic}[1]
        \Require Test function $u\in\mathcal{U}$
        \Ensure Random feature approximation $\hat{G}(u)$ of $G(u)$
        \State Use sampling operator $S_\mathcal{U}$ to get function values of $u$ at collocation points $\{\mathbf{x}_j\}_{j\in[n]}$ and let $\mathbf{u} = S_\mathcal{U}(u)$.
        \State Obtain random feature map $\hat{f}$ from Algorithm \ref{f_hat_algorithm}.
        \State Substitute $\mathbf{u}$ into $\hat{f}$ and let $\mathbf{v} = \hat{f}(\mathbf{u})$.
        \State Return finite element interpolant $R_\mathcal{V}[\mathbf{v}]$ of the form (\ref{R_v}).
    \end{algorithmic}
\end{algorithm}

\section{Conditioning of Random Feature Matrix}
\label{conditioning}
In this section, we discuss bounds on the extreme singular values of the random feature matrix associated with this approach. We consider the overparameterized setting, where the number of random features $N$ is greater than the number of training samples $m$. We consider random Fourier features with feature weights drawn from the multivariate Student's $t$ distribution in $\mathbb{R}^d$ centered at $\mathbf{0}$ with scale parameter $\sigma > 0$ and degrees of freedom $\nu>0$, whose probability density function is given by:
\begin{equation*}
    \rho(\boldsymbol{\omega}) = \frac{\Gamma\left(\frac{\nu+d}{2}\right)}{\sigma\pi^{d/2}\nu^{d/2}\Gamma\left(\frac{\nu}{2}\right)}\left(1+\frac{1}{\sigma^2 \nu}\|\boldsymbol{\omega}\|_2^2\right)^{-\tfrac{\nu+d}{2}}.
\end{equation*}
This can be considered as an extension and generalization of the prior works \cite{chen2024conditioning, chen2022concentration, liao2024differentially, saha2023harfe, hashemi2023generalization, liao2025cauchyrandomfeaturesoperator}, since the Cauchy and Gaussian cases are subsets of the Student's $t$ distribution. This allows for greater flexibility for different applications.

\begin{remark}
    For $\nu=1$, the multivariate Student's $t$ distribution is the Cauchy distribution. As $\nu\to\infty$, the multivariate Student's $t$ distribution limits to the standard Gaussian distribution $\mathcal{N}(0,1)$.
\end{remark}

\subsection{Background} \label{background}
Consider a compact domain $D \subset \mathbb{R}^d$. Define the function space on D:
\begin{equation*}
    \mathcal{F}(\rho):=\left\{f(\textbf{x})=\int_{\mathbb{R}^d}\hat{f}(\boldsymbol{\omega})\exp(i\langle\boldsymbol{\omega},\textbf{x}\rangle)d\boldsymbol{\omega}:\| f\|_\rho^2 = \mathbb{E}_{\boldsymbol{\omega}} \left|\frac{\hat{f}(\boldsymbol{\omega})}{\rho(\boldsymbol{\omega})}\right|^2<\infty\right\},
\end{equation*}
where $\hat{f}(\boldsymbol{\omega})$ is the Fourier transform of the function $f$. Proposition 4.1 in \cite{rahimi2008uniformapproximationfunctionsrandombases} shows that the function space $\mathcal{F}(\rho)$ defines an RKHS with the kernel function:
\begin{equation*}
    k(\textbf{x},\textbf{y}) = \int_{\mathbb{R}^d} \exp(i\langle\boldsymbol{\omega},\textbf{x}-\textbf{y}\rangle)d\rho(\boldsymbol{\omega}).
\end{equation*}
For the multivariate Student's $t$ distribution, this is the Mat\'ern kernel:
\begin{equation*}
    k(\mathbf{x},\mathbf{y}) = \frac{\left(\sigma\sqrt{\nu}\|\mathbf{x}-\mathbf{y}\|_2\right)^{\nu/2}}{2^{\nu/2-1}\Gamma\left(\frac{\nu} {2}\right)} K_{\nu/2}\left(\sigma\sqrt{\nu}\|\mathbf{x}-\mathbf{y}\|_2\right),
\end{equation*}
which generates the classical Sobolev space:
\begin{align*}
    H^{\frac{\nu+d}{2}}\left(\mathbb{R}^d\right) &= W^{\frac{\nu+d}{2},2}\left(\mathbb{R}^d\right)\\ &= \left\{u\in L^2\left(\mathbb{R}^d\right):D^pu\in L^2\left(\mathbb{R}^d\right) \text{ for all } |p|\leq\tfrac{\nu+d}{2}\right\},
\end{align*}
and is a natural function space for studying PDEs \cite{porcu2024matern}. The Sobolev embedding theorem ensures that all functions in this space are well-defined and continuous provided $\nu>0$ \cite{evans2010pde}.

\subsection{Main Result} \label{main_result}
Before we show this section's main result, we state the assumption on the training samples and feature weights that is used. 
\begin{assumption}\label{rf_assumption}
    Let $D \subset \mathbb{R}^d$ be a compact domain. For the data points $\{\mathbf{x}_j\}_{j\in[m]}\subset D$, there exists a constant $\kappa>0$ such that $\|\mathbf{x}_j-\mathbf{x}_{j'}\|_2 \geq \kappa$ for all $j,j'\in[m]$ with $j\neq j'$. The feature weights $\{\boldsymbol{\omega}_k\}_{k\in N}\subset\mathbb{R}^d$ are sampled from the multivariate Student's $t$ distribution centered at $\mathbf{0}$ with scale parameter $\sigma$ and degrees of freedom $\nu > 0$. The random feature matrix $\mathbf{A} \in \mathbb{C}^{m\times N}$ is defined component-wise by $\mathbf{A}_{j,k}=\exp(i\langle\boldsymbol{\omega}_k,\mathbf{x}_j\rangle)$.  
\end{assumption}

Using this assumption, we prove the following result on the concentration of the singular values of the associated random feature matrix.

\begin{theorem}[Concentration Property of Random Feature Matrix]\label{thm:concentration}
Let $D \subset \mathbb{R}^d$ be a compact domain. Assume the data $\{\mathbf{x}_j\}_{j\in[m]}\subset D$, the random feature weights $\{\boldsymbol{\omega}_k\}_{k\in[N]}$, and the random feature matrix $\mathbf{A}\in\mathbb{C}^{m\times N}$ satisfy Assumption \ref{rf_assumption}. Suppose the following conditions hold: 
\begin{align}
    &N \geq C\eta^{-2}m\log\left(\frac{2m}{\delta}\right)\label{cond1}\\
    &\frac{(\sigma\sqrt{\nu}\kappa)^{\nu/2}}{2^{\nu/2-1}\Gamma\left(\frac{\nu}{2}\right)} K_{\nu/2}(\sigma\sqrt{\nu}\kappa)\leq \frac{\eta}{m}\label{cond2}
\end{align}
for some $\delta,\eta \in(0,1)$, where $C>0$ is a universal constant and $K_{\nu/2}(x)$ is the modified Bessel function of the second kind with order $\nu/2$. Then with probability at least $1-\delta$, we have:
\begin{equation*}
    \left\|\frac{1}{N}\mathbf{A}\mathbf{A}^*-\mathbf{I}_m\right\|_2\leq 2\eta.
\end{equation*}
\end{theorem}

This result characterizes the behavior of the condition number of the random feature matrix as a function of the complexity ratio $\frac{N}{m}$. If $m\ll N$, then $\eta$ can be small so the eigenvalues of $\frac{1}{N}\mathbf{A}\mathbf{A}^*$ concentrate near 1 and the condition number of $\mathbf{A}$ is small with probability at least $1-\delta$. We also have as a result that the matrix $\mathbf{A}\mathbf{A}^*$ is invertible with high probability so the pseudoinverse $\mathbf{A}^\dagger = \mathbf{A}^*(\mathbf{A}\mathbf{A}^*)^{-1}$ is well-defined.

\begin{proof}[Proof of Theorem~\ref{thm:concentration}]
By triangle's inequality,
\begin{equation}
    \left\|\frac{1}{N}\mathbf{A}\mathbf{A}^*-\mathbf{I}_m\right\|_2 \leq \left\|\frac{1}{N}\mathbf{A}\mathbf{A}^*-\mathbb{E}_{\boldsymbol{\omega}}\left[\frac{1}{N}\mathbf{A}\mathbf{A}^*\right]\right\|_2 + \left\|\mathbb{E}_{\boldsymbol{\omega}}\left[\frac{1}{N}\mathbf{A}\mathbf{A}^*\right] - \mathbf{I}_m\right\|_2.
    \label{triangle_inequal}
\end{equation}
 To bound the first term in \eqref{triangle_inequal}, let $\mathbf{W}_\ell = \mathbf{A}_{:,\ell}$, the $\ell$-th column of $\mathbf{A}$. Define the matrices $\{\mathbf{Y}_\ell\}_{\ell\in [N]} \subset \mathbb{C}^{m\times m}$ as:
\begin{equation*}
    \mathbf{Y}_\ell = \mathbf{W}_\ell\mathbf{W}_\ell^* - \mathbb{E}_{\boldsymbol{\omega}}\left[\mathbf{W}_\ell\mathbf{W}_\ell^*\right].
\end{equation*}

\noindent We have $(\mathbf{Y}_\ell)_{j,j} = 0$ and 
\begin{align*}
    (\mathbf{Y}_\ell)_{j,k} &= \exp(i\langle \boldsymbol{\omega}_\ell, \mathbf{x}_j-\mathbf{x}_k\rangle)-\mathbb{E}_{\boldsymbol{\omega}}\left[\exp(i\langle \boldsymbol{\omega}_\ell, \mathbf{x}_j-\mathbf{x}_k\rangle)\right]\\
    &= \exp(i\langle \boldsymbol{\omega}_\ell, \mathbf{x}_j-\mathbf{x}_k\rangle) - \phi_{\boldsymbol{\omega}}(\mathbf{x}_j-\mathbf{x}_k),
\end{align*}
for $j,k\in[m]$, where:
\begin{equation*}
    \phi_{\boldsymbol{\omega}}(\mathbf{t}) = \mathbb{E}_{\boldsymbol{\omega}}\left[\exp\left(i\boldsymbol{\omega}^\top\mathbf{t}\right)\right] = \frac{\left(\sigma\sqrt{\nu}\|\mathbf{t}\|_2\right)^{\nu/2}}{2^{\nu/2-1}\Gamma\left(\frac{\nu} {2}\right)} K_{\nu/2}\left(\sigma\sqrt{\nu}\|\mathbf{t}\|_2\right),
\end{equation*}
is the characteristic function of the multivariate Student's $t$ distribution. Note that $K_{\nu/2}(x) > 0$ for $x > 0$ and $\nu \in \mathbb{R}$ and $\frac{d}{dx}\left[x^\nu K_\nu(x)\right] = -x^\nu K_{\nu-1}(x) < 0$ for $x > 0$, so $\phi_{\boldsymbol{\omega}}(\mathbf{t}) > 0$ for all $\mathbf{t}\neq \mathbf{0}$ and $\phi_{\boldsymbol{\omega}}(\mathbf{t}) \leq \phi_{\boldsymbol{\omega}}(\mathbf{t}')$ for all $\mathbf{t},\mathbf{t}'\neq 0$ such that $\|\mathbf{t}\|_2 \geq \|\mathbf{t}'\|_2$. Also, $\mathbf{Y}_\ell = \mathbf{Y}_\ell^*$ so the singular values of $\mathbf{Y}_\ell$ are the absolute values of its eigenvalues. Using these facts and applying Gershgorin's disc theorem, $\|\mathbf{x}_j-\mathbf{x}_k\|_2 \geq \kappa$ for $j,k\in[m]$ with $j\neq k$, and condition (\ref{cond2}) yields:
\begin{align*}
\begin{split}
    \|\mathbf{Y}_\ell\|_2 &\leq \max\limits_{j\in[m]} \sum\limits_{j\neq k} |\exp(i\langle\boldsymbol{\omega}_\ell,\mathbf{x}_j-\mathbf{x}_k\rangle)-\phi_{\boldsymbol{\omega}}(\mathbf{x}_j-\mathbf{x}_k)|\\
    &\leq \max\limits_{j\in[m]} \sum\limits_{j\neq k} 1 + \phi_{\boldsymbol{\omega}}(\mathbf{x}_j-\mathbf{x}_k)\\
    &\leq \max\limits_{j\in[m]} m\left(1+\frac{(\sigma\sqrt{\nu}\kappa)^{\nu/2}}{2^{\nu/2-1}\Gamma\left(\frac{\nu}{2}\right)} K_{\nu/2}(\sigma\sqrt{\nu}\kappa)\right)\\
    &\leq m + \eta.
\end{split}
\end{align*}

 Note that $\mathbb{E}_{\boldsymbol{\omega}}\left[\mathbf{W}_\ell\mathbf{W}_\ell^*\right]$ is Hermitian so its singular values are the absolute values of its eigenvalues. Also, $\left(\mathbb{E}_{\boldsymbol{\omega}}\left[\mathbf{W}_\ell\mathbf{W}_\ell^*\right]\right)_{j,j}=1$ and $\left(\mathbb{E}_{\boldsymbol{\omega}}\left[\mathbf{W}_\ell\mathbf{W}_\ell^*\right]\right)_{j,k} = \phi_{\boldsymbol{\omega}}\left(\mathbf{x}_j-\mathbf{x}_k\right)$ for $j,k\in[m]$, hence
\begin{align}
\begin{split}\label{exp_ww*}
\left\|\mathbb{E}_{\boldsymbol{\omega}}\left[\mathbf{W}_\ell\mathbf{W}_\ell^*\right]\right\|_2 & \leq 1 + \max\limits_{j\in[m]} \sum\limits_{j\neq k} |\phi_{\boldsymbol{\omega}}\left(\mathbf{x}_j-\mathbf{x}_k\right)| \\
&\leq 1 + m\frac{(\sigma\sqrt{\nu}\kappa)^{\nu/2}}{2^{\nu/2-1}\Gamma\left(\frac{\nu}{2}\right)} K_{\nu/2}(\sigma\sqrt{\nu}\kappa)\\
&\leq 1+ \eta,
\end{split}
\end{align}
by Gershgorin's disc theorem, $\|\mathbf{x}_j-\mathbf{x}_k\|_2\geq \kappa$ for $j,k\in[m]$ with $j\neq k$, and condition (\ref{cond2}). Therefore, the norm of the variance of the sum of $\mathbf{Y}_\ell$'s is bounded by: 

\begin{align*}
\begin{split}
    \left\|\sum\limits_{\ell=1}^N \mathbb{E}_{\boldsymbol{\omega}}\left[\mathbf{Y}_\ell^2\right]\right\|_2 &\leq \sum\limits_{\ell=1}^N \left\|\mathbb{E}_{\boldsymbol{\omega}}\left[\mathbf{Y}_\ell^2\right]\right\|_2 \\
    &\leq \sum\limits_{\ell=1}^N \left\|\mathbb{E}_{\boldsymbol{\omega}}\left[\left(\mathbf{W}_\ell\mathbf{W}_\ell^*-\mathbb{E}_{\boldsymbol{\omega}}\left[\mathbf{W}_\ell\mathbf{W}_\ell^*\right]\right)^2\right]\right\|_2\\
    &\leq \sum\limits_{\ell=1}^N \left\|\mathbb{E}_{\boldsymbol{\omega}}\left[\mathbf{W}_\ell\mathbf{W}_\ell^*\mathbf{W}_\ell\mathbf{W}_\ell^*\right]-\left(\mathbb{E}_{\boldsymbol{\omega}}\left[\mathbf{W}_\ell\mathbf{W}_\ell^*\right]\right)^2\right\|_2\\
    &= \sum\limits_{\ell=1}^N m\left\|\mathbb{E}_{\boldsymbol{\omega}}\left[\mathbf{W}_\ell\mathbf{W}_\ell^*\right]\right\|_2+\left\|\mathbb{E}_{\boldsymbol{\omega}}\left[\mathbf{W}_\ell\mathbf{W}_\ell^*\right]\right\|_2^2 \\
    &\leq N\left(m(1+\eta)+(1+\eta)^2\right),
\end{split}
\end{align*}
where we used the fact that: $\mathbf{W}_\ell^*\mathbf{W}_\ell = \|\mathbf{W}_\ell\|_2^2 = m$ and (\ref{exp_ww*}).

 Since $\left\{\mathbf{Y}_\ell\right\}_{\ell\in\left[N\right]}$ are independent, mean-zero, Hermitian matrices, we can apply Matrix Bernstein's inequality to obtain:

\begin{align*}
\begin{split}
    \mathbb{P}\left(\left\|\frac{1}{N}\mathbf{A}\mathbf{A}^*-\mathbb{E}_{\boldsymbol{\omega}}\left[\frac{1}{N}\mathbf{A}\mathbf{A}^*\right]\right\|_2 \geq \eta\right) &= \mathbb{P} \left(\left\|\sum_{\ell=1}^N \mathbf{Y}_\ell\right\|_2 \geq N\eta\right) \\
    &\leq 2m\exp\left(-\frac{N\eta^2}{2\left[m(1+\eta)+(1+\eta)^2+(m+\eta)\eta/3\right]}\right)\\
    &\leq 2m \exp\left(-\frac{N\eta^2}{5m+9}\right) \\
    &\leq 2m \exp\left(-\frac{Cm\log\left(\frac{2m}{\delta}\right)}{5m+9}\right)\\
    &\leq 2m\exp\left(-\log\left(\frac{2m}{\delta}\right)\right)\\
    &=\delta,
\end{split}
\end{align*}
provided that $\eta < 1$ and condition \eqref{cond1} is satisfied with $C=6$ and $m\geq 9$.

For the second term in \eqref{triangle_inequal}, define the matrix $\mathbf{B} = \mathbb{E}_{\boldsymbol{\omega}} \left[\frac{1}{N}\mathbf{A}\mathbf{A}^*\right]-\mathbf{I}_m$. We have that $\mathbf{B}$ is Hermitian with $\mathbf{B}_{j,j} = 0$ and 
\begin{align*}
    \mathbf{B}_{j,k} &= \frac{1}{N} \sum_{\ell=1}^N \left(\mathbb{E}_{\boldsymbol{\omega}}\left[\mathbf{W}_\ell\mathbf{W}_\ell^*\right]\right)_{j,k} \\
    &= \frac{1}{N} \sum_{\ell=1}^N \phi_{\boldsymbol{\omega}}\left(\mathbf{x}_j-\mathbf{x}_k\right)\\
&=\phi_{\boldsymbol{\omega}}\left(\mathbf{x}_j-\mathbf{x}_k\right),
\end{align*}
for all $j,k\in[m]$ such that $j\neq k$. Using that $\mathbf{B}$ is Hermitian and applying Gershgorin's disc theorem, $\left\|\mathbf{x}_j-\mathbf{x}_k\right\|_2 \geq \kappa$ for $j,k \in [m]$ with $j\neq k$, and condition \eqref{cond2}, we conclude:
\begin{align*}
\begin{split}
    \left\|\mathbf{B}\right\|_2 &\leq \max_{j\in[m]} \sum_{k\neq j} \left|\phi_{\boldsymbol{\omega}}\left(\mathbf{x}_j - \mathbf{x}_k\right)\right|\\
    &\leq \max_{j\in[m]} m \left( \frac{(\sigma\sqrt{\nu}\kappa)^{\nu/2}}{2^{\nu/2-1}\Gamma\left(\frac{\nu}{2}\right)} K_{\nu/2}(\sigma\sqrt{\nu}\kappa) \right)\\
    &\leq m \left(\frac{\eta}{m}\right)\\
    &=\eta.
\end{split}
\end{align*}\end{proof}

We can use the concentration property of the random feature matrix $\mathbf{A}\in \mathbb{C}^{m\times N}$ to derive estimation and generalization bounds similar to \cite{liao2025cauchyrandomfeaturesoperator}.

\section{Numerical Experiments}
\label{numerical_experiments}
We performed numerical experiments on the following PDEs: advection equation, Burgers' equation, Darcy flow problem, Helmholtz equation, Navier-Stokes equation, and structural mechanics problem\footnote{The code is available at: \url{https://github.com/tracyyu13}}. The examples and datasets are from \cite{lu2022comprehensive,dehoop2022costaccuracy} and used as benchmark cases in \cite{batlle2023kernelmethodscompetitiveoperator,liao2025cauchyrandomfeaturesoperator}. 

Unless otherwise stated, the data is corrupted by 5\% relative Gaussian noise to the input and output functions during the training process, and in the testing phase, the input function evaluations are corrupted by 5\% relative Gaussian noise. We compare RFF and RRFF approximations (see Section \ref{rfm} for definitions) for $\hat{f}$ with feature weights sampled from the tensor-product Gaussian and multivariate Student's $t$ ($\nu=2,3$) distributions. Note that the Gaussian case is the Student's $t$ distribution with $\nu = \infty$. In this section, we refer to these approximations as RFF-2, RFF-3, RFF-$\infty$ when Student's $t$ RFF is used with $\nu= 2, 3, \infty$, respectively, and RRFF-2, RRFF-3, RRFF-$\infty$ when Student's $t$ RRFF is used with $\nu = 2, 3, \infty$, respectively, following the notation we defined in Section \ref{rfm}. Additionally, we compare RFF-FEM and RRFF-FEM approximations (see Section \ref{fem} for details) for $R_v \circ \hat{f}$, where $\hat{f}$ is a random feature approximation of f as before and $R_v$ is a finite element recovery map. Feature weights are drawn from the Student's $t$ distribution with $\nu = 2, 3, \infty$, and following the notation introduced in Section \ref{fem}, we denote the RFF-FEM predictions as RFF-FEM-2, RFF-FEM-3, RFF-FEM-$\infty$ for $\nu=2,3,\infty$, respectively, and the RRFF-FEM predictions as RRFF-FEM-2, RRFF-FEM-3, RRFF-FEM-$\infty$ for $\nu = 2,3,\infty$, respectively. We split the spatial grid and learn $\hat{f}$ on two-thirds of the grid points using Algorithm \ref{f_hat_algorithm} and use the remaining grid points to validate $R_v \circ \hat{f}$ with Algorithm \ref{RRFF-FEM-inference}. To measure generalization performance, we define the average relative test error for an estimator $\hat{G}:\mathcal{U}\to\mathcal{V}$ to be:
\begin{equation*}
    \text{Error}(\hat{G}) = \dfrac{1}{|\textbf{Test}|}\sum_{k\in\textbf{Test}}\dfrac{||G(u_k)-\hat{G}(u_k)||_{L^2(D_\mathcal{V})}}{||\hat{G}(u_k)||_{L^2(D_\mathcal{V})}},
\end{equation*}
where $G:\mathcal{U}\to\mathcal{V}$ is the true operator. On the full grid, we measure the performance of $\hat{f}$. On the split grid, we measure the performance of $R_\mathcal{V}\circ\hat{f}$, which is equivalent to measuring the generalization performance of $\hat{G}$. 

\subsection{Advection Equations \RomanNumeral{1}, \RomanNumeral{2}, and \RomanNumeral{3}}
Consider the one-dimensional advection equation:
\begin{equation*}
\begin{aligned}
    \dfrac{\partial v}{\partial t} + \dfrac{\partial v}{\partial x} &= 0, \hspace{0.5cm} &&(x,t) \in (0,1)\times(0,1] \\
    v(x,0) &= u(x), \hspace{0.5cm} &&x\in[0,1]
\end{aligned}
\end{equation*}
with periodic boundary conditions. The goal is to learn the operator $G: u(x) \mapsto v(x,0.5)$, i.e., the map from the initial condition $u(x)$ to the solution at time $t = 0.5$. The domains of the input and output function spaces $\mathcal{U}$ and $\mathcal{V}$ are $D_\mathcal{U} = D_\mathcal{V} = (0,1)$. Various distributions for the initial condition $u(x)$ were considered for this problem in \cite{lu2022comprehensive,dehoop2022costaccuracy}. In \cite{lu2022comprehensive}, the first initial condition, which we denote as Advection \RomanNumeral{1}, is:
\begin{equation*}
    u(x) = h\mathbbm{1}_{\{c-\frac{b}{2},c+\frac{b}{2}\}},
\end{equation*}
i.e., a square wave centered at $x=c$ with width $b$ and height $h$. The parameters $(c,b,h)$ are drawn uniformly from $[0.3,0.7]\times[0.3,0.6]\times[1,2]$. 

For RFF-FEM and RRFF-FEM, a grid is nonuniformly discretized with 27 grid points. We train $\hat{f}$ with 1000 samples and test its performance with 800 samples. We evaluate the performance of $R_\mathcal{V}\circ\hat{f}$ on 200 samples. Table \ref{advection_param_table} lists the parameters used in the numerical experiments on the advection equations. Figure \ref{Fig:AD1_train_recovery_Gaussian} shows an example of training input and output with 5\% noise, test function and approximations of $R_\mathcal{V}\circ \hat{f}$ using RFF-FEM-$\infty$ and RRFF-FEM-$\infty$, along with the corresponding pointwise errors. Figure \ref{Fig:AD1_recovery_Student} shows a test example and predictions of $R_\mathcal{V}\circ \hat{f}$ using RFF-FEM-2, RRFF-FEM-2, RFF-FEM-3, and RRFF-FEM-3 as well as their pointwise errors. 

\begin{table}[!htbp]
\centering
    \scriptsize	
    \begin{tabular}{{|P{0.11\linewidth} | P{0.12\linewidth} | P{0.08\linewidth} | P{0.08\linewidth} |
    P{0.08\linewidth} | P{0.08\linewidth} |}}
\hline
& \addstackgap{Distribution} & \addstackgap{$N$} & \addstackgap{$\sigma$} & \addstackgap{$\alpha$} & \addstackgap{$p$} 
\\ \hline \multirow{3}{*}{\addstackgap{\makebox[\linewidth][c]{Advection \RomanNumeral{1}, \RomanNumeral{2}}}} & \multicolumn{1}{c|}{\addstackgap{Gaussian}} & \multicolumn{1}{c|}{5k} & \multicolumn{1}{c|}{0.2} & \multicolumn{1}{c|}{0.01} & \multicolumn{1}{c|}{2}  \\ \cline{2-6}
& \multicolumn{1}{c|}{\addstackgap{Student ($\nu=2$)}} & \multicolumn{1}{c|}{5k} & \multicolumn{1}{c|}{0.02} & \multicolumn{1}{c|}{0.1}& \multicolumn{1}{c|}{2} \\ \cline{2-6}
& \multicolumn{1}{c|}{\addstackgap{Student ($\nu=3$)}} & \multicolumn{1}{c|}{5k} & \multicolumn{1}{c|}{0.02} & \multicolumn{1}{c|}{0.01}& \multicolumn{1}{c|}{2} \\
\hline
 \multirow{3}{*}{\addstackgap{\makebox[\linewidth][c]{Advection \RomanNumeral{3}}}} & \multicolumn{1}{c|}{\addstackgap{Gaussian}} & \multicolumn{1}{c|}{\addstackgap{5k}} & \multicolumn{1}{c|}{\addstackgap{$\sqrt{2 \times 10^{-3}}$}} & \multicolumn{1}{c|}{\addstackgap{0.1}} & \multicolumn{1}{c|}{\addstackgap{2}} \\ \cline{2-6}
& \multicolumn{1}{c|}{\addstackgap{Student ($\nu=2$)}} & \multicolumn{1}{c|}{5k} & \multicolumn{1}{c|}{0.001} & \multicolumn{1}{c|}{0.5}& \multicolumn{1}{c|}{2}\\ \cline{2-6}
& \multicolumn{1}{c|}{\addstackgap{Student ($\nu=3$)}} & \multicolumn{1}{c|}{5k} & \multicolumn{1}{c|}{0.001} & \multicolumn{1}{c|}{0.1}& \multicolumn{1}{c|}{2}\\
\hline
\end{tabular}
\caption{\textbf{Advection \RomanNumeral{1}, \RomanNumeral{2}, \RomanNumeral{3}}: Parameters for the numerical experiments, where $N$ is the number of random features and $\sigma$ is the scale parameter for the RFF, RRFF, RFF-FEM, and RRFF-FEM methods, and $\alpha$ and $p$ are regularization parameters for the RRFF and RRFF-FEM methods.}
\label{advection_param_table}
\end{table}

\begin{figure}[!htbp]
\centering
\subfigure{\includegraphics[width=36mm]{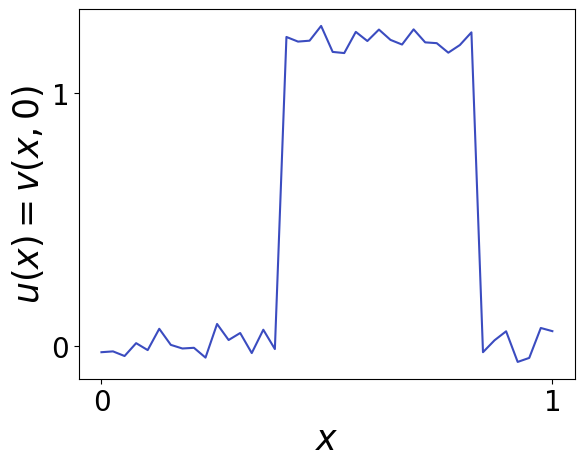}}
\subfigure{\includegraphics[width=36mm]{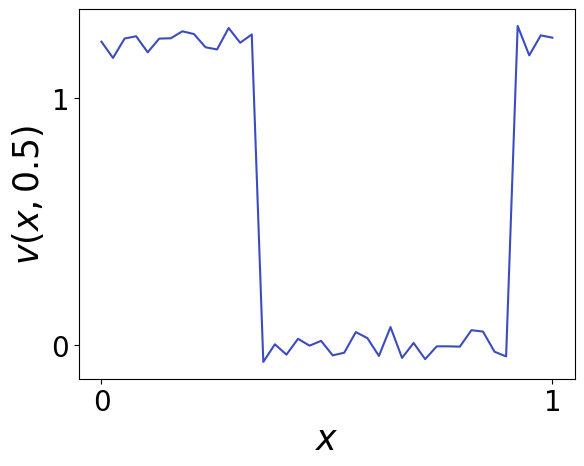}}
\subfigure{\includegraphics[width=36mm]{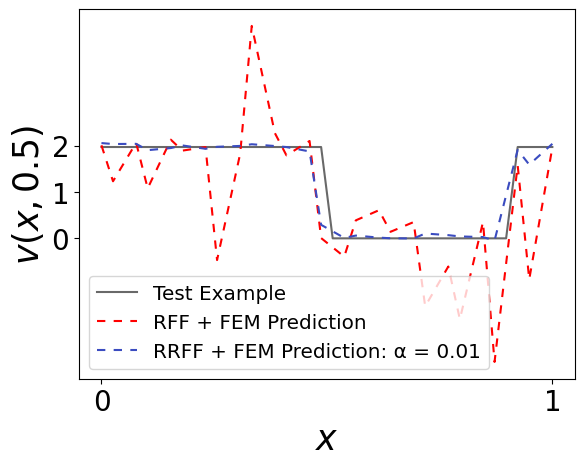}}
\subfigure{\includegraphics[width=37mm]{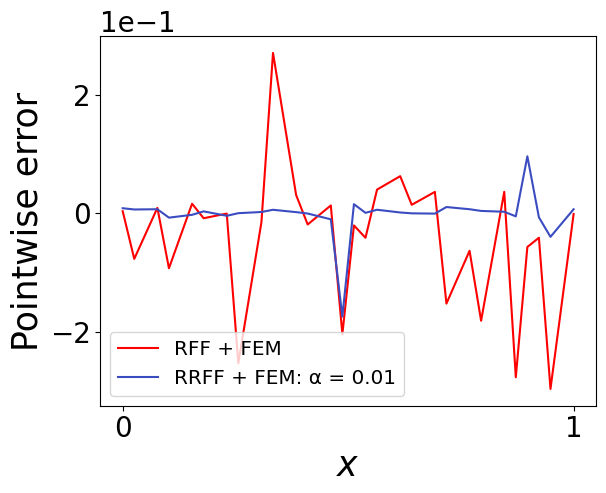}}
\caption{\textbf{Advection \RomanNumeral{1}}: (left to right) An example of (a) training input with 5\% noise, (b) training output with 5\% noise, (c) test example and predictions of $R_\mathcal{V}\circ\hat{f}$ using RFF-FEM-$\infty$ and RRFF-FEM-$\infty$, (d) pointwise errors for predictions of $R_\mathcal{V}\circ\hat{f}$ using RFF-FEM-$\infty$ and RRFF-FEM-$\infty$.}
\label{Fig:AD1_train_recovery_Gaussian}
\end{figure}

\begin{figure}[!htbp]
\centering     
\subfigure{\includegraphics[width=36mm]{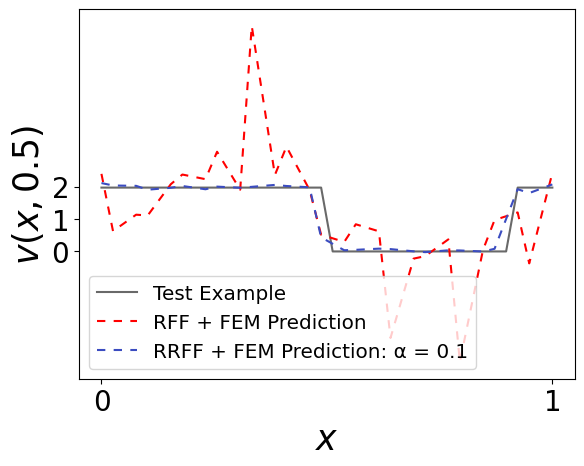}}
\subfigure{\includegraphics[width=37mm]{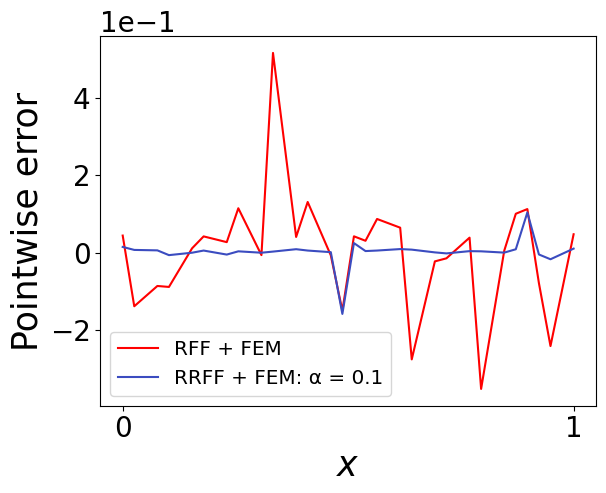}}
\subfigure{\includegraphics[width=36mm]{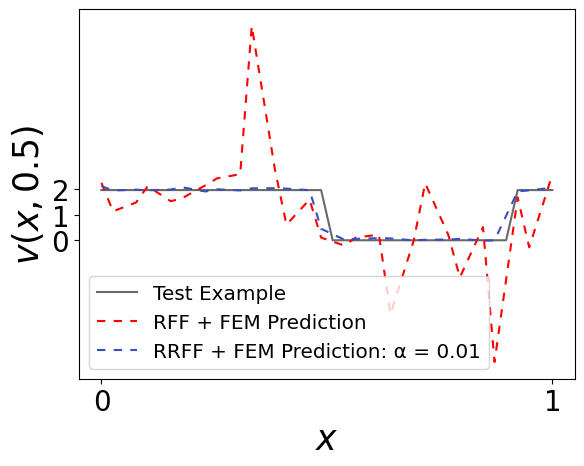}}
\subfigure{\includegraphics[width=39mm]{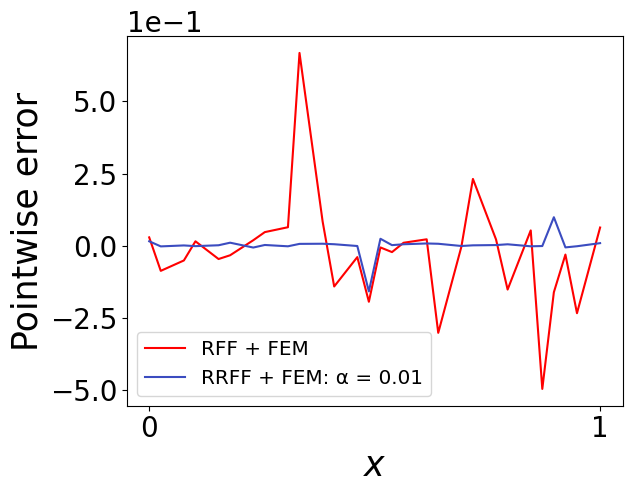}}
\caption{\textbf{Advection \RomanNumeral{1}}: (left to right) (a) test example and predictions of $R_\mathcal{V}\circ\hat{f}$  using RFF-FEM-2 and RRFF-FEM-2, (b) pointwise errors for predictions of $R_\mathcal{V}\circ\hat{f}$  using RFF-FEM-2 and RRFF-FEM-2, (c) test example and predictions of $R_\mathcal{V}\circ\hat{f}$  using RFF-FEM-3 and RRFF-FEM-3, (d) pointwise errors for predictions of $R_\mathcal{V}\circ\hat{f}$  using RFF-FEM-3 and RRFF-FEM-3.}
\label{Fig:AD1_recovery_Student}
\end{figure}

A more complex initial condition was also considered in \cite{lu2022comprehensive}, which takes the form:
\begin{equation*}
    u(x) = h_1\mathbbm{1}_{\{c_1-\frac{b}{2},c_1+\frac{b}{2}\}}+\sqrt{\max{(h_2^2-a^2(x-c_2)^2,0)}}.
\end{equation*}
We will refer to this as Advection \RomanNumeral{2}. A coarse nonuniform grid with 27 grid points is used for the RFF-FEM and RRFF-FEM methods. The function $\hat{f}$ is trained with 1000 samples, and its performance is tested with 800 samples. We assess $R_\mathcal{V}\circ\hat{f}$ on the hold-out 200 samples. Figure \ref{Fig:AD2_train_recovery_Gaussian} plots an example of training input and output with 5\% noise, test example and predictions of $R_\mathcal{V}\circ\hat{f}$ from RFF-FEM-$\infty$ and RRFF-FEM-$\infty$, and the pointwise errors. Figure \ref{Fig:AD2_recovery_Student} shows test examples with approximations of $R_\mathcal{V}\circ\hat{f}$ from RFF-FEM-2, RRFF-FEM-2, RFF-FEM-3, and RRFF-FEM-3 and the associated pointwise errors.

\begin{figure}[!t]
\centering
\subfigure{\includegraphics[width=36mm]{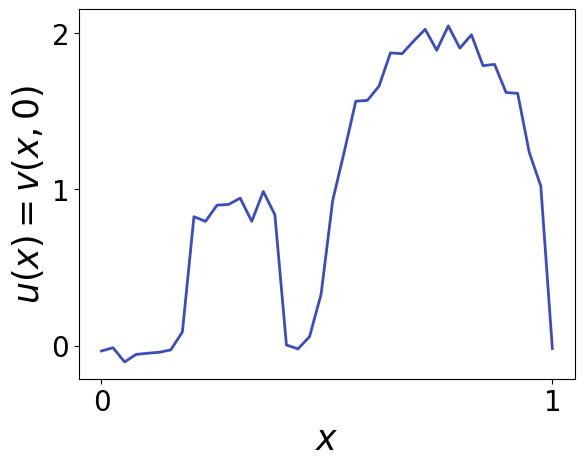}}
\subfigure{\includegraphics[width=36mm]{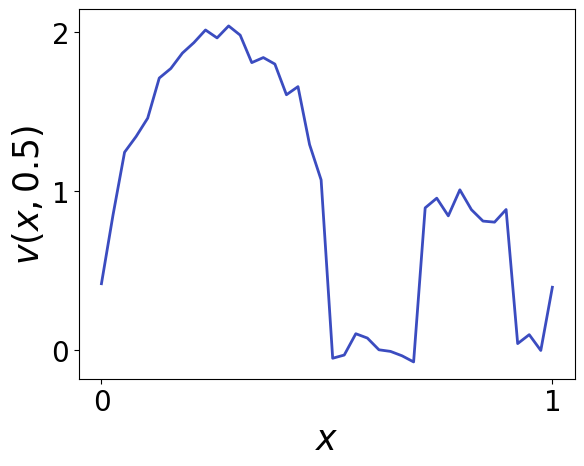}}
\subfigure{\includegraphics[width=36mm]{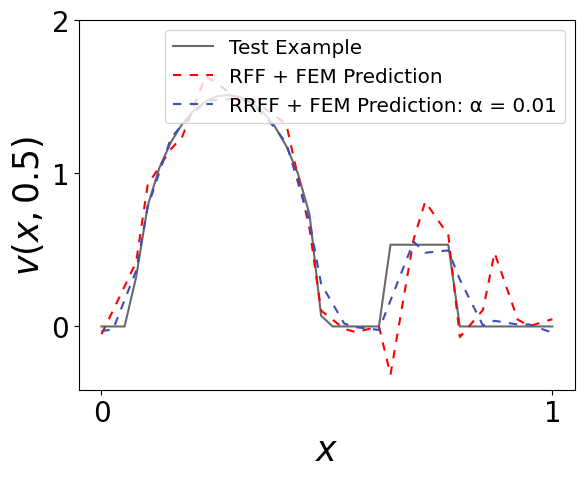}}
\subfigure{\includegraphics[width=39mm]{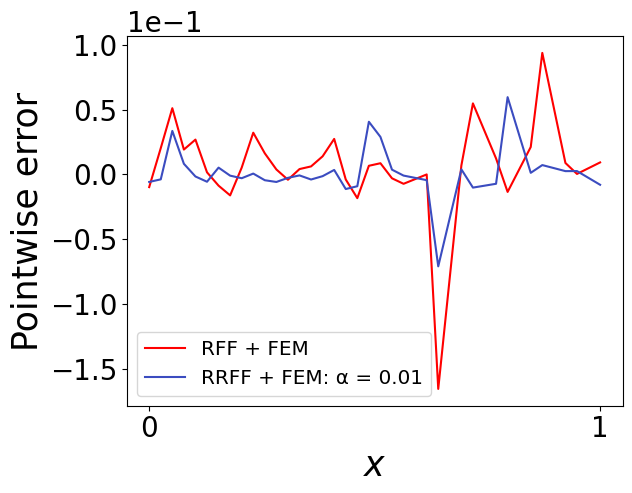}}
\caption{\textbf{Advection \RomanNumeral{2}}: (left to right) An example of (a) training input with 5\% noise, (b) training output with 5\% noise, (c) test example and predictions of $R_\mathcal{V}\circ\hat{f}$ using RFF-FEM-$\infty$ and RRFF-FEM-$\infty$, (d) pointwise errors for predictions of $R_\mathcal{V}\circ\hat{f}$ using RFF-FEM-$\infty$ and RRFF-FEM-$\infty$.}
\label{Fig:AD2_train_recovery_Gaussian}
\end{figure}

\begin{figure}[!t]
\centering     
\subfigure{\includegraphics[width=36mm]{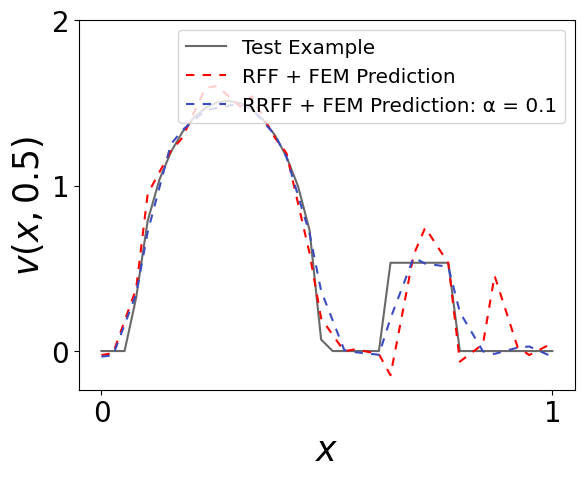}}
\subfigure{\includegraphics[width=39mm]{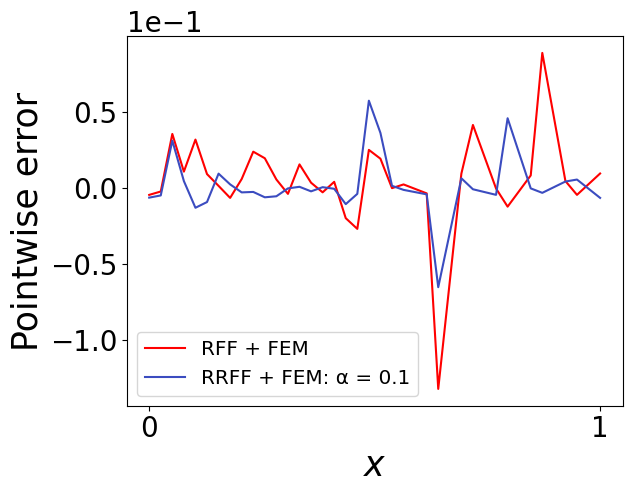}}
\subfigure{\includegraphics[width=36mm]{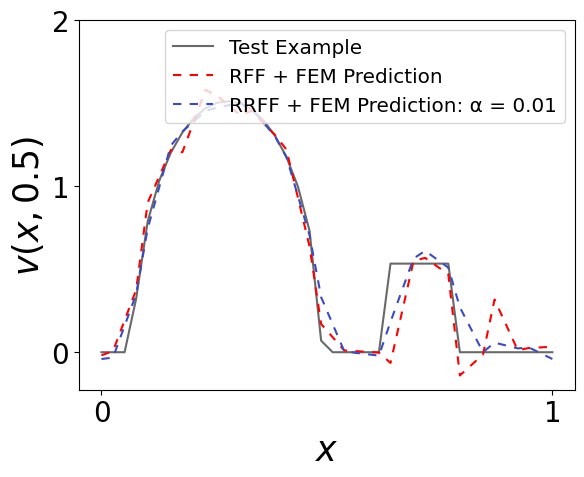}}
\subfigure{\includegraphics[width=39mm]{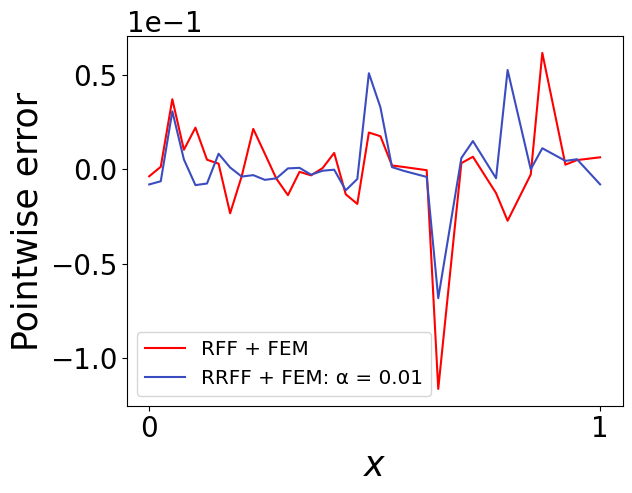}}
\caption{\textbf{Advection \RomanNumeral{2}}: (left to right) (a) test example and predictions of $R_\mathcal{V}\circ\hat{f}$ using RFF-FEM-2 and RRFF-FEM-2, (b) pointwise errors for predictions of $R_\mathcal{V}\circ\hat{f}$ using RFF-FEM-2 and RRFF-FEM-2, (c) test example and predictions of $R_\mathcal{V}\circ\hat{f}$ using RFF-FEM-3 and RRFF-FEM-3, (d) pointwise errors for predictions of $R_\mathcal{V}\circ\hat{f}$ using RFF-FEM-3 and RRFF-FEM-3.}
\label{Fig:AD2_recovery_Student}
\end{figure}

In \cite{dehoop2022costaccuracy}, the initial condition, referred to as Advection \RomanNumeral{3} here, is:
\begin{equation*}
    u(x) = -1 + 2 \mathbbm{1}_{\{\tilde{u_0}\geq 0\}},
\end{equation*}
where $\tilde{u_0}$ is generated from a Gaussian process, $\mathcal{GP}(0,(-\Delta+9\mathbf{I})^{-2})$, $\Delta$ denotes the Laplacian on $D_\mathcal{U} = (0,1)$ subject to periodic boundary conditions, and $\mathbf{I}$ denotes the identity matrix.

For the RFF-FEM and RRFF-FEM methods, a low resolution grid with nonuniform spacing and 134 grid points is employed. In the training phase of $\hat{f}$, 1000 samples are used, and in its testing stage, 800 samples are used. 200 samples are held-out to test the performance of $R_\mathcal{V}\circ\hat{f}$. We show an example of training input with 5\% noise, training output with 5\% noise, test examples and predictions of $R_\mathcal{V}\circ \hat{f}$ using RFF-FEM-2, RRFF-FEM-2, RFF-FEM-3, RRFF-FEM-3, RFF-FEM-$\infty$, and RRFF-FEM-$\infty$, and their respective pointwise errors in Figure \ref{Fig:AD3_train_recovery_Gaussian} and Figure \ref{Fig:AD3_recovery_Student}. In all cases, we note the improved reduction of noise when using the RRFF approach. 

\begin{figure}[!t]
\centering
\subfigure{\includegraphics[width=36mm]{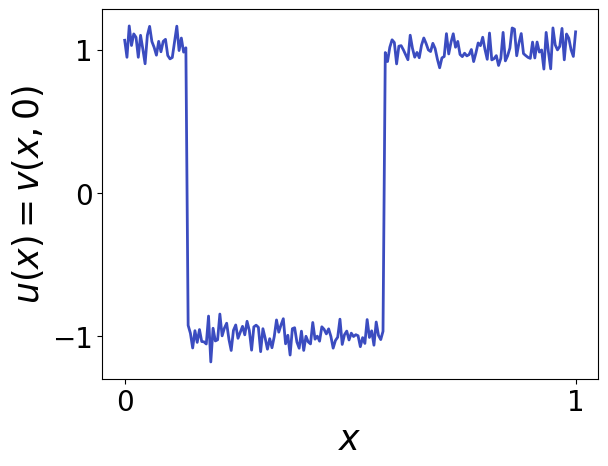}}
\subfigure{\includegraphics[width=36mm]{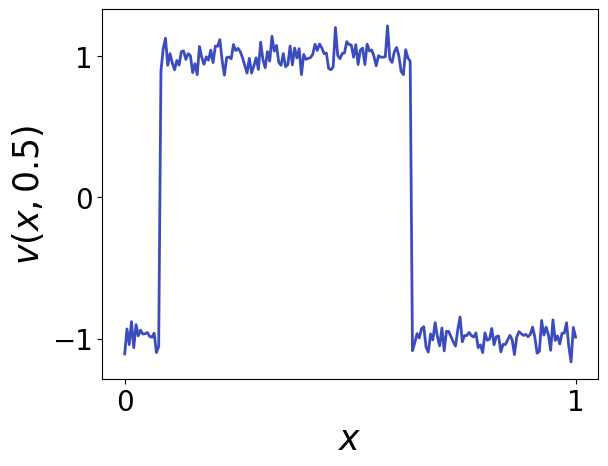}}
\subfigure{\includegraphics[width=36mm]{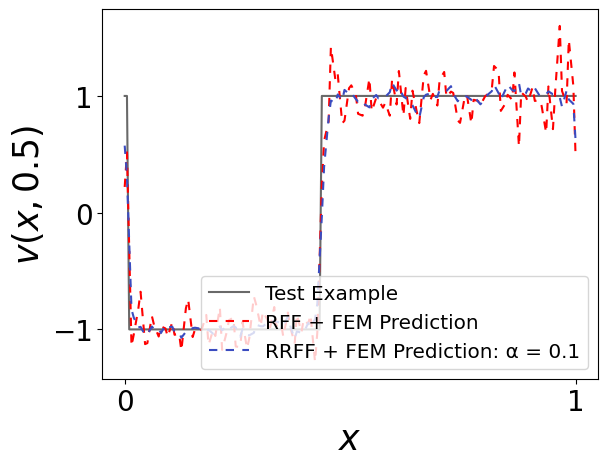}}
\subfigure{\includegraphics[width=37mm]{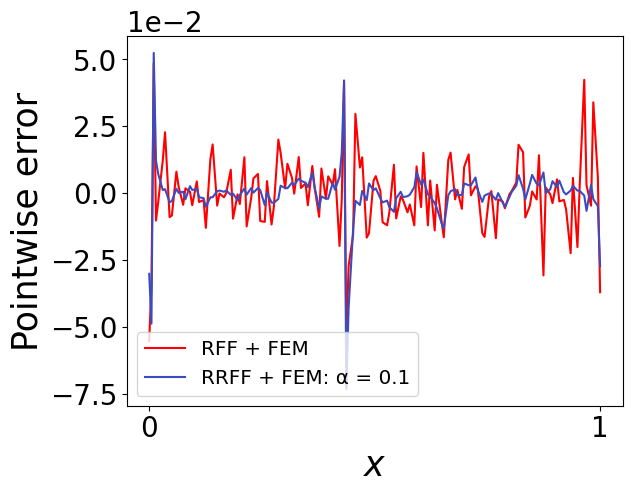}}
\caption{\textbf{Advection \RomanNumeral{3}}: (left to right) An example of (a) training input with 5\% noise, (b) training output with 5\% noise, (c) test example and predictions of $R_\mathcal{V}\circ\hat{f}$ using RFF-FEM-$\infty$ and RRFF-FEM-$\infty$, (d) pointwise errors for predictions of $R_\mathcal{V}\circ\hat{f}$ using RFF-FEM-$\infty$ and RRFF-FEM-$\infty$.}
\label{Fig:AD3_train_recovery_Gaussian}
\end{figure}

\begin{figure}[!htbp]
\centering     
\subfigure{\includegraphics[width=36mm]{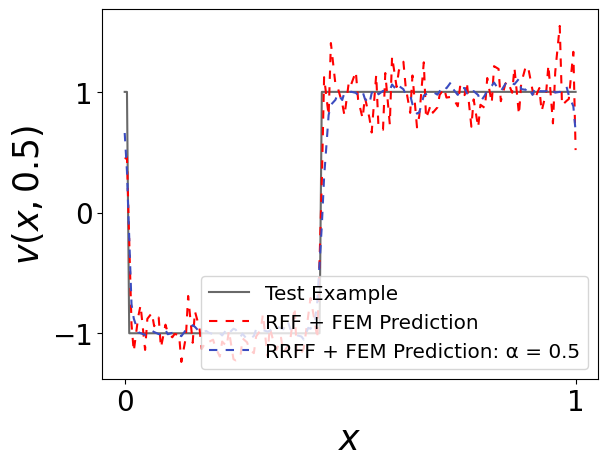}}
\subfigure{\includegraphics[width=37mm]{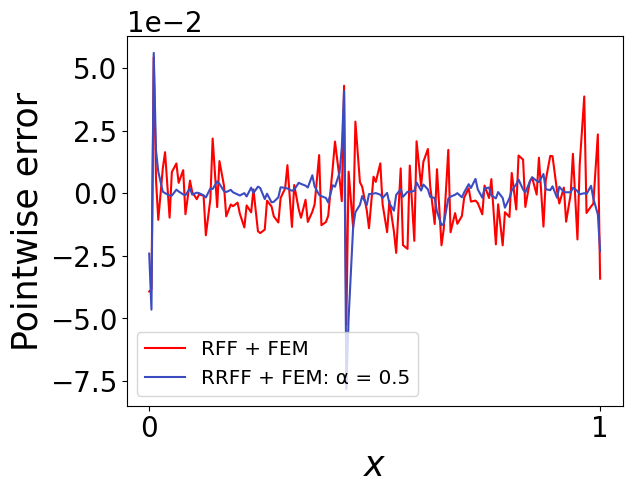}}
\subfigure{\includegraphics[width=36mm]{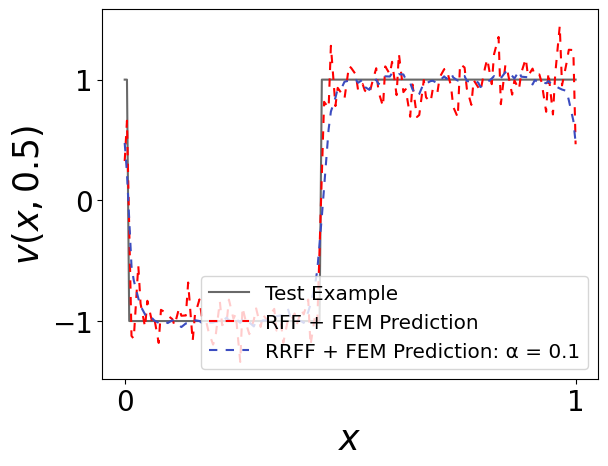}}
\subfigure{\includegraphics[width=36mm]{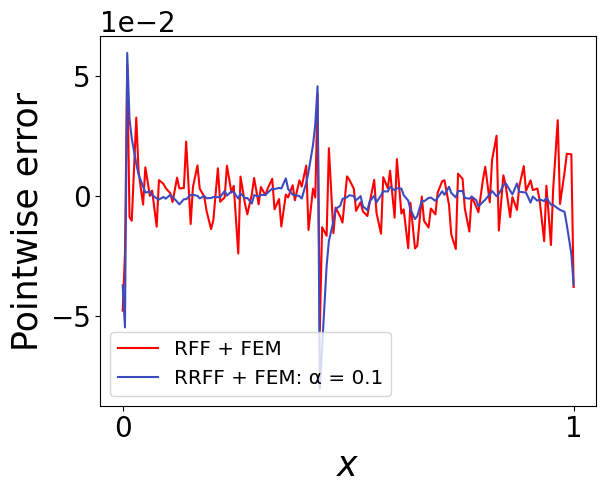}}
\caption{\textbf{Advection \RomanNumeral{3}}: (left to right) (a) test example and predictions of $R_\mathcal{V}\circ\hat{f}$ using RFF-FEM-2 and RRFF-FEM-2, (b) pointwise errors for predictions of $R_\mathcal{V}\circ\hat{f}$ using RFF-FEM-2 and RRFF-FEM-2, (c) test example and predictions of $R_\mathcal{V}\circ\hat{f}$ using RFF-FEM-3 and RRFF-FEM-3, (d) pointwise errors for predictions of $R_\mathcal{V}\circ\hat{f}$ using RFF-FEM-3 and RRFF-FEM-3.}
\label{Fig:AD3_recovery_Student}
\end{figure}

\subsection{Burgers' Equation}
Consider the one-dimensional Burgers' equation:
\begin{equation*}
\begin{aligned}
    \dfrac{\partial w}{\partial t} + w \dfrac{\partial w}{\partial x} &= \mu \dfrac{\partial^2 w}{\partial x^2}, \hspace{0.5cm} &&(x,t) \in (0,1)\times(0,1] \\
    w(x,0) &= u(x), \hspace{0.5cm} &&x\in[0,1]
\end{aligned}
\end{equation*}
with periodic boundary conditions, where we set the viscosity parameter $\mu = 0.1$. We learn the operator  $G: u(x) \mapsto w(x,1)$, where $u(x) = w(x,0)$ is the initial condition and $w(x,1)$ is the solution at time $t=1$. We have $D_\mathcal{U} = D_\mathcal{V} = (0,1)$. As in \cite{lu2022comprehensive}, the initial condition $u(x)$ is generated from a Gaussian process $\mathcal{GP}(0,625(-\Delta+25\mathbf{I})^{-2})$ with $\Delta$ being the Laplacian on $D_\mathcal{U} = (0,1)$ subject to periodic boundary conditions and $\mathbf{I}$ being the identity matrix. For the RFF and RRFF models, the grid uses 128 uniformly spaced points. The training set uses about 1.8k instances as input-output pairs  and 200 instances to test the performance.

For RFF-FEM and RRFF-FEM, a coarser nonuniform grid with 86 grid points was used. Around 1.6k samples are used for training $\hat{f}$, and 200 samples are used to test its performance. The performance of $R_\mathcal{V}\circ\hat{f}$ is tested on the remaining 200 samples. In Table \ref{burgers_param_table}, the parameters for the numerical experiments are shown. Figure \ref{Fig:Burger_train_alpha} shows an example of training input and output corrupted with 5\% noise. We tested a range of values for the regularization parameter $\alpha$ to find the optimal $\alpha$ such that the RRFF method minimizes the test error. Figure \ref{Fig:Burger_train_alpha} includes plots of $\log_{10}(\alpha)$ versus the average relative test error of the RRFF-2, RRFF-3, and RRFF-$\infty$ methods over 20 trials. In all the cases, as $\alpha$ grows, the test error decreases until it reaches a minimum, after which it begins to increase. This behavior indicates that model overfits and learns the noise when there is no regularization, but excessive regularization causes the model to disregard the data and drive the coefficients toward zero, resulting in underfitting. Figure \ref{Fig:Burger_test_pred_Gaussian_Student2} presents examples of test functions and the RFF-2, RRFF-2, RFF-$\infty$, and RRFF-$\infty$ predictions of $\hat{f}$, as well as their pointwise errors. Figure \ref{Fig:Burger_test_pred_Student_recovery_Gaussian} displays plots of a test function and the RFF-3 and RRFF-3 approximations of $\hat{f}$, along with the corresponding pointwise errors. Additionally, this figure shows plots of a test example and predictions of $R_\mathcal{V}\circ \hat{f}$ obtained from RFF-FEM-$\infty$ and RRFF-FEM-$\infty$ with their associated pointwise errors. In Figure \ref{Fig:Burger_recovery_Student}, test functions and predictions of $R_\mathcal{V}\circ\hat{f}$ using the RFF-FEM-2, RRFF-FEM-2, RFF-FEM-3, and RRFF-FEM-3 methods with their respective pointwise errors are shown. Both the RFF-FEM and RRFF-FEM methods produces accurate solutions, with the RRFF approach also reducing the noise. 

\begin{table}[!htbp]
\centering
    \scriptsize	
    \begin{tabular}{{|P{0.13\linewidth} | P{0.08\linewidth} | P{0.08\linewidth} |
    P{0.08\linewidth} | P{0.08\linewidth} |}}
\hline
\addstackgap{Distribution} & \addstackgap{$N$} & \addstackgap{$\sigma$} & \addstackgap{$\alpha$} & \addstackgap{$p$} 
\\ \hline \addstackgap{Gaussian} & 10k & 0.2 & 0.1 & 2  \\ \hline \addstackgap{Student ($\nu=2$)} & 10k & 0.02 & 0.25 & 2 \\ \hline \addstackgap{Student ($\nu=3$)} & 10k & 0.02 & 0.01 & 2 \\
\hline
\end{tabular}
\caption{\textbf{Burgers' Equation}: Parameters for the numerical experiments, where $N$ is the number of random features and $\sigma$ is the scale parameter for the RFF, RRFF, RFF-FEM, and RRFF-FEM methods, and $\alpha$ and $p$ are regularization parameters for the RRFF and RRFF-FEM methods.}
\label{burgers_param_table}
\end{table}

\begin{figure}[!htbp]
\centering
\subfigure{\includegraphics[width=36mm]{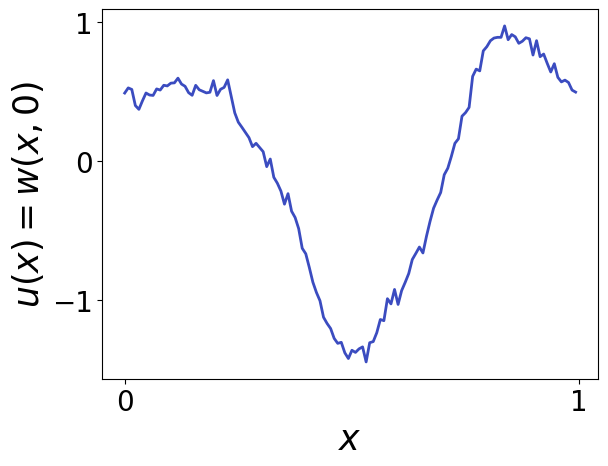}}
\subfigure{\includegraphics[width=36mm]{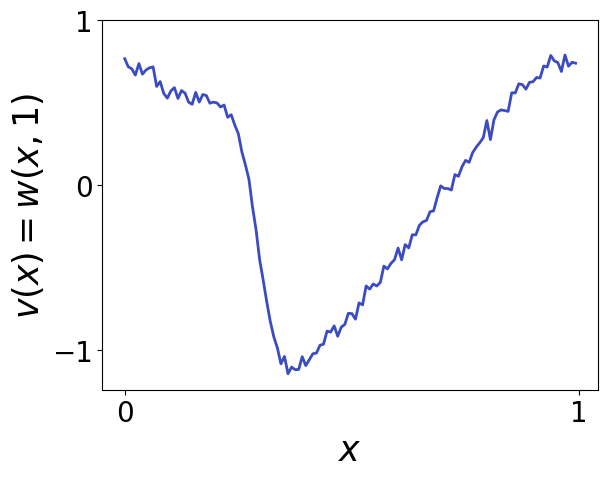}}
\subfigure{\includegraphics[width=34mm]{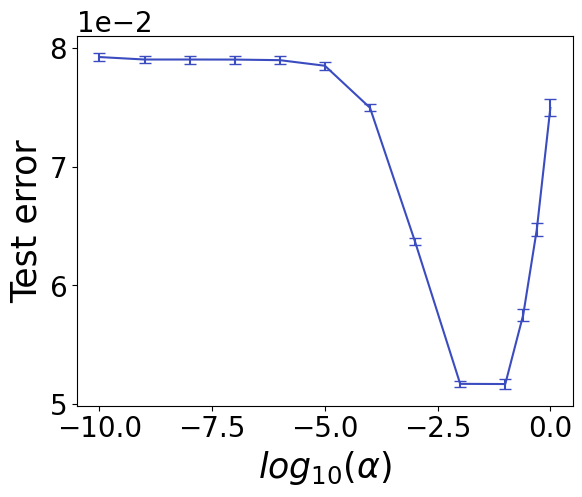}}
\subfigure{\includegraphics[width=36mm]{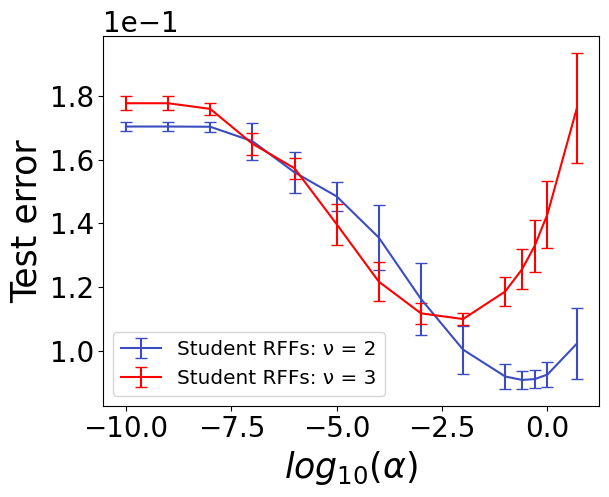}}
\caption{\textbf{Burgers' Equation}: (left to right) An example of (a) training input with 5\% noise, (b) training output with 5\% noise, (c) $\log_{10}(\alpha)$ versus average relative test error for RRFF-$\infty$ over 20 trials, (d) $\log_{10}(\alpha)$ versus average relative test error for RRFF-2 and RRFF-3 over 20 trials.}
\label{Fig:Burger_train_alpha}
\end{figure}

\begin{figure}[!htbp]
\centering     
\subfigure{\includegraphics[width=36mm]{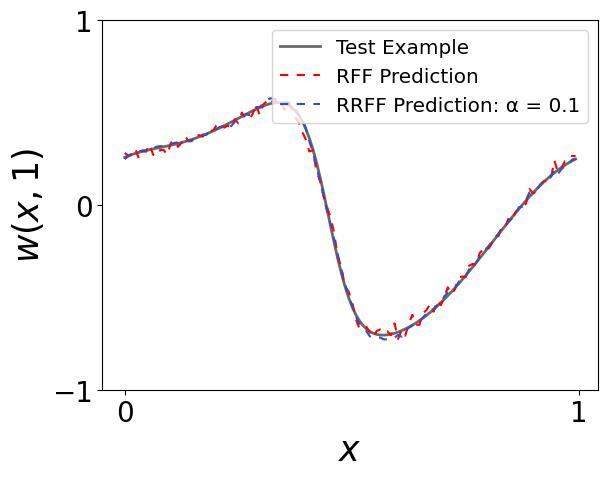}}
\subfigure{\includegraphics[width=36mm]{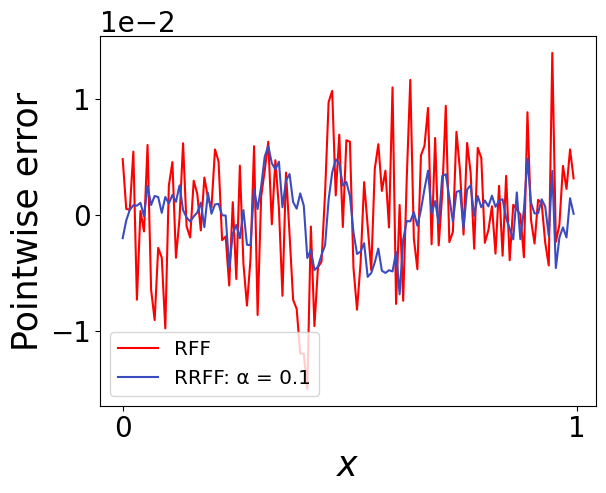}}
\subfigure{\includegraphics[width=36mm]{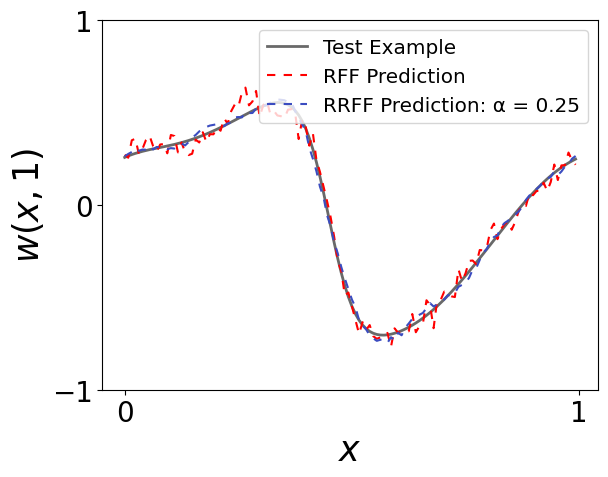}}
\subfigure{\includegraphics[width=36mm]{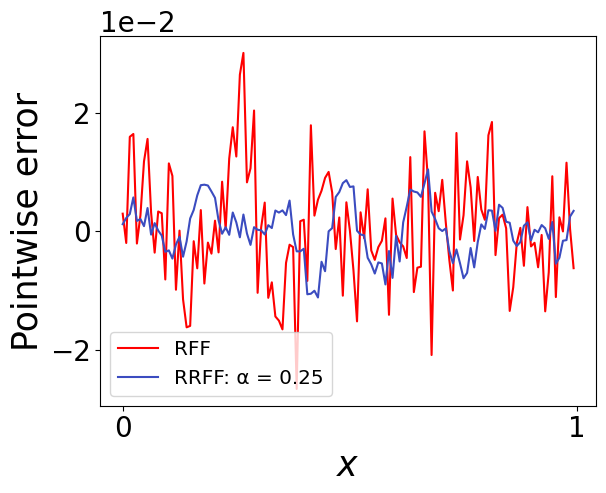}}
\caption{\textbf{Burgers' Equation}: (left to right) (a) test example and RFF-$\infty$ and RRFF-$\infty$ approximations of $\hat{f}$, (b) pointwise errors for RFF-$\infty$ and RRFF-$\infty$ approximations of $\hat{f}$, (c) test example and RFF-2 and RRFF-2 approximations of $\hat{f}$, (d) pointwise errors for RFF-2 and RRFF-2 approximations of $\hat{f}$.}
\label{Fig:Burger_test_pred_Gaussian_Student2}
\end{figure}

\begin{figure}[!htbp]
\centering     
\subfigure{\includegraphics[width=36mm]{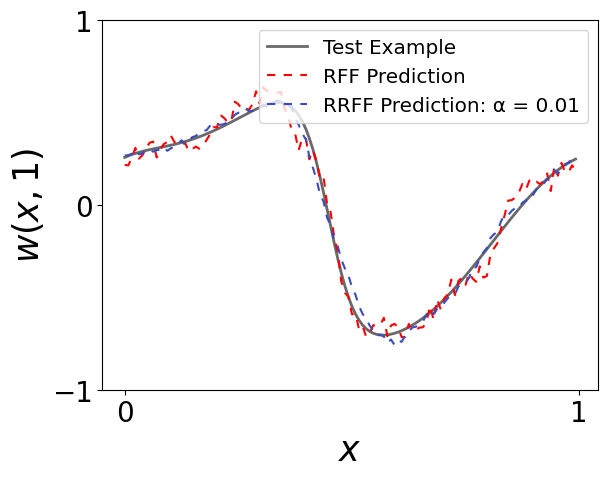}}
\subfigure{\includegraphics[width=36mm]{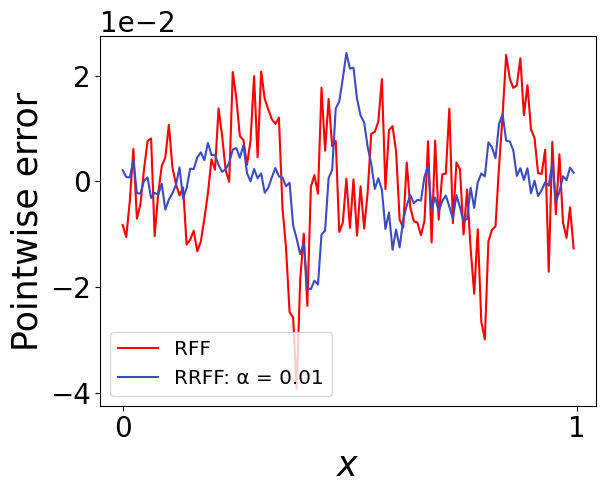}}
\subfigure{\includegraphics[width=36mm]{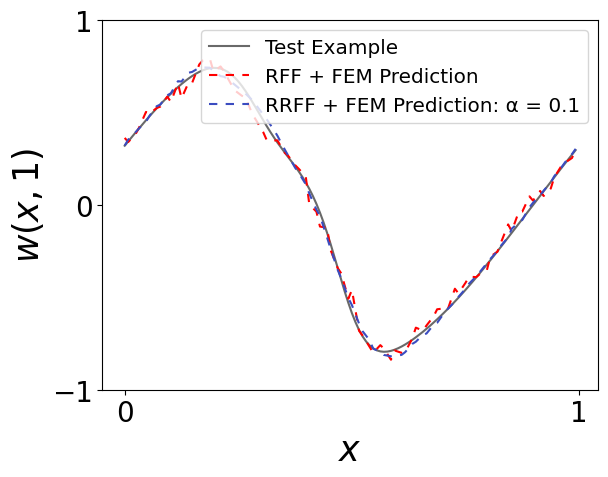}}
\subfigure{\includegraphics[width=36mm]{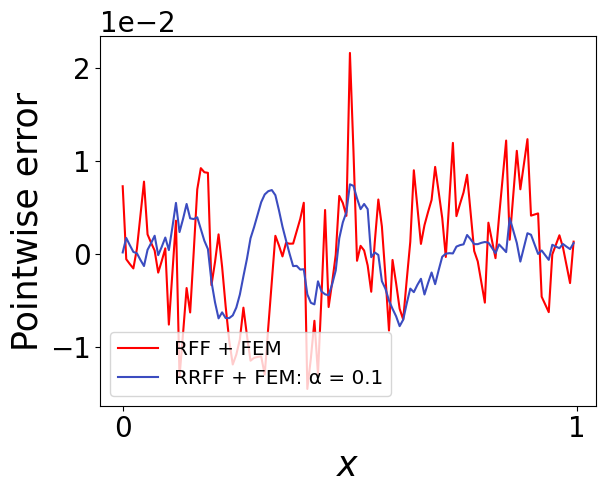}}
\caption{\textbf{Burgers' Equation}: (left to right) (a) test example and RFF-3 and RRFF-3 approximations of $\hat{f}$, (b) pointwise errors for RFF-3 and RRFF-3 approximations of $\hat{f}$, (c) test example and predictions of $R_\mathcal{V}\circ\hat{f}$ using RFF-FEM-$\infty$ and RRFF-FEM-$\infty$, (d) pointwise errors for predictions of $R_\mathcal{V}\circ\hat{f}$ using RFF-FEM-$\infty$ and RRFF-FEM-$\infty$.}
\label{Fig:Burger_test_pred_Student_recovery_Gaussian}
\end{figure}

\begin{figure}[!htbp]
\centering     
\subfigure{\includegraphics[width=36mm]{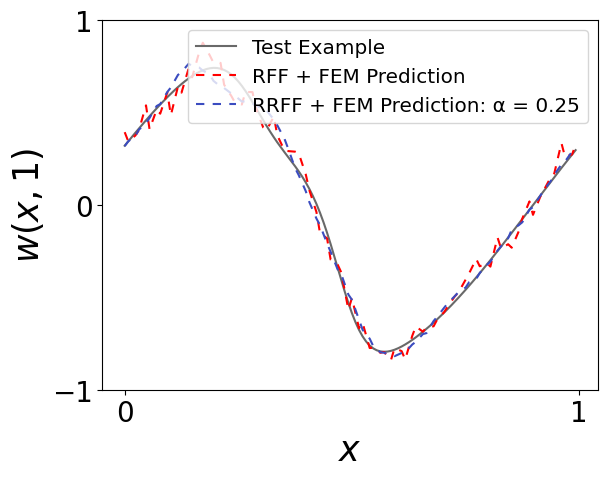}}
\subfigure{\includegraphics[width=36mm]{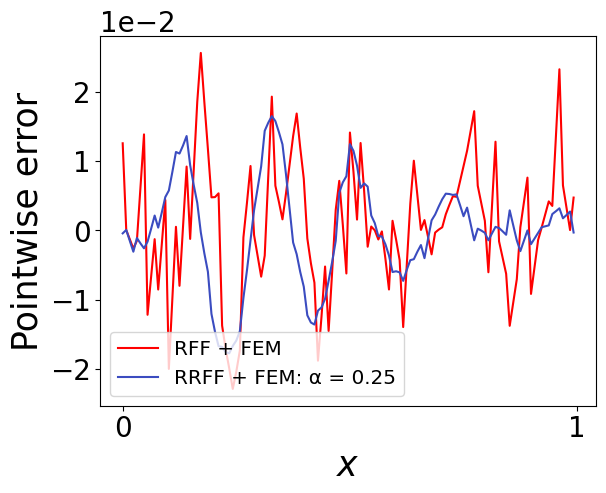}}
\subfigure{\includegraphics[width=36mm]{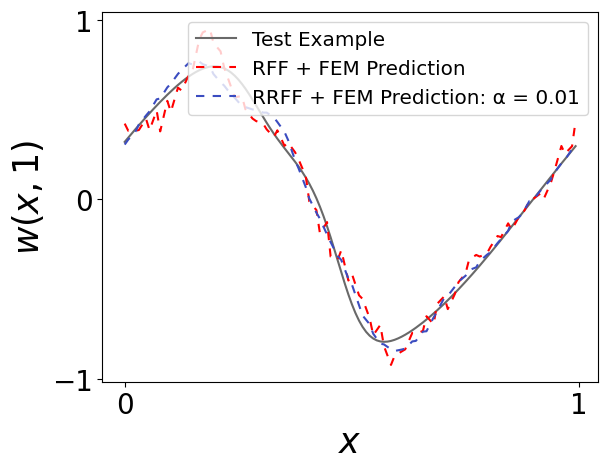}}
\subfigure{\includegraphics[width=36mm]{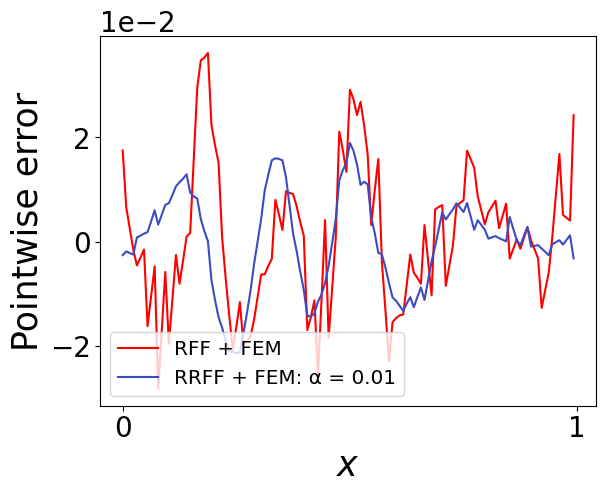}}
\caption{\textbf{Burgers' Equation}: (left to right) (a) test example and predictions of $R_\mathcal{V}\circ\hat{f}$ using RFF-FEM-2 and RRFF-FEM-2, (b) pointwise errors for predictions of $R_\mathcal{V}\circ\hat{f}$ using RFF-FEM-2 and RRFF-FEM-2, (c) test example and predictions of $R_\mathcal{V}\circ\hat{f}$ using RFF-FEM-3 and RRFF-FEM-3, (d) pointwise errors for predictions of $R_\mathcal{V}\circ\hat{f}$ using RFF-FEM-3 and RRFF-FEM-3.}
\label{Fig:Burger_recovery_Student}
\end{figure} 

\subsection{Darcy Flow}
Consider the two-dimensional Darcy flow problem:
\begin{equation*}
\begin{aligned}
    -\text{div }e^u\nabla v &= w, &&\text{in } D \\
    v&=0,&&\text{on } \partial D,
\end{aligned}
\end{equation*}
where $D = (0,1)^2$ and the source term $w$ is fixed. We aim to learn the operator $G: u \mapsto v$, where $u$ is the permeability field and $v$ is the solution. The domains are defined as $D_\mathcal{U} = D_\mathcal{V} = D=(0,1)^2$. As described in \cite{lu2022comprehensive}, the permeability/diffusion coefficient $u$ is sampled by taking $u\sim\log(h(\beta))$, where $\beta=\mathcal{GP}(0,(-\Delta+9\mathbf{I})^{-2})$ is a Gaussian process and $h$ is a binary function mapping positive values to $12$ and negative values to $3$. Hence, $e^u$ is a piecewise constant permeability field. The RFF and RRFF methods use a $29\times 29$ uniform grid for $u$ and $v$, 1000 samples for training, and 200 samples for testing.

A lower resolution nonuniform grid on $(0,1)^2$ with 561 grid points is used for RFF-FEM and RRFF-FEM. We train $\hat{f}$ with 800 samples and test its performance with 200 samples. The performance of $R_\mathcal{V}\circ\hat{f}$ is evaluated on a hold-out set of 200 samples. Table \ref{darcy_param_table} presents the parameters used in the numerical experiments. In Figure \ref{Fig:Darcy_train_alpha}, an example of training input and output with 5\% noise is shown. The figure also includes plots of $\log_{10}(\alpha)$ versus the average relative test errors for the RRFF-2, RRFF-3, and RRFF-$\infty$ methods over 20 runs. For RRFF-$\infty$, there is a clear distinction between the model's overfitting and underfitting regimes for the range of $\alpha$'s shown. However, for RRFF-2 and RRFF-3, as $\alpha$ increases, the test error decreases before reaching a plateau, which indicates that the model will not continue oversmoothing the noise in this range. Figures \ref{Fig:Darcy_test_pred_Gaussian}, \ref{Fig:Darcy_test_pred_Student2}, and \ref{Fig:Darcy_test_pred_Student3} plot examples of test functions, the RFF-$\infty$, RRFF-$\infty$, RFF-2, RRFF-2, RFF-3, and RRFF-3 approximations of $\hat{f}$, respectively, and their pointwise errors. Figures \ref{Fig:Darcy_recovery_Gaussian}, \ref{Fig:Darcy_recovery_Student2}, and \ref{Fig:Darcy_recovery_Student3} present test examples, the RFF-FEM-$\infty$, RRFF-FEM-$\infty$, RFF-FEM-2, RRFF-FEM-2, RFF-FEM-3, and RRFF-FEM-3 predictions of $R_\mathcal{V}\circ\hat{f}$, respectively, alongside the corresponding pointwise errors.

\begin{table}[!htbp]
\centering
    \scriptsize	
    \begin{tabular}{{|P{0.13\linewidth} | P{0.08\linewidth} | P{0.08\linewidth} |
    P{0.08\linewidth} | P{0.08\linewidth} |}}
\hline
\addstackgap{Distribution} & \addstackgap{$N$} & \addstackgap{$\sigma$} & \addstackgap{$\alpha$} & \addstackgap{$p$} 
\\ \hline \addstackgap{Gaussian} & \addstackgap{100k} & \addstackgap{$\sqrt{2\times 10^{-5}}$} & \addstackgap{0.1} & \addstackgap{2}  \\ \hline \addstackgap{Student ($\nu=2$)} & \addstackgap{100k} & \addstackgap{0.02} & \addstackgap{0.5} & \addstackgap{2} \\ \hline \addstackgap{Student ($\nu=3$)} & \addstackgap{100k} & \addstackgap{0.02} & \addstackgap{0.5} & \addstackgap{2} \\
\hline
\end{tabular}
\caption{\textbf{Darcy Flow}: Parameters for the numerical experiments, where $N$ is the number of random features and $\sigma$ is the scale parameter for the RFF, RRFF, RFF-FEM, and RRFF-FEM methods, and $\alpha$ and $p$ are regularization parameters for the RRFF and RRFF-FEM methods.}
\label{darcy_param_table}
\end{table}

\begin{figure}[!htbp]
\centering
\subfigure{\includegraphics[width=66mm]{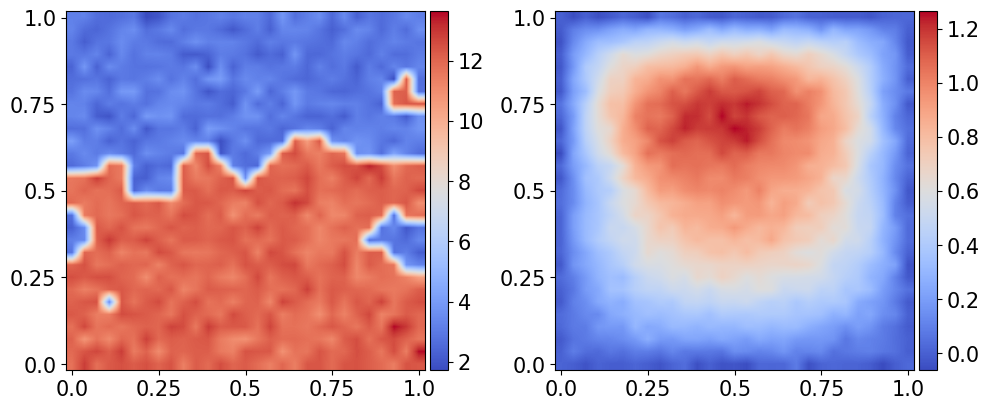}}
\subfigure{\includegraphics[width=36mm]{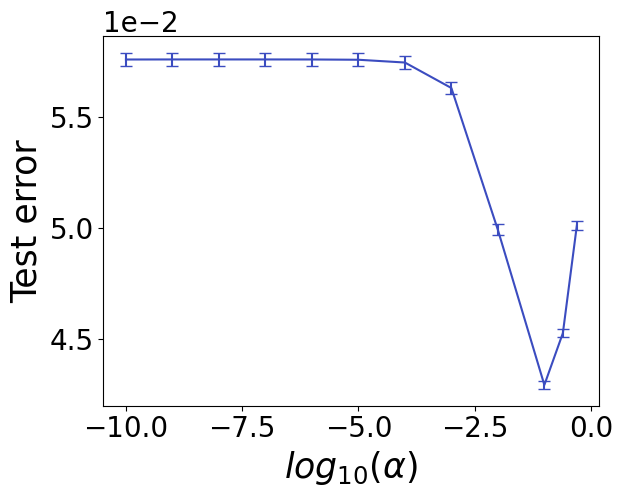}}
\subfigure{\includegraphics[width=36mm]{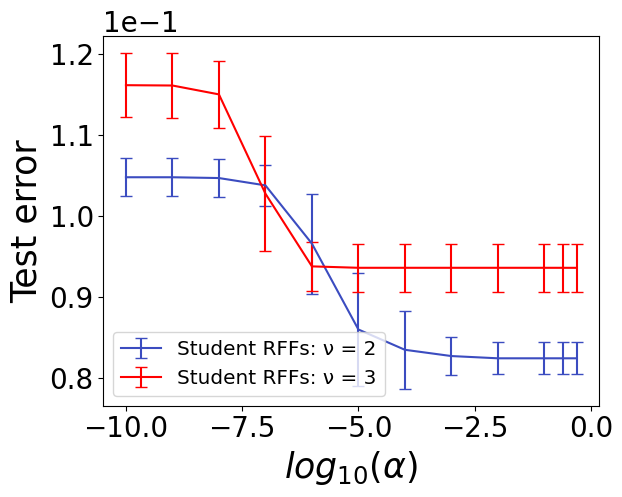}}
\caption{\textbf{Darcy Flow}: (left to right) An example of (a) training input with 5\% noise, (b) training output with 5\% noise, (c) $\log_{10}(\alpha)$ versus average relative test error for RRFF-$\infty$ over 20 trials, (d) $\log_{10}(\alpha)$ versus average relative test error for RRFF-2 and RRFF-3 over 20 trials.}
\label{Fig:Darcy_train_alpha}
\end{figure}

\begin{figure}[!htbp]
\centering     
\subfigure{\includegraphics[width=97.5mm]{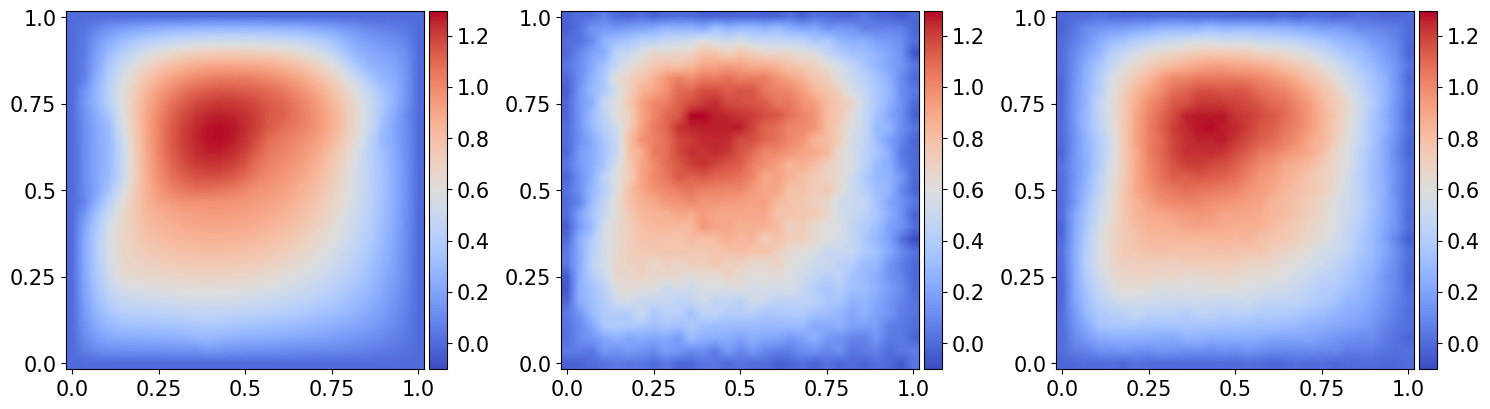}}
\subfigure{\includegraphics[width=65.5mm]{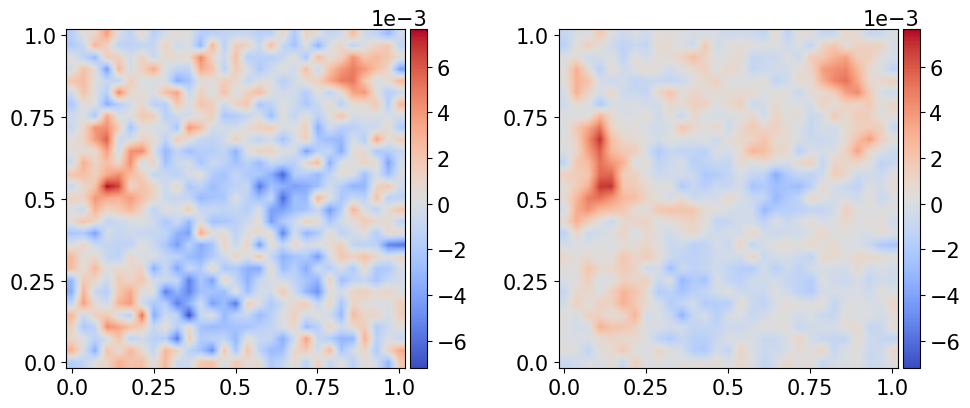}}
\caption{\textbf{Darcy Flow}: (left to right) (a) test example, (b) RFF-$\infty$ approximation of $\hat{f}$, (c) RRFF-$\infty$ approximation of $\hat{f}$, (d) pointwise error for RFF-$\infty$ approximation of $\hat{f}$, (e) pointwise error for RRFF-$\infty$ approximation of $\hat{f}$.}
\label{Fig:Darcy_test_pred_Gaussian}
\end{figure}

\begin{figure}[!htbp]
\centering     
\subfigure{\includegraphics[width=97mm]{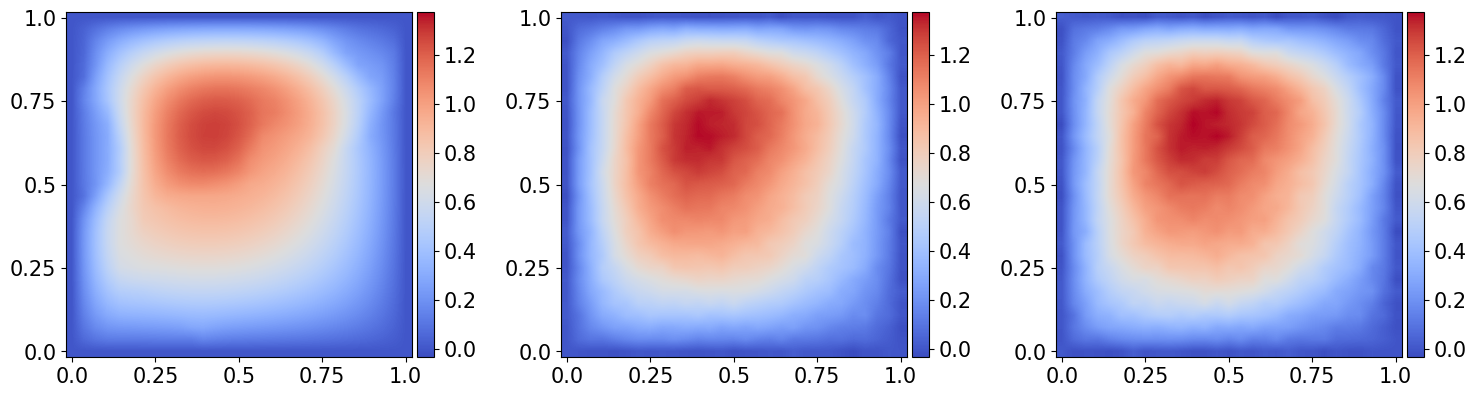}}
\subfigure{\includegraphics[width=66.5mm]{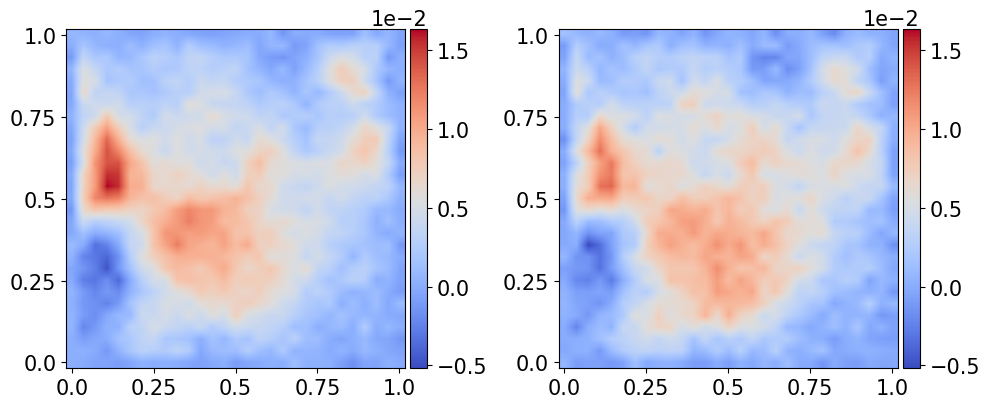}}
\caption{\textbf{Darcy Flow}: (left to right) (a) test example,  (b) RFF-2 approximation of $\hat{f}$, (c) RRFF-2 approximation of $\hat{f}$, (d) pointwise error for RFF-2 approximation of $\hat{f}$, (e) pointwise error for RRFF-2 approximation of $\hat{f}$.}
\label{Fig:Darcy_test_pred_Student2}
\end{figure}

\begin{figure}[!htbp]
\centering     
\subfigure{\includegraphics[width=97mm]{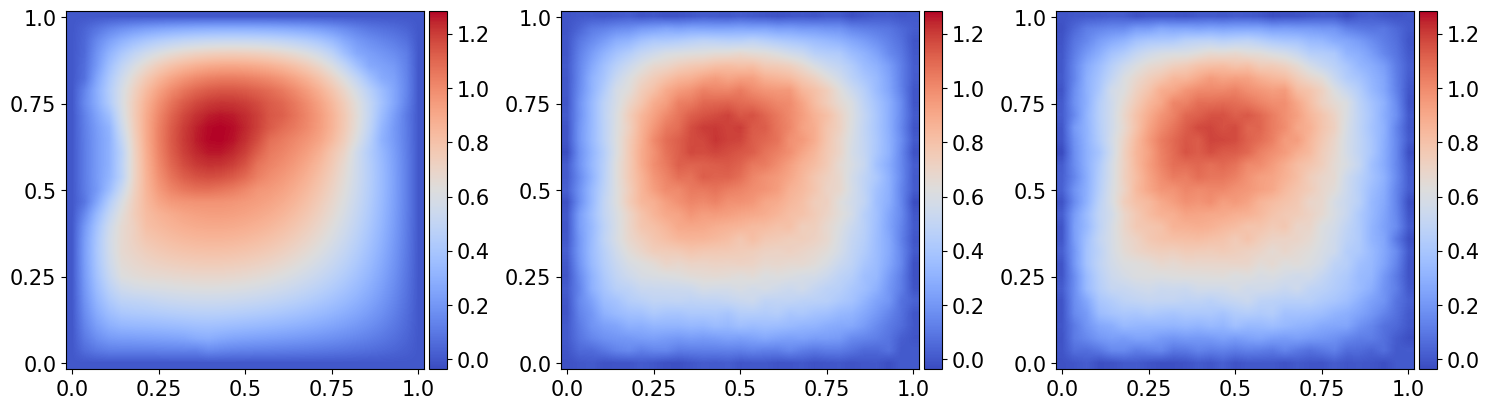}}
\subfigure{\includegraphics[width=66mm]{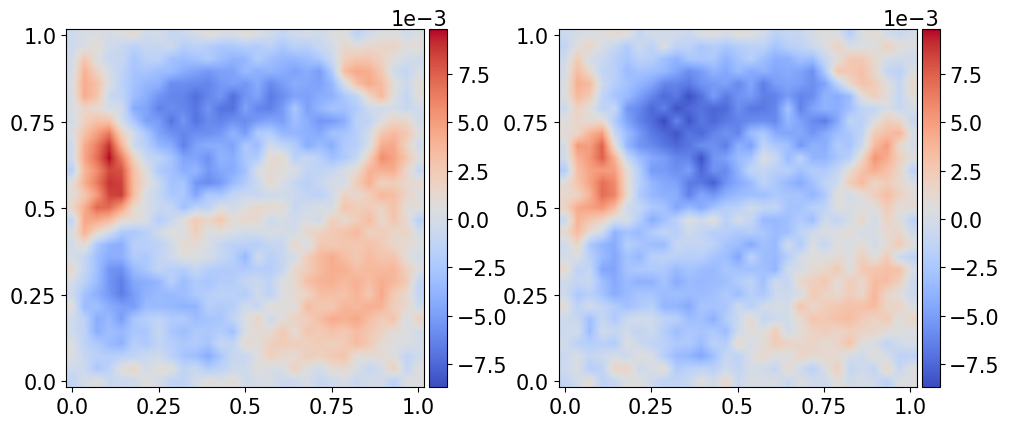}}
\caption{\textbf{Darcy Flow}: (left to right) (a) test example,  (b) RFF-3 approximation of $\hat{f}$, (c) RRFF-3 approximation of $\hat{f}$, (d) pointwise error for RFF-3 approximation of $\hat{f}$, (e) pointwise error for RRFF-3 approximation of $\hat{f}$.}
\label{Fig:Darcy_test_pred_Student3}
\end{figure}

\begin{figure}[!htbp]
\centering     
\subfigure{\includegraphics[width=98.5mm]{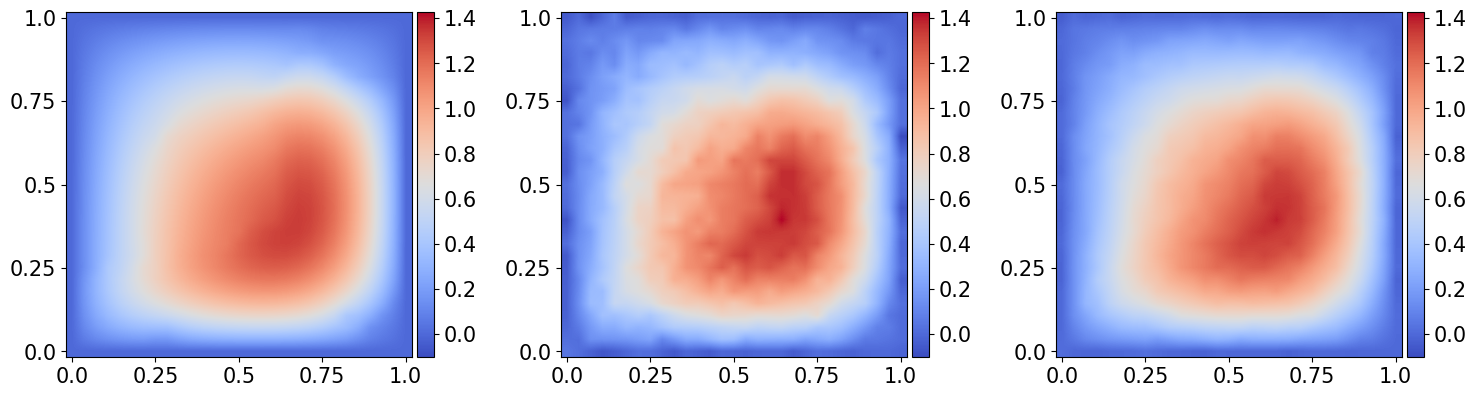}}
\subfigure{\includegraphics[width=64.5mm]{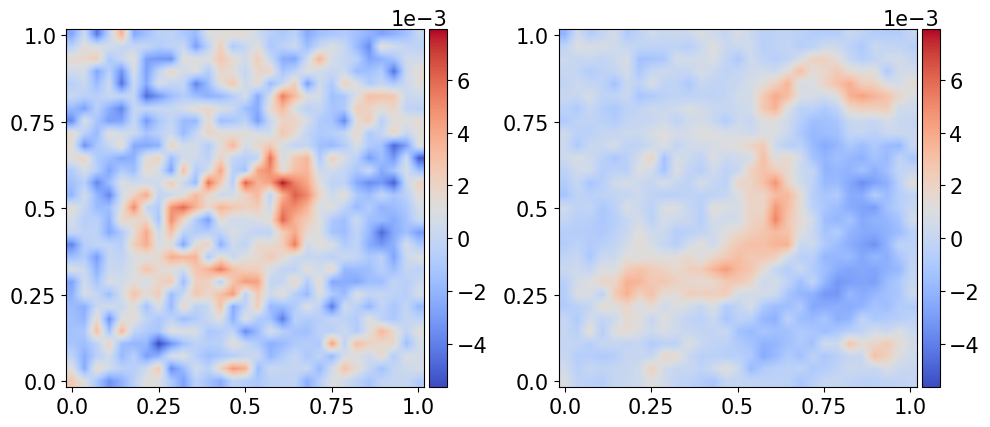}}
\caption{\textbf{Darcy Flow}: (left to right) (a) test example, (b) prediction of $R_\mathcal{V}\circ\hat{f}$ using RFF-FEM-$\infty$, (c) prediction of $R_\mathcal{V}\circ\hat{f}$ using RRFF-FEM-$\infty$, (d) pointwise error for prediction of $R_\mathcal{V}\circ\hat{f}$ using RFF-FEM-$\infty$, (e) pointwise error for prediction of $R_\mathcal{V}\circ\hat{f}$ using RRFF-FEM-$\infty$.}
\label{Fig:Darcy_recovery_Gaussian}
\end{figure}

\begin{figure}[!htbp]
\centering     
\subfigure{\includegraphics[width=98mm]{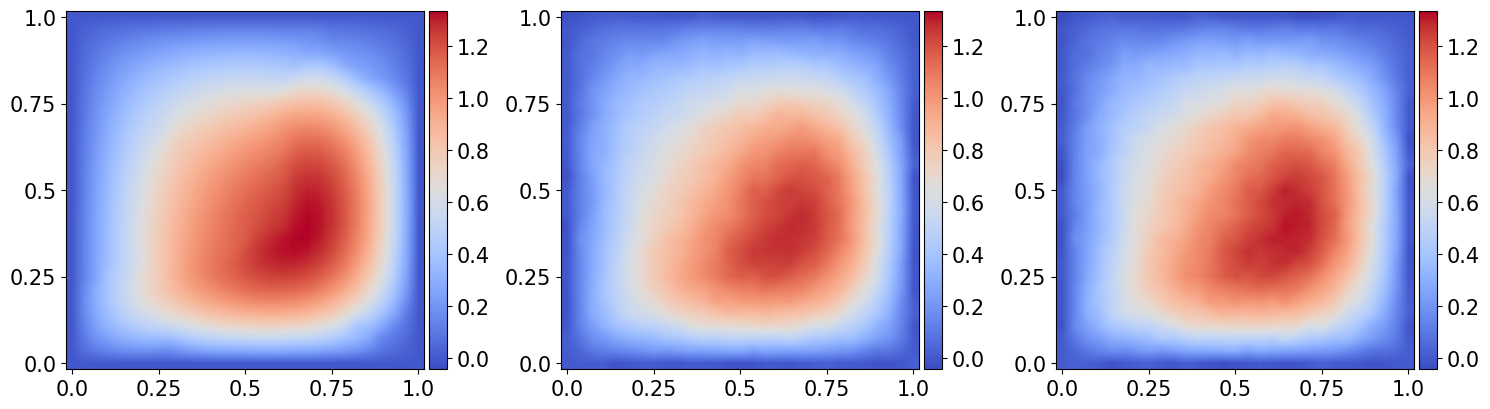}}
\subfigure{\includegraphics[width=65mm]{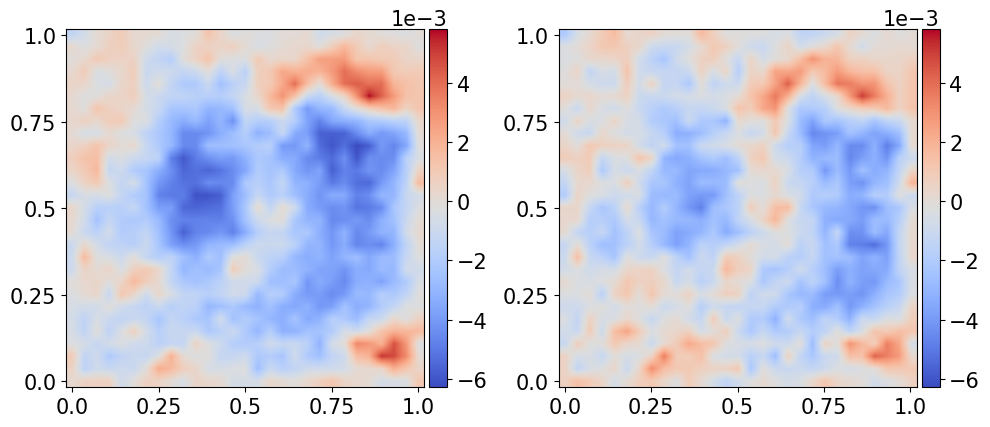}}
\caption{\textbf{Darcy Flow}: (left to right) (a) test example, (b) prediction of $R_\mathcal{V}\circ\hat{f}$ using RFF-FEM-2, (c) prediction of $R_\mathcal{V}\circ\hat{f}$ using RRFF-FEM-2, (d) pointwise error for prediction of $R_\mathcal{V}\circ\hat{f}$ using RFF-FEM-2, (e) pointwise error for prediction of $R_\mathcal{V}\circ\hat{f}$ using RRFF-FEM-2.}
\label{Fig:Darcy_recovery_Student2}
\end{figure}

\begin{figure}[!htbp]
\centering     
\subfigure{\includegraphics[width=98.5mm]{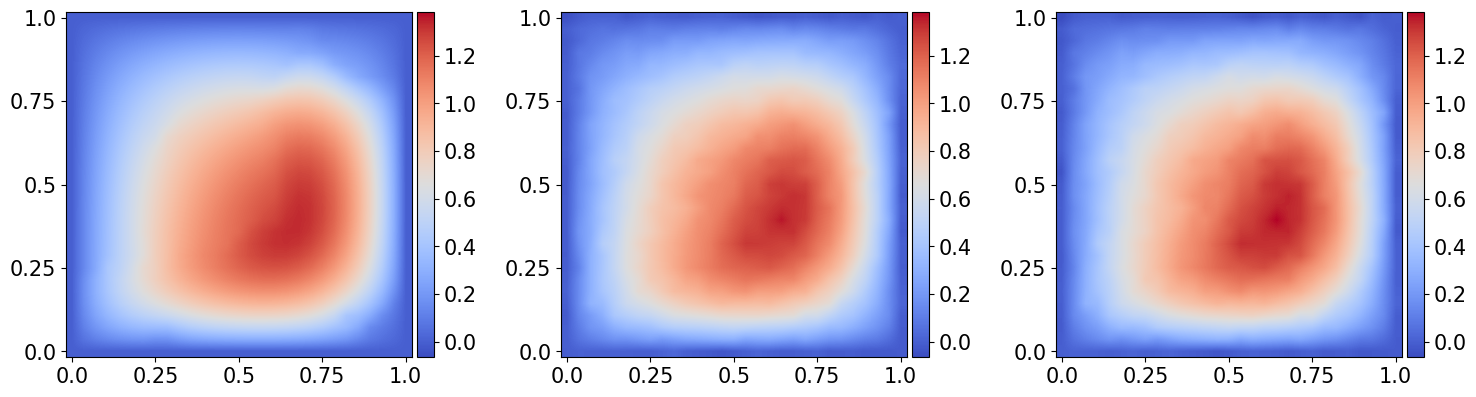}}
\subfigure{\includegraphics[width=64.5mm]{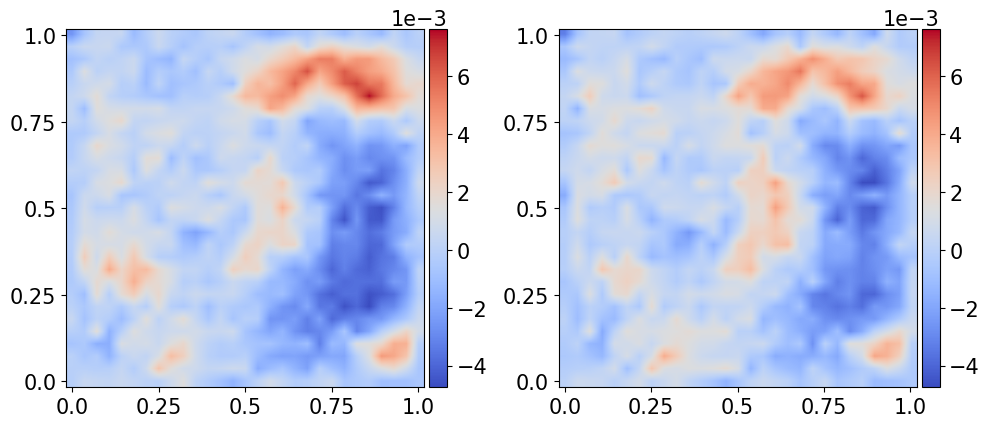}}
\caption{\textbf{Darcy Flow}: (left to right) (a) test example, (b) prediction of $R_\mathcal{V}\circ\hat{f}$ using RFF-FEM-3, (c) prediction of $R_\mathcal{V}\circ\hat{f}$ using RRFF-FEM-3, (d) pointwise error for prediction of $R_\mathcal{V}\circ\hat{f}$ using RFF-FEM-3, (e) pointwise error for prediction of $R_\mathcal{V}\circ\hat{f}$ using RRFF-FEM-3.}
\label{Fig:Darcy_recovery_Student3}
\end{figure}

\subsection{Helmholtz Equation}
\label{helmholtz_eq}
Consider the Helmholtz equation on $D=(0,1)^2$. Given a fixed frequency $\omega$ and wavespeed field $u:D\to\mathbb{R}$, the excitation field $v:D\to\mathbb{R}$ solves the PDE:
\begin{equation}\label{helmholtz}
    \begin{aligned}
        \bigg(-\Delta-\dfrac{\omega^2}{u^2(x)}\bigg)v &= 0, &&\text{in } D\\
        \dfrac{\partial v}{\partial n} &= 0, &&\text{on } \partial D_1\\
        \dfrac{\partial v }{\partial n} &= v_N, &&\text{on } \partial D_2,
    \end{aligned}
\end{equation}
where $\partial D_1 = \{0,1\} \times[0,1] \cup [0,1]\times\{0\}$ and $\partial D_2 = [0,1]\times\{1\}$. As in \cite{dehoop2022costaccuracy}, we take $\omega = 10^3$ and $v_N = \mathbbm{1}_{\{0.35\leq x \leq 0.65\}}$. We learn the operator $G: u\mapsto v$, i.e., the map from the wavespeed field $u$ to the excitation field $v$. The domains of the function spaces are $D_\mathcal{U} = D_\mathcal{V} = D=(0,1)^2$. As detailed in \cite{dehoop2022costaccuracy}, the wavespeed field $u$ has the form $u(x) = 20 + \tanh{(\tilde{u}(x))}$, where $\tilde{u}\sim\mathcal{GP}(0,(-\Delta+9\mathbf{I})^{-2})$, a Gaussian process. A discretized grid of size $101\times 101$ is used for $u$. Samples for the excitation field $v$ are generated by solving (\ref{helmholtz}) using a finite element method on a $101\times 101$ grid. The data was corrupted by 1\% relative Gaussian noise.

For RFF-FEM and RRFF-FEM, a coarser nonuniform grid on $(0,1)^2$ with about 6.8k grid points is utilized. The training set of $\hat{f}$ contains 10k samples, and the test set for $\hat{f}$ has 25k samples. We test the performance of $R_\mathcal{V}\circ\hat{f}$ on the remaining 5k samples. The parameters for the numerical experiments are listed in Table \ref{helmholtz_param_table}. Figure \ref{Fig:Helmholtz_training_data} shows an example of training input and output corrupted with 1\% noise. In Figures \ref{Fig:Helmholtz_recovery_Gaussian}, \ref{Fig:Helmholtz_recovery_Student2}, and \ref{Fig:Helmholtz_recovery_Student3}, plots of test examples, approximations of $R_\mathcal{V}\circ\hat{f}$ using RFF-FEM-$\infty$, RRFF-FEM-$\infty$, RFF-FEM-2, RRFF-FEM-2, RFF-FEM-3, and RRFF-FEM-3, respectively, and their associated pointwise errors are presented.

\begin{table}[!htbp]
\centering
    \scriptsize	
    \begin{tabular}{{|P{0.13\linewidth} | P{0.08\linewidth} | P{0.08\linewidth} |
    P{0.08\linewidth} | P{0.08\linewidth} |}}
\hline
\addstackgap{Distribution} & \addstackgap{$N$} & \addstackgap{$\sigma$} & \addstackgap{$\alpha$} & \addstackgap{$p$} 
\\ \hline \addstackgap{Gaussian} & \addstackgap{15k} & \addstackgap{$\sqrt{2\times 10^{-5}}$} & \addstackgap{0.01} & \addstackgap{2}  \\ \hline \addstackgap{Student ($\nu=2$)} & \addstackgap{15k} & \addstackgap{$10^{-5}$} & \addstackgap{0.01} & \addstackgap{2} \\ \hline \addstackgap{Student ($\nu=3$)} & \addstackgap{15k} & \addstackgap{$10^{-5}$} & \addstackgap{0.01} & \addstackgap{4} \\
\hline
\end{tabular}
\caption{\textbf{Helmholtz Equation}: Parameters for the numerical experiments, where $N$ is the number of random features and $\sigma$ is the scale parameter for the RFF, RRFF, RFF-FEM, and RRFF-FEM methods, and $\alpha$ and $p$ are regularization parameters for the RRFF and RRFF-FEM methods.}
\label{helmholtz_param_table}
\end{table}

\begin{figure}[!htbp]
\centering
\includegraphics[width=65mm]{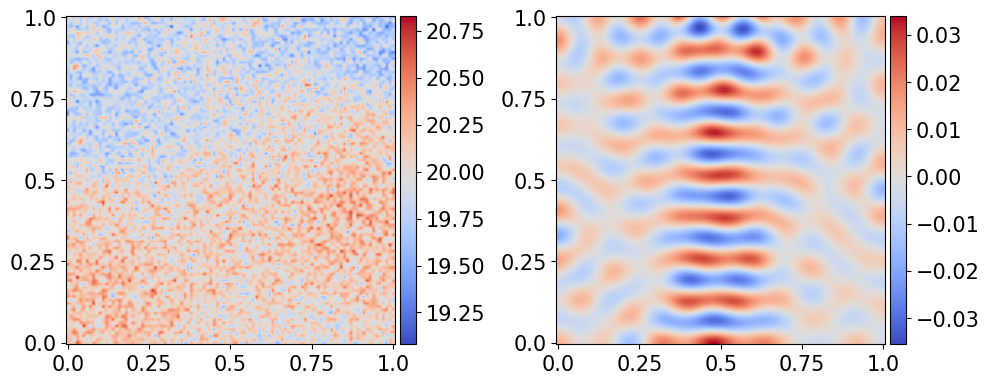}
\caption{\textbf{Helmholtz Equation}: (left to right) An example of (a) training input with 1\% noise, (b) training output with 1\% noise.}
\label{Fig:Helmholtz_training_data}
\end{figure}

\begin{figure}[!htbp]
\centering     
\subfigure{\includegraphics[width=100.5mm]{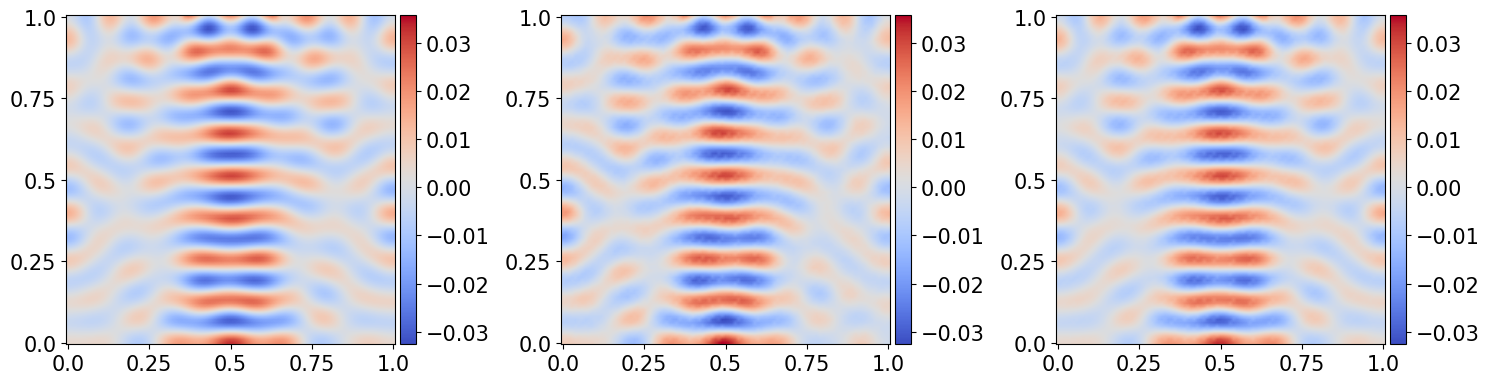}}
\subfigure{\includegraphics[width=62.5mm]{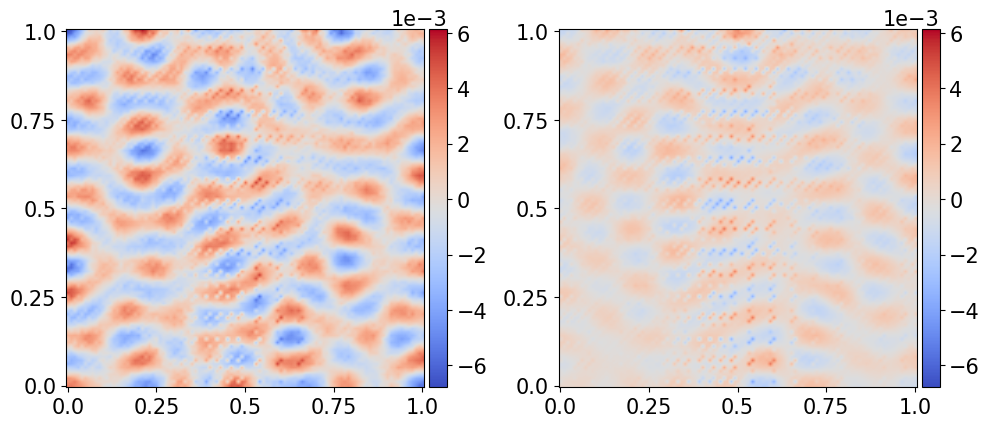}}
\caption{\textbf{Helmholtz Equation}: (left to right) (a) test example, (b) prediction of $R_\mathcal{V}\circ\hat{f}$ using RFF-FEM-$\infty$, (c) prediction of $R_\mathcal{V}\circ\hat{f}$ using RRFF-FEM-$\infty$, (d) pointwise error for prediction of $R_\mathcal{V}\circ\hat{f}$ using RFF-FEM-$\infty$, (e) pointwise error for prediction of $R_\mathcal{V}\circ\hat{f}$ using RRFF-FEM-$\infty$.}
\label{Fig:Helmholtz_recovery_Gaussian}
\end{figure}

\begin{figure}[!htbp]
\centering     
\subfigure{\includegraphics[width=99.5mm]{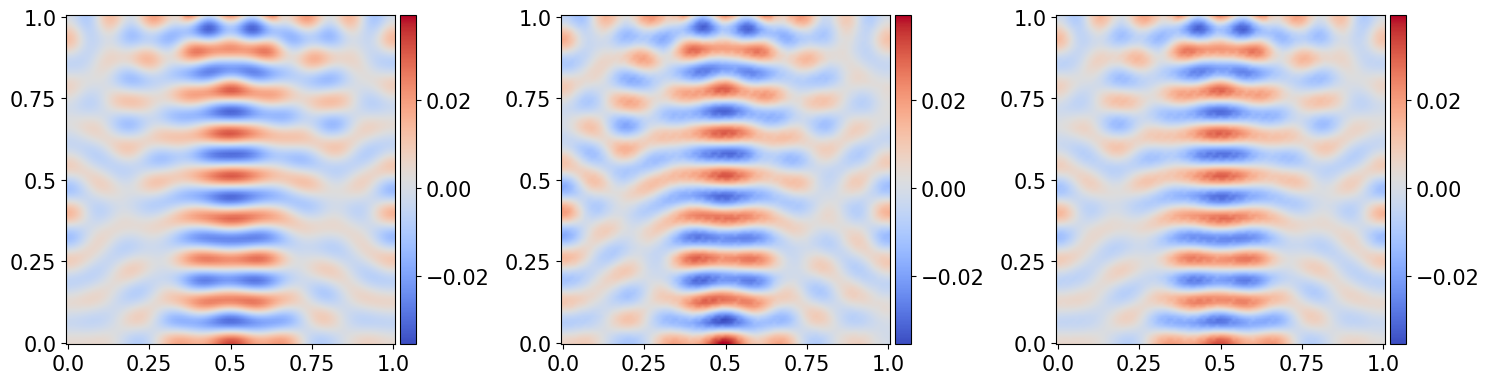}}
\subfigure{\includegraphics[width=63.5mm]{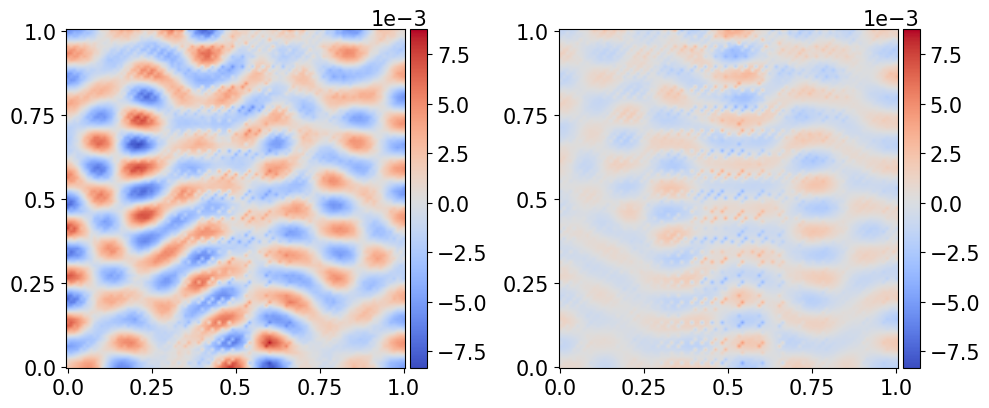}}
\caption{\textbf{Helmholtz Equation}: (left to right) (a) test example, (b) prediction of $R_\mathcal{V}\circ\hat{f}$ using RFF-FEM-2, (c) prediction of $R_\mathcal{V}\circ\hat{f}$ using RRFF-FEM-2, (d) pointwise error for prediction of $R_\mathcal{V}\circ\hat{f}$ using RFF-FEM-2, (e) pointwise error for prediction of $R_\mathcal{V}\circ\hat{f}$ using RRFF-FEM-2.}
\label{Fig:Helmholtz_recovery_Student2}
\end{figure}

\begin{figure}[!htbp]
\centering     
\subfigure{\includegraphics[width=99.5mm]{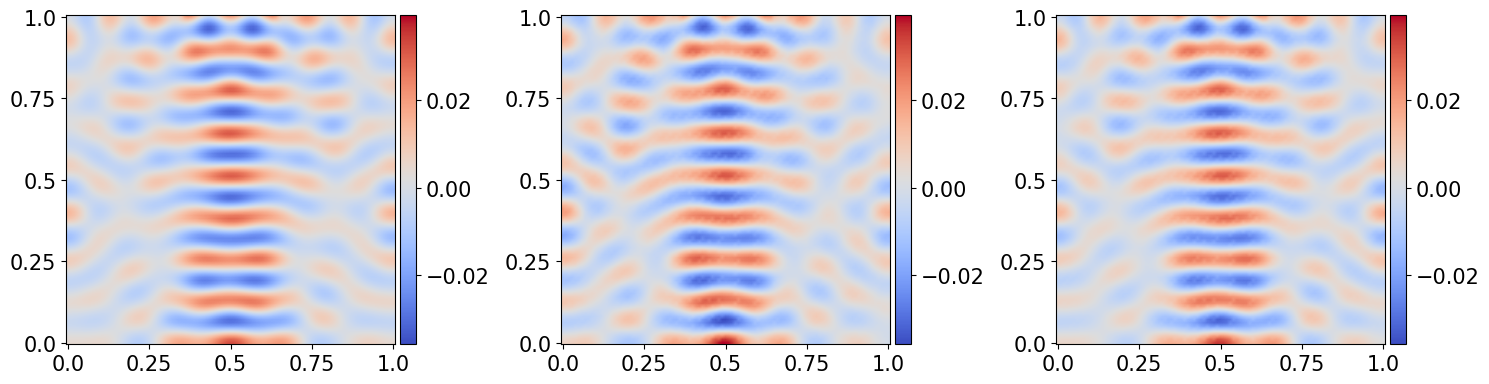}}
\subfigure{\includegraphics[width=63.5mm]{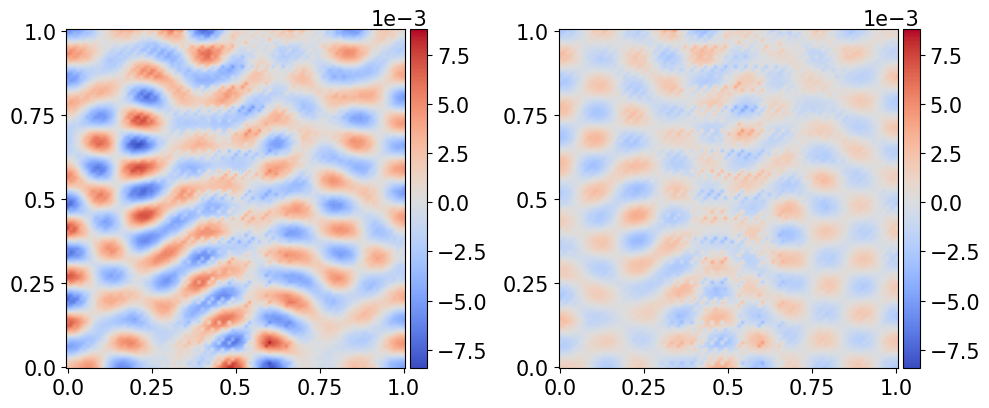}}
\caption{\textbf{Helmholtz Equation}: (left to right) (a) test example, (b) prediction of $R_\mathcal{V}\circ\hat{f}$ using RFF-FEM-3, (c) prediction of $R_\mathcal{V}\circ\hat{f}$ using RRFF-FEM-3, (d) pointwise error for prediction of $R_\mathcal{V}\circ\hat{f}$ using RFF-FEM-3, (e) pointwise error for prediction of $R_\mathcal{V}\circ\hat{f}$ using RRFF-FEM-3.}
\label{Fig:Helmholtz_recovery_Student3}
\end{figure}

\subsection{Navier-Stokes Equation}
\label{NS_section}
Consider the vorticity-stream $(\omega,\psi)$ formulation of the incompressible Navier-Stokes equations on $D$:
\begin{align} \label{navier}
\begin{split}
    &\dfrac{\partial \omega}{\partial t} + (c\cdot \nabla)\omega - \mu\Delta \omega = u\\
    &\omega = -\Delta \psi, \hspace{0.7cm}\int_D\psi=0\\
    &c = \bigg(\dfrac{\partial \psi}{\partial x_2},-\dfrac{\partial\psi}{\partial x_1}\bigg)
\end{split}
\end{align}
with the viscosity $\mu=0.025$, periodic boundary conditions, and fixed initial condition $\omega(\cdot,0)$. Our goal is to learn the operator $G: u \mapsto \omega(\cdot,T)$, which maps the forcing term $u$ to the vorticity field $\omega(\cdot,T)$ at time $t = T$. The domains are $D_\mathcal{U} = D_\mathcal{V} = D = [0,2\pi]^2$. As in \cite{dehoop2022costaccuracy}, the forcing term $u$ is generated from a Gaussian process, $\mathcal{GP}(0,(-\Delta+9\mathbf{I})^{-4})$. The fixed initial condition $\omega(\cdot,0)$ is sampled from the same distribution. We set final time $T = 10$. Equation (\ref{navier}) is solved on a discretized grid of size $64 \times 64$ using a pseudo-spectral method and Crank-Nicolson time integration; see \cite{dehoop2022costaccuracy} for additional implementation details.

For the RFF-FEM and RRFF-FEM methods, a coarser nonuniform grid on $[0,2\pi]^2$ is used with around 2.7k grid points. In the training stage of $\hat{f}$, 10k samples are used, and during its testing process, 25k samples are used. The performance of $R_\mathcal{V}\circ\hat{f}$ is measured using the hold-out 5k samples. Table \ref{NS_param_table} summarizes the parameters in the numerical experiments. We provide an example of training input and output with 5\% noise in Figure \ref{Fig:NS_training_data}. Figures \ref{Fig:NS_recovery_Gaussian}, \ref{Fig:NS_recovery_Student2}, and \ref{Fig:NS_recovery_Student3} show examples of test functions, the predictions of $R_\mathcal{V}\circ\hat{f}$ obtained by RFF-FEM-$\infty$, RRFF-FEM-$\infty$, RFF-FEM-2, RRFF-FEM-2, RFF-FEM-3, and RRFF-FEM-3, respectively, as well as the pointwise errors.

\begin{table}[!htbp]
\centering
    \scriptsize	
    \begin{tabular}{{|P{0.13\linewidth} | P{0.08\linewidth} | P{0.08\linewidth} |
    P{0.08\linewidth} | P{0.08\linewidth} |}}
\hline
\addstackgap{Distribution} & \addstackgap{$N$} & \addstackgap{$\sigma$} & \addstackgap{$\alpha$} & \addstackgap{$p$} 
\\ \hline \addstackgap{Gaussian} & \addstackgap{15k} & \addstackgap{$\sqrt{2\times 10^{-3}}$} & \addstackgap{$10^{-5}$} & \addstackgap{4}  \\ \hline \addstackgap{Student ($\nu=2$)} & \addstackgap{15k} & \addstackgap{$10^{-4}$} & \addstackgap{$10^{-5}$} & \addstackgap{4} \\ \hline \addstackgap{Student ($\nu=3$)} & \addstackgap{15k} & \addstackgap{$10^{-4}$} & \addstackgap{$10^{-5}$} & \addstackgap{4} \\
\hline
\end{tabular}
\caption{\textbf{Navier-Stokes Equation}: Parameters for the numerical experiments, where $N$ is the number of random features and $\sigma$ is the scale parameter for the RFF, RRFF, RFF-FEM, and RRFF-FEM methods, and $\alpha$ and $p$ are regularization parameters for the RRFF and RRFF-FEM methods.}
\label{NS_param_table}
\end{table}

\begin{figure}[!htbp]
\centering
\includegraphics[width=65mm]{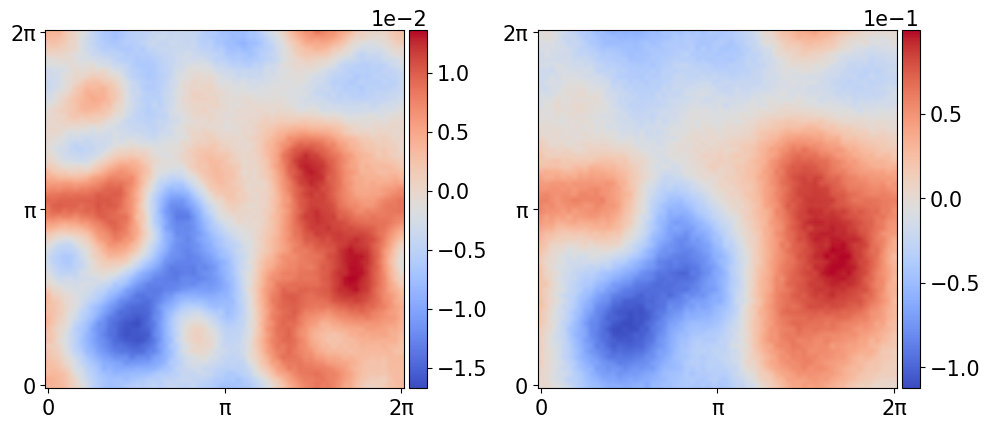}
\caption{\textbf{Navier-Stokes Equation}: (left to right) An example of (a) training input with 5\% noise, (b) training output with 5\% noise.}
\label{Fig:NS_training_data}
\end{figure}

\begin{figure}[!htbp]
\centering     
\subfigure{\includegraphics[width=99mm]{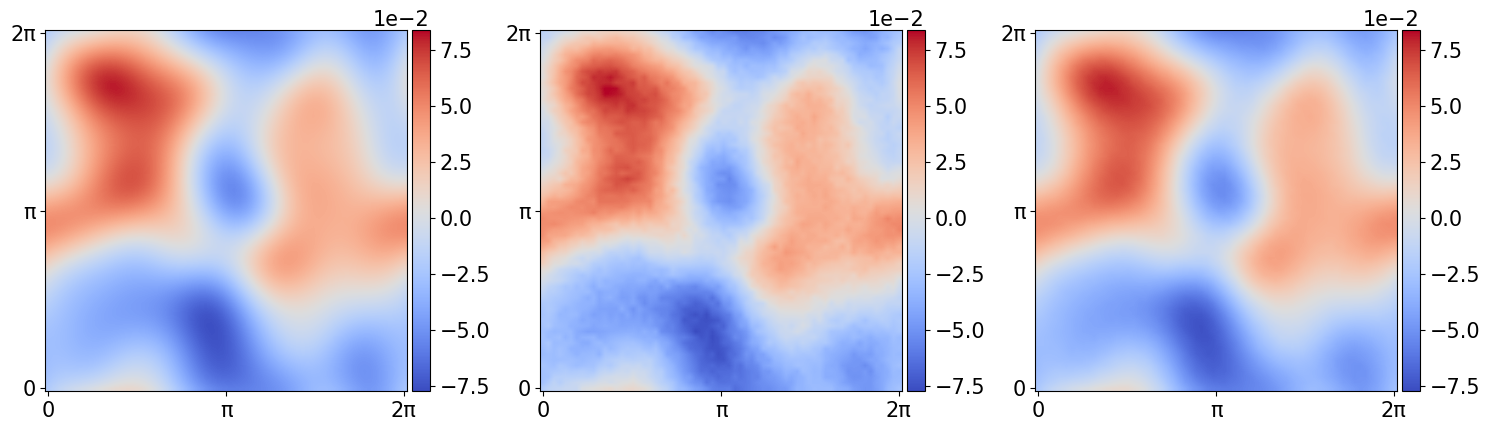}}
\subfigure{\includegraphics[width=64mm]{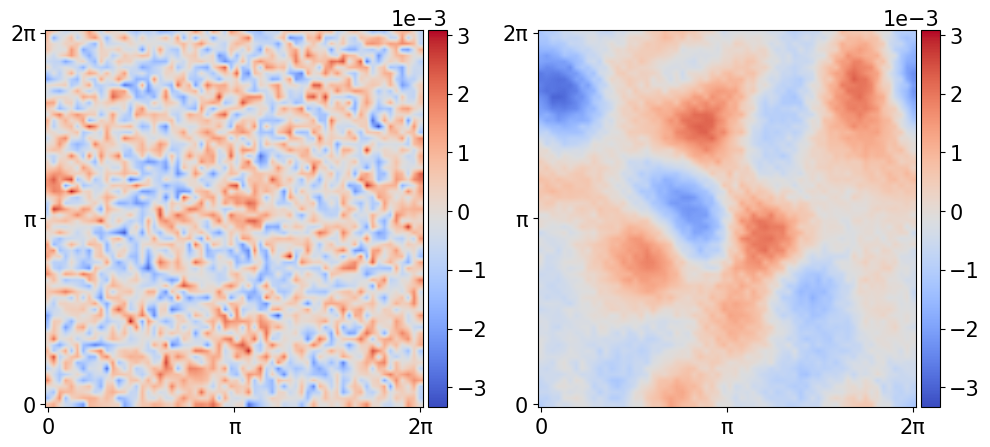}}
\caption{\textbf{Navier-Stokes Equation}: (left to right) (a) test example, (b) prediction of $R_\mathcal{V}\circ\hat{f}$ using RFF-FEM-$\infty$, (c) prediction of $R_\mathcal{V}\circ\hat{f}$ using RRFF-FEM-$\infty$, (d) pointwise error for prediction of $R_\mathcal{V}\circ\hat{f}$ using RFF-FEM-$\infty$, (e) pointwise error for prediction of $R_\mathcal{V}\circ\hat{f}$ using RRFF-FEM-$\infty$.}
\label{Fig:NS_recovery_Gaussian}
\end{figure}

\begin{figure}[!htbp]
\centering     
\subfigure{\includegraphics[width=99mm]{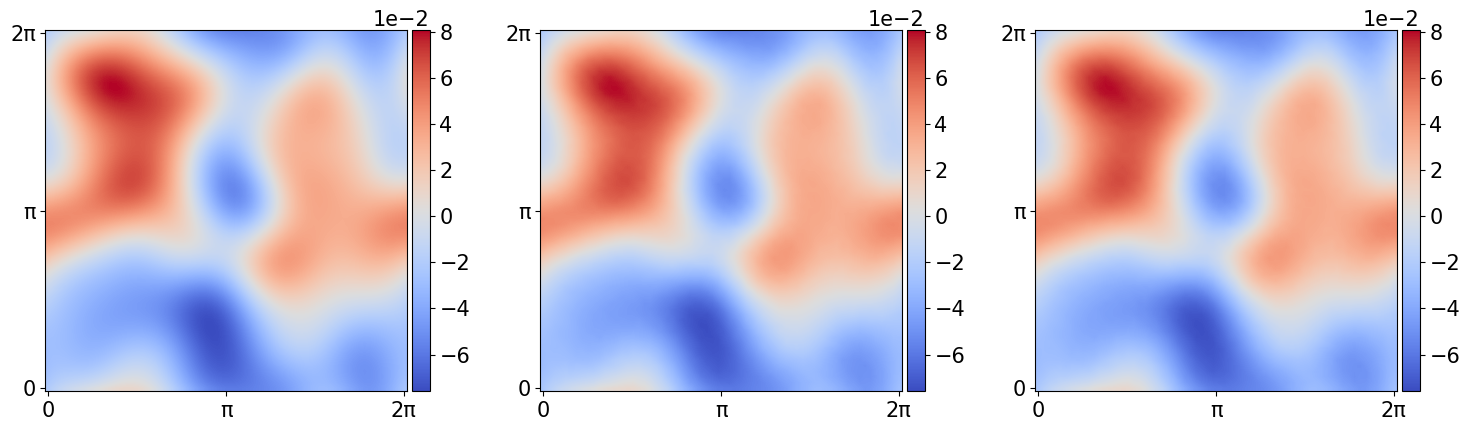}}
\subfigure{\includegraphics[width=64mm]{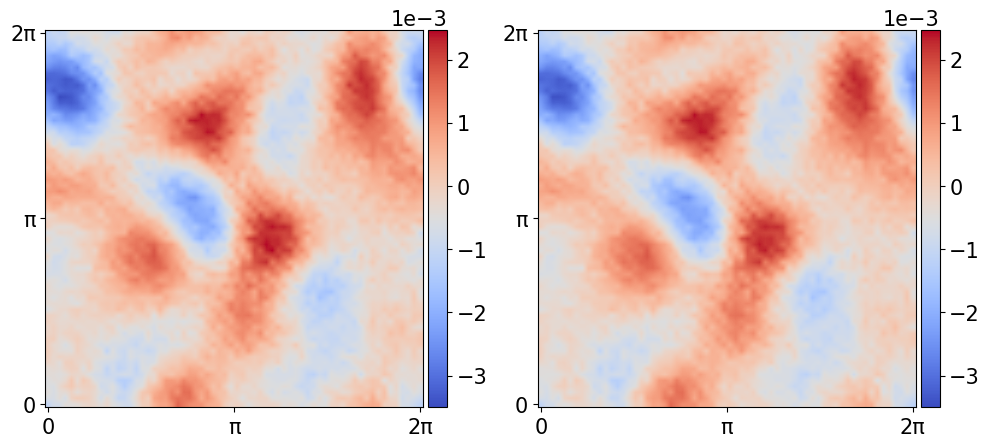}}
\caption{\textbf{Navier-Stokes Equation}: (left to right) (a) test example, (b) prediction of $R_\mathcal{V}\circ\hat{f}$ using RFF-FEM-2, (c) prediction of $R_\mathcal{V}\circ\hat{f}$ using RRFF-FEM-2, (d) pointwise error for prediction of $R_\mathcal{V}\circ\hat{f}$ using RFF-FEM-2, (e) pointwise error for prediction of $R_\mathcal{V}\circ\hat{f}$ using RRFF-FEM-2.}
\label{Fig:NS_recovery_Student2}
\end{figure}

\begin{figure}[!htbp]
\centering     
\subfigure{\includegraphics[width=99mm]{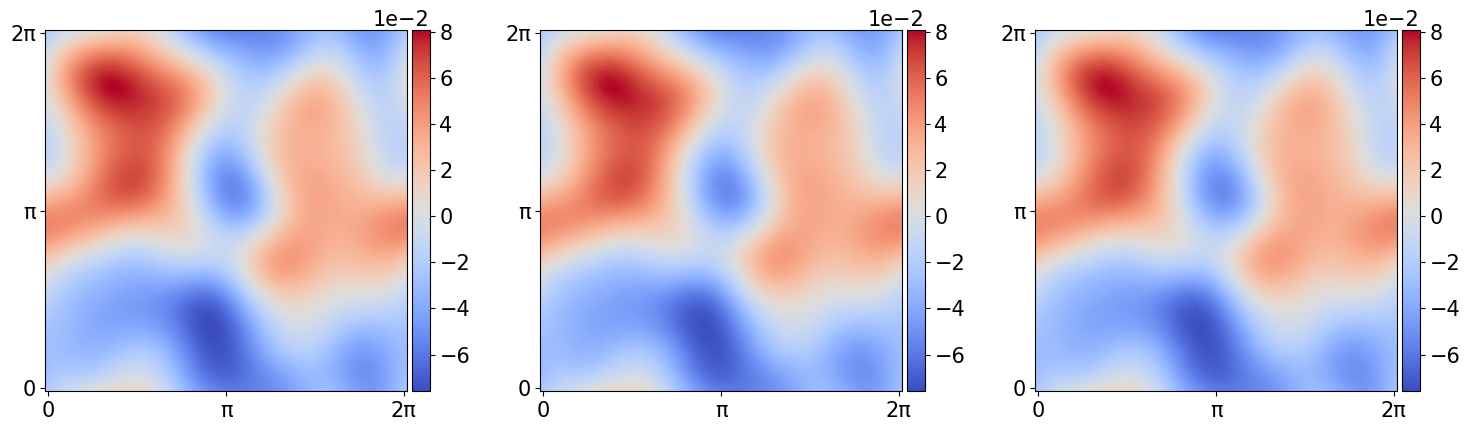}}
\subfigure{\includegraphics[width=64mm]{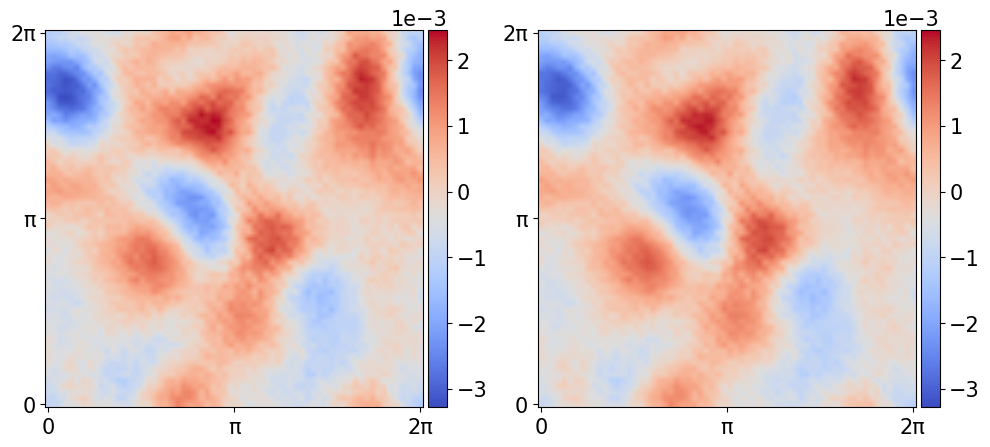}}
\caption{\textbf{Navier-Stokes Equation}: (left to right) (a) test example, (b) prediction of $R_\mathcal{V}\circ\hat{f}$ using RFF-FEM-3, (c) prediction of $R_\mathcal{V}\circ\hat{f}$ using RRFF-FEM-3, (d) pointwise error for prediction of $R_\mathcal{V}\circ\hat{f}$ using RFF-FEM-3, (e) pointwise error for prediction of $R_\mathcal{V}\circ\hat{f}$ using RRFF-FEM-3.}
\label{Fig:NS_recovery_Student3}
\end{figure}

\subsection{Structural Mechanics}
The governing equation of the displacement vector $w$ in an elastic solid undergoing infinitesimal
deformations on $D = (0,1)^2$ is:
    \begin{align*}
        \nabla \cdot \sigma = 0, \hspace{0.5cm}&\text{in } D \\
        w = \bar{w},\hspace{0.5cm} &\text{on } \Gamma_w\\
        \sigma \cdot n = u, \hspace{0.5cm} &\text{on } \Gamma_u
    \end{align*}
where $\sigma$ is the Cauchy stress tensor, $\bar{w}$ is the prescribed displacement on $\Gamma_w$, and $u$ is the surface traction on $\Gamma_u$. The boundary $\partial D$ is split into $\Gamma_u = [0,1]\times1$ and its complement $\Gamma_w$ with outward unit normal vector $n$. We want to learn the operator $G: u \mapsto v$, the map between the surface traction $u$ and the von Mises stress field $v$. Here we have $D_\mathcal{U} = (0,1)$ and $D_\mathcal{V}=(0,1)^2$. The load $u$ is sampled from the Gaussian process $\mathcal{GP}(100,400^2(-\Delta+9\mathbf{I})^{-1})$, where $\Delta$ is the Laplacian on $D_\mathcal{U} = (0,1)$ subject to homogeneous Neumann boundary conditions on the space of zero-mean functions and $\mathbf{I}$ is the identity. Samples for $v$ are obtained using the NNFEM library, a finite element code \cite{xu2021learningconstitutiverelationssymmetricpositivedefinite,huang2020learningconstitutiverelationsindirectobservations}. A total of 40k samples are generated. 

In the RFF-FEM and RRFF-FEM models, a grid on $(0,1)$ is nonuniformly discretized with 28 grid points for the input function, and a grid on $(0,1)^2$ is also nonuniformly discretized with 784 grid points for the output function. We train $\hat{f}$ with 20k samples and test its performance with 15k samples. We measure the performance of $R_\mathcal{V}\circ\hat{f}$ on the 5k samples that remain. The parameters for the numerical experiments are provided in Table \ref{structural_param_table}. An example of training input and output corrupted with 5\% noise is presented in Figure \ref{Fig:Structural_training_data}. Plots of test examples, the approximations of $R_\mathcal{V}\circ\hat{f}$ from the RFF-FEM-$\infty$, RRFF-FEM-$\infty$, RFF-FEM-2, RRFF-FEM-2, RFF-FEM-3, and RRFF-FEM-3 methods, and the corresponding pointwise errors are shown in Figures \ref{Fig:Structural_recovery_Gaussian}, \ref{Fig:Structural_recovery_Student2}, and \ref{Fig:Structural_recovery_Student3}.

\begin{table}[!htbp]
\centering
    \scriptsize	
    \begin{tabular}{{|P{0.13\linewidth} | P{0.08\linewidth} | P{0.08\linewidth} |
    P{0.08\linewidth} | P{0.08\linewidth} |}}
\hline
\addstackgap{Distribution} & \addstackgap{$N$} & \addstackgap{$\sigma$} & \addstackgap{$\alpha$} & \addstackgap{$p$} 
\\ \hline \addstackgap{Gaussian} & \addstackgap{25k} & \addstackgap{$\sqrt{2\times 10^{-6}}$} & \addstackgap{0.5} & \addstackgap{2}  \\ \hline \addstackgap{Student ($\nu=2$)} & \addstackgap{25k} & \addstackgap{$10^{-3}$} & \addstackgap{0.5} & \addstackgap{2} \\ \hline \addstackgap{Student ($\nu=3$)} & \addstackgap{25k} & \addstackgap{$10^{-3}$} & \addstackgap{0.5} & \addstackgap{2} \\
\hline
\end{tabular}
\caption{\textbf{Structural Mechanics}: Parameters for the numerical experiments, where $N$ is the number of random features and $\sigma$ is the scale parameter for the RFF, RRFF, RFF-FEM, and RRFF-FEM methods, and $\alpha$ and $p$ are regularization parameters for the RRFF and RRFF-FEM methods.}
\label{structural_param_table}
\end{table}

\begin{figure}[!htbp]
\centering
\subfigure{\includegraphics[width=37mm]{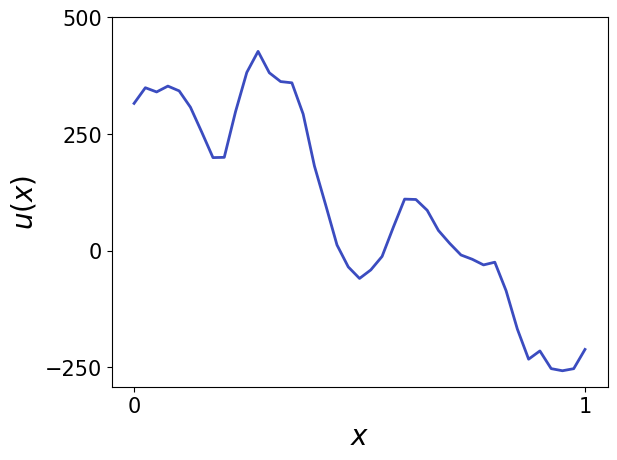}}
\subfigure{\includegraphics[width=33mm]{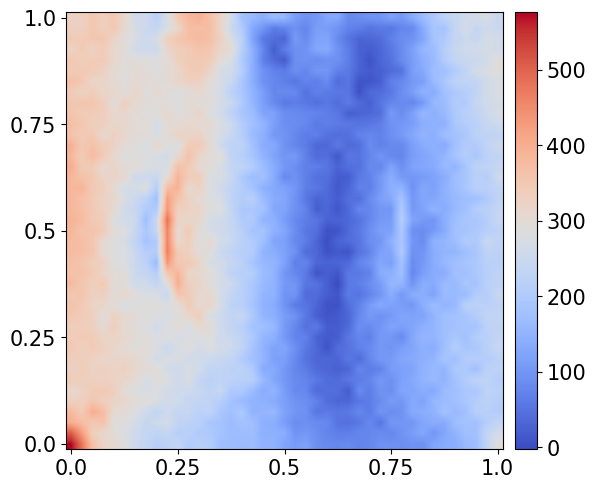}}
\caption{\textbf{Structural Mechanics}: (left to right) An example of (a) training input with 5\% noise, (b) training output with 5\% noise.}
\label{Fig:Structural_training_data}
\end{figure}

\begin{figure}[!htbp]
\centering     
\subfigure{\includegraphics[width=98mm]{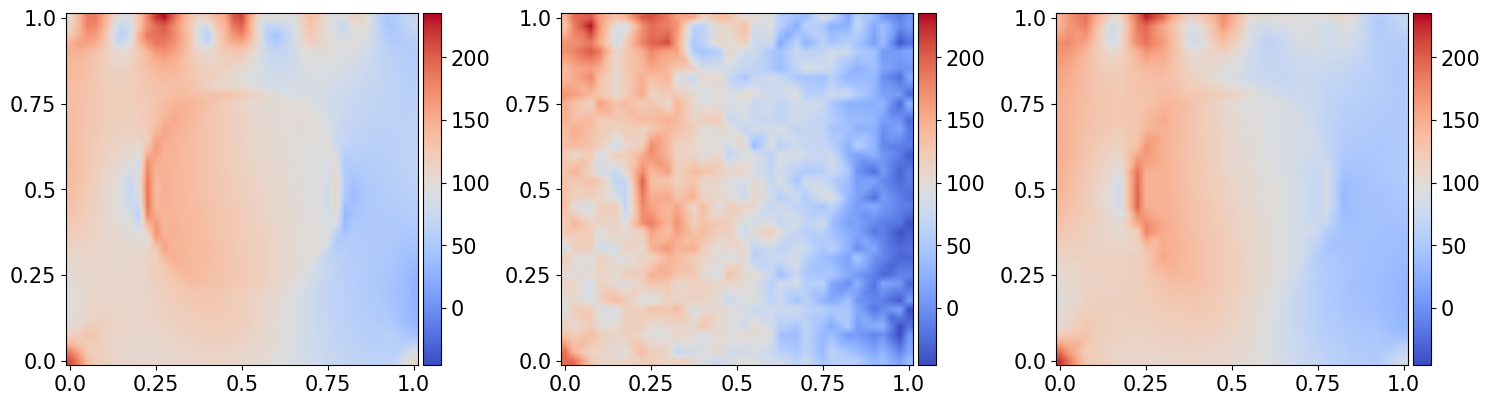}}
\subfigure{\includegraphics[width=65mm]{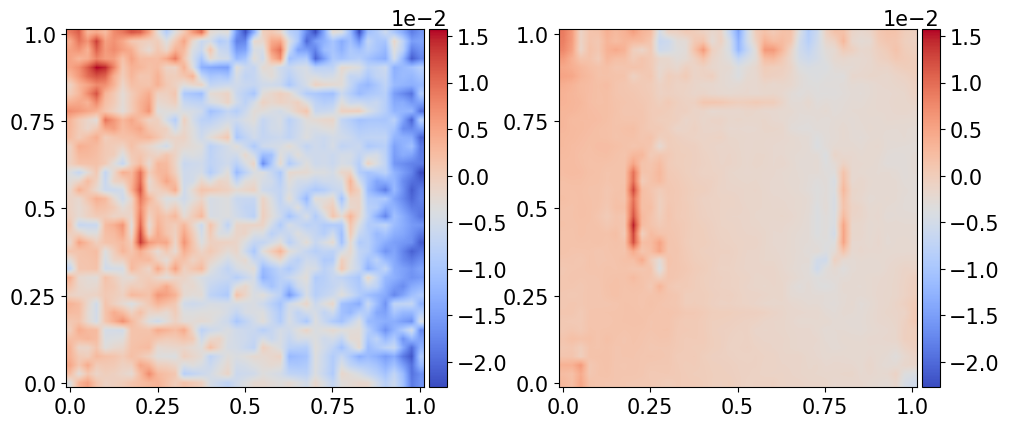}}
\caption{\textbf{Structural Mechanics}: (left to right) (a) test example, (b) prediction of $R_\mathcal{V}\circ\hat{f}$ using RFF-FEM-$\infty$, (c) prediction of $R_\mathcal{V}\circ\hat{f}$ using RRFF-FEM-$\infty$, (d) pointwise error for prediction of $R_\mathcal{V}\circ\hat{f}$ using RFF-FEM-$\infty$, (e) pointwise error for prediction of $R_\mathcal{V}\circ\hat{f}$ using RRFF-FEM-$\infty$.}
\label{Fig:Structural_recovery_Gaussian}
\end{figure}

\begin{figure}[!htbp]
\centering     
\subfigure{\includegraphics[width=100.5mm]{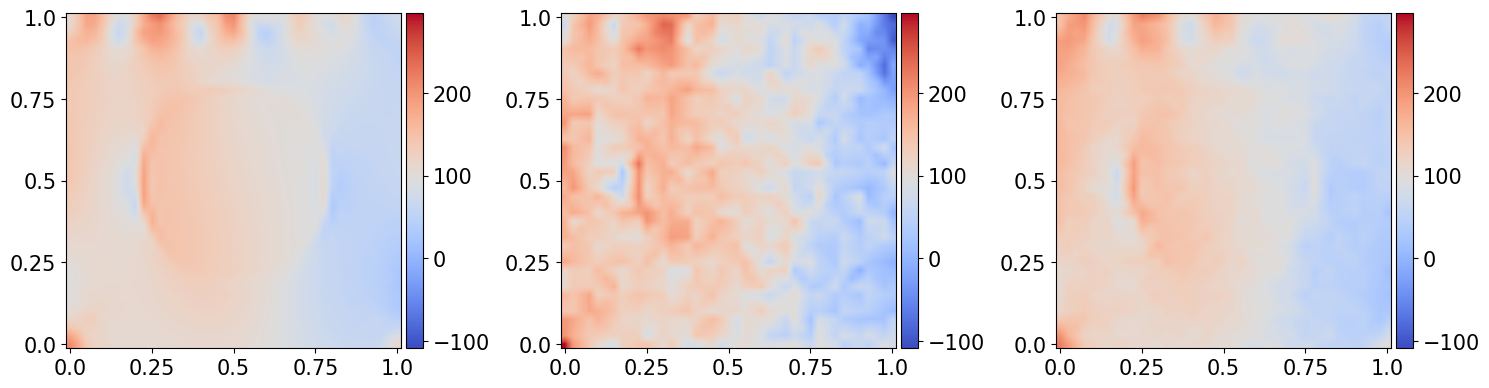}}
\subfigure{\includegraphics[width=62.5mm]{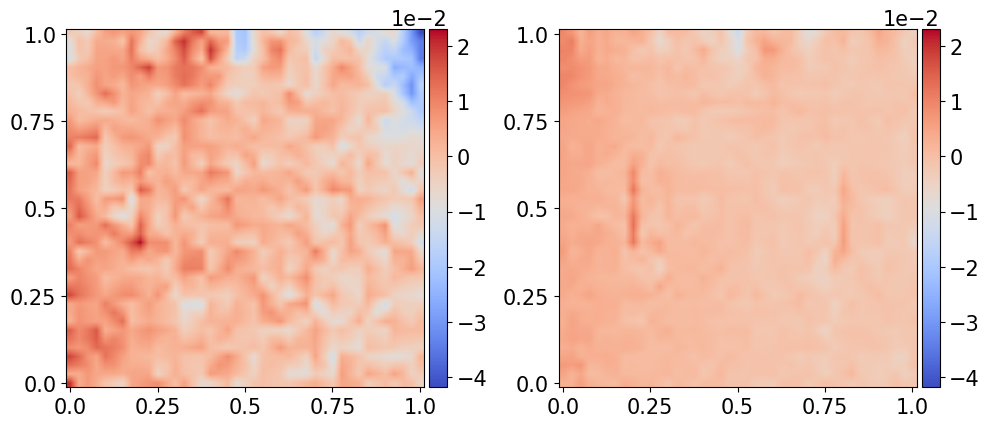}}
\caption{\textbf{Structural Mechanics}: (left to right) (a) test example, (b) prediction of $R_\mathcal{V}\circ\hat{f}$ using RFF-FEM-2, (c) prediction of $R_\mathcal{V}\circ\hat{f}$ using RRFF-FEM-2, (d) pointwise error for prediction of $R_\mathcal{V}\circ\hat{f}$ using RFF-FEM-2, (e) pointwise error for prediction of $R_\mathcal{V}\circ\hat{f}$ using RRFF-FEM-2.}
\label{Fig:Structural_recovery_Student2}
\end{figure}

\begin{figure}[!htbp]
\centering     
\subfigure{\includegraphics[width=98mm]{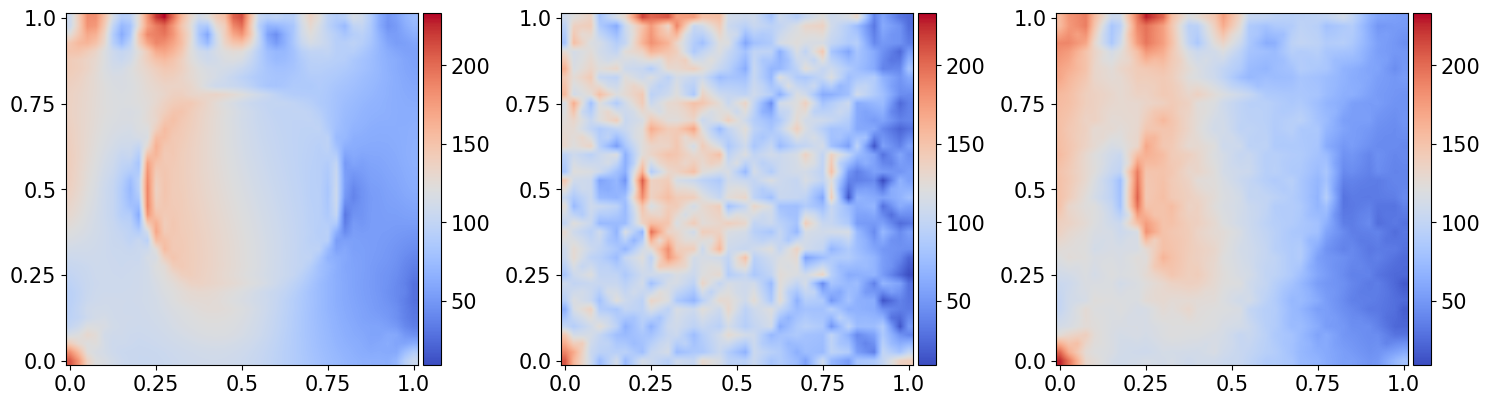}}
\subfigure{\includegraphics[width=65mm]{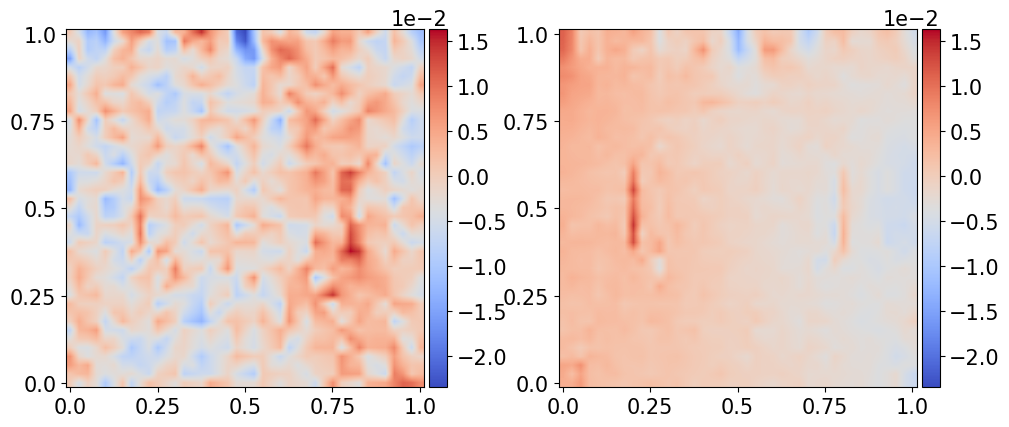}}
\caption{\textbf{Structural Mechanics}: (left to right) (a) test example, (b) prediction of $R_\mathcal{V}\circ\hat{f}$ using RFF-FEM-3, (c) prediction of $R_\mathcal{V}\circ\hat{f}$ using RRFF-FEM-3, (d) pointwise error for prediction of $R_\mathcal{V}\circ\hat{f}$ using RFF-FEM-3, (e) pointwise error for prediction of $R_\mathcal{V}\circ\hat{f}$ using RRFF-FEM-3.}
\label{Fig:Structural_recovery_Student3}
\end{figure}

\subsection{Discussion of Results}
In Table \ref{f_hat_table}, we report the average relative test errors and training times of the RFF and RRFF methods for learning $\hat{f}$ over 20 trials for a set of benchmark PDEs. We observe that in all cases, RRFF consistently outperforms RFF with respect to test accuracy and training times. Note also that the regularizer in the RRFF methods improves the conditioning of the random feature matrix, so training using linear iterative methods is faster. RRFF-$\infty$ achieves consistently lower error among all the models. One potential reason could be because the Student's $t$ distribution with $\nu=\infty$ corresponds to the Gaussian distribution, which may be a better prior for these examples. 

We provide comparisons to kernel methods used to learn $\hat{f}$ for Burgers' equation and Darcy flow in Table \ref{kernel_table}. The radial basis function (RBF) and Mat\'ern kernels are used, and the relative test errors and training times are reported for these methods. For these benchmark PDEs, our RRFF models have test errors comparable to those of kernel methods while requiring significantly lower training times. In particular, RRFF-$\infty$ outperforms the RBF kernel method for Burgers' equation, and both the RFF-$\infty$ and RRFF-$\infty$ methods outperform the RBF kernel method for Darcy flow. Additionally, comparisons of neural operator methods are found in Table \ref{deeponet_fno_table}. The relative test errors of DeepONet and FNO for learning $\hat{f}$ in the Navier-Stokes and structural mechanics problems are cited from \cite{lu2022comprehensive,dehoop2022costaccuracy,batlle2023kernelmethodscompetitiveoperator,liao2025cauchyrandomfeaturesoperator}. As a reference, the data in these examples are noise-free. 

Table \ref{R_v_f_hat_table} presents the average relative test errors of the RFF-FEM and RRFF-FEM predictions for $R_v \circ \hat{f}$ over 20 runs, where $\hat{f}$ is learned using RFF and RRFF and $R_v$ is obtained from finite element interpolation. Across all cases, RRFF-FEM achieves better performance than RFF-FEM with lower relative test errors.

\begin{table}[!htbp]
    \centering\scriptsize	
    \begin{tabular}{{|P{0.075\linewidth} | P{0.06\linewidth} | P{0.09\linewidth} | P{0.09\linewidth} |
    P{0.09\linewidth} | P{0.09\linewidth} |
    P{0.09\linewidth} | P{0.09\linewidth} |}}
\hline
& & \addstackgap{RFF-$\infty$} & \addstackgap{RRFF-$\infty$} & \addstackgap{RFF-2} & \addstackgap{RRFF-2} & \addstackgap{RFF-3} & \addstackgap{RRFF-3} \\
 \hline
  \multirow{2}{*}{\addstackgap{\shortstack{Advection\\ \RomanNumeral{1}}}} & \multicolumn{1}{P{1cm}|}{\addstackgap{\shortstack{Relative\\ Error}}} & \multicolumn{1}{c|}{8.79} & \multicolumn{1}{c|}{4.18} & \multicolumn{1}{c|}{15.5} & \multicolumn{1}{c|}{4.46} & \multicolumn{1}{c|}{15.2} &\multicolumn{1}{c|}{4.53} \\ \cline{2-8}
& \multicolumn{1}{P{1cm}|}{\addstackgap{\shortstack{Training\\ Time}}} & \multicolumn{1}{c|}{0.45} & \multicolumn{1}{c|}{0.34} & \multicolumn{1}{c|}{0.45}& \multicolumn{1}{c|}{0.34}& \multicolumn{1}{c|}{0.45}& \multicolumn{1}{c|}{0.34}\\
\hline
 \multirow{2}{*}{\addstackgap{\shortstack{Advection\\ \RomanNumeral{2}}}} & \multicolumn{1}{P{1cm}|}{\addstackgap{\shortstack{Relative\\ Error}}} & \multicolumn{1}{c|}{10.0} & \multicolumn{1}{c|}{4.48} & \multicolumn{1}{c|}{15.5} & \multicolumn{1}{c|}{4.56} & \multicolumn{1}{c|}{14.2} &\multicolumn{1}{c|}{4.53} \\ \cline{2-8}
& \multicolumn{1}{P{1cm}|}{\addstackgap{\shortstack{Training\\ Time}}} & \multicolumn{1}{c|}{0.45} & \multicolumn{1}{c|}{0.34} & \multicolumn{1}{c|}{0.44}& \multicolumn{1}{c|}{0.34}& \multicolumn{1}{c|}{0.44}& \multicolumn{1}{c|}{0.34}\\
\hline
 \multirow{2}{*}{\addstackgap{\shortstack{Advection\\ \RomanNumeral{3}}}} & \multicolumn{1}{P{1cm}|}{\addstackgap{\shortstack{Relative\\ Error}}} & \multicolumn{1}{c|}{23.2} & \multicolumn{1}{c|}{14.9} & \multicolumn{1}{c|}{30.9} & \multicolumn{1}{c|}{15.6} & \multicolumn{1}{c|}{27.3} &\multicolumn{1}{c|}{17.4} \\ \cline{2-8}
& \multicolumn{1}{P{1cm}|}{\addstackgap{\shortstack{Training\\ Time}}} & \multicolumn{1}{c|}{0.44} & \multicolumn{1}{c|}{0.33} & \multicolumn{1}{c|}{0.44}& \multicolumn{1}{c|}{0.33}& \multicolumn{1}{c|}{0.44}& \multicolumn{1}{c|}{0.33}\\
\hline
 \multirow{2}{*}{Burgers'} & \multicolumn{1}{P{1cm}|}{\addstackgap{\shortstack{Relative\\ Error}}} & \multicolumn{1}{c|}{7.92} & \multicolumn{1}{c|}{5.17} & \multicolumn{1}{c|}{17.0} & \multicolumn{1}{c|}{9.08} & \multicolumn{1}{c|}{17.8} &\multicolumn{1}{c|}{11.0} \\ \cline{2-8}
& \multicolumn{1}{P{1cm}|}{\addstackgap{\shortstack{Training\\ Time}}} & \multicolumn{1}{c|}{2.31} & \multicolumn{1}{c|}{1.79} & \multicolumn{1}{c|}{2.32}& \multicolumn{1}{c|}{1.81}& \multicolumn{1}{c|}{2.33}& \multicolumn{1}{c|}{1.81}\\
\hline
 \multirow{2}{*}{Darcy} & \multicolumn{1}{P{1cm}|}{\addstackgap{\shortstack{Relative\\ Error}}} & \multicolumn{1}{c|}{5.76} & \multicolumn{1}{c|}{4.30} & \multicolumn{1}{c|}{10.5} & \multicolumn{1}{c|}{8.24}& \multicolumn{1}{c|}{11.6}& \multicolumn{1}{c|}{9.36}\\ \cline{2-8}
& \multicolumn{1}{P{1cm}|}{\addstackgap{\shortstack{Training\\ Time}}} & \multicolumn{1}{c|}{2.58} & \multicolumn{1}{c|}{2.12} & \multicolumn{1}{c|}{2.56}& \multicolumn{1}{c|}{2.11}& \multicolumn{1}{c|}{2.56}& \multicolumn{1}{c|}{2.11}\\
\hline
 \multirow{2}{*}{Helmholtz} & \multicolumn{1}{P{1cm}|}{\addstackgap{\shortstack{Relative\\ Error}}} & \multicolumn{1}{c|}{28.1} & \multicolumn{1}{c|}{13.1} & \multicolumn{1}{c|}{37.7} & \multicolumn{1}{c|}{14.0}& \multicolumn{1}{c|}{36.7}& \multicolumn{1}{c|}{18.1}\\ \cline{2-8}
& \multicolumn{1}{P{1cm}|}{\addstackgap{\shortstack{Training\\ Time}}} & \multicolumn{1}{c|}{24.20} & \multicolumn{1}{c|}{22.98} & \multicolumn{1}{c|}{24.55}& \multicolumn{1}{c|}{23.26}& \multicolumn{1}{c|}{23.75}& \multicolumn{1}{c|}{22.65}\\
\hline
 \multirow{2}{*}{\addstackgap{\shortstack{Navier \\Stokes}}} & \multicolumn{1}{P{1cm}|}{\addstackgap{\shortstack{Relative\\ Error}}} & \multicolumn{1}{c|}{8.13} & \multicolumn{1}{c|}{5.01} & \multicolumn{1}{c|}{5.68} & \multicolumn{1}{c|}{5.64}& \multicolumn{1}{c|}{5.58}& \multicolumn{1}{c|}{5.52}\\ \cline{2-8}
& \multicolumn{1}{P{1cm}|}{\addstackgap{\shortstack{Training\\ Time}}} & \multicolumn{1}{c|}{16.60} & \multicolumn{1}{c|}{15.28} & \multicolumn{1}{c|}{14.56}& \multicolumn{1}{c|}{13.39}& \multicolumn{1}{c|}{14.66}& \multicolumn{1}{c|}{13.51}\\
\hline
 \multirow{2}{*}{\addstackgap{\shortstack{Structural \\Mechanics}}} & \multicolumn{1}{P{1cm}|}{\addstackgap{\shortstack{Relative\\ Error}}} & \multicolumn{1}{c|}{28.7} & \multicolumn{1}{c|}{5.86} & \multicolumn{1}{c|}{34.7} & \multicolumn{1}{c|}{13.2}& \multicolumn{1}{c|}{34.6}& \multicolumn{1}{c|}{7.86}\\ \cline{2-8}
& \multicolumn{1}{P{1cm}|}{\addstackgap{\shortstack{Training\\ Time}}} & \multicolumn{1}{c|}{27.01} & \multicolumn{1}{c|}{23.42} & \multicolumn{1}{c|}{27.20}& \multicolumn{1}{c|}{23.97}& \multicolumn{1}{c|}{26.93}& \multicolumn{1}{c|}{23.62}\\
\hline
\end{tabular}
\caption{\textbf{Comparison of RFF and RRFF}: Average relative test errors (\%) and average training times (seconds) of RFF and RRFF methods for $\hat{f}$ over 20 trials. }
\label{f_hat_table}
\end{table}

\begin{table}
    \centering \scriptsize	
    \begin{tabular}{{|P{0.075\linewidth} | P{0.06\linewidth} | P{0.1\linewidth} | P{0.12\linewidth} |
    P{0.12\linewidth} |}}
\hline
& & \addstackgap{RBF Kernel} & \addstackgap{\shortstack{Mat\'ern Kernel\\($\nu=2$)}} & \addstackgap{\shortstack{Mat\'ern Kernel\\($\nu = 3$)}}
\\ \hline
 \multirow{2}{*}{Burgers'} & \multicolumn{1}{P{1cm}|}{\addstackgap{\shortstack{Relative\\ Error}}} & \multicolumn{1}{c|}{7.65} & \multicolumn{1}{c|}{4.27} & \multicolumn{1}{c|}{4.75}\\ \cline{2-5}
& \multicolumn{1}{P{1cm}|}{\addstackgap{\shortstack{Training\\ Time}}} & \multicolumn{1}{c|}{16.48} & \multicolumn{1}{c|}{25.01} & \multicolumn{1}{c|}{23.55}\\
\hline
 \multirow{2}{*}{Darcy} & \multicolumn{1}{P{1cm}|}{\addstackgap{\shortstack{Relative\\ Error}}} & \multicolumn{1}{c|}{7.86} & \multicolumn{1}{c|}{4.76} & \multicolumn{1}{c|}{5.86} \\ \cline{2-5}
& \multicolumn{1}{P{1cm}|}{\addstackgap{\shortstack{Training\\ Time}}} & \multicolumn{1}{c|}{27.39} & \multicolumn{1}{c|}{72.47} & \multicolumn{1}{c|}{33.58}\\
\hline
\end{tabular}
\caption{\textbf{Comparison of Kernel Methods}: Relative test errors (\%) and training times (seconds) of radial basis function (RBF) and Mat\'ern kernel methods for $\hat{f}$. }
\label{kernel_table}
\end{table}

\begin{table}
    \centering \scriptsize
    \begin{tabular}{|P{0.08\linewidth} | P{0.08\linewidth} | P{0.08\linewidth} | }
\hline
& \addstackgap{DeepONet} & \addstackgap{FNO} \\
 \hline
\addstackgap{\shortstack{Navier\\ Stokes}} &  3.63 &  0.26 \\
\hline
\addstackgap{\shortstack{Structural\\ Mechanics}} & 5.20 &  4.76 \\
\hline
\end{tabular}
\caption{\textbf{Comparison of DeepONet and FNO}: Relative test errors (\%) of DeepONet and FNO for $\hat{f}$ for the noiseless problem cited from \cite{lu2022comprehensive,dehoop2022costaccuracy,batlle2023kernelmethodscompetitiveoperator,liao2025cauchyrandomfeaturesoperator}.}
\label{deeponet_fno_table}
\end{table}

\begin{table}[!htbp]
    \centering \scriptsize
    \begin{tabular}{|P{0.09\linewidth} | P{0.11\linewidth} | P{0.12\linewidth} |
    P{0.11\linewidth} | P{0.11\linewidth} |
    P{0.11\linewidth} | P{0.11\linewidth} |}
\hline
& \addstackgap{RFF-FEM-$\infty$} & \addstackgap{RRFF-FEM-$\infty$} & \addstackgap{RFF-FEM-2} & \addstackgap{RRFF-FEM-2} & \addstackgap{RFF-FEM-3} & \addstackgap{RRFF-FEM-3} \\
 \hline
\addstackgap{\shortstack{Advection\\ \RomanNumeral{1}}} & 41.7 &  18.8 & 44.9 & 18.6 & 34.9 & 18.5\\
 \hline
\addstackgap{\shortstack{Advection\\ \RomanNumeral{2}}} & 45.6 &  21.3 & 42.9 & 21.0 & 35.1 & 20.9\\
 \hline
\addstackgap{\shortstack{Advection\\ \RomanNumeral{3}}} & 25.6 &  16.2 & 29.1 & 17.4 & 24.9 & 19.3\\
 \hline
 \addstackgap{Burgers'} & 8.57 & 6.14 & 17.3 & 10.0 & 17.8 & 12.3\\
\hline
 \addstackgap{Darcy} & 6.11 & 4.75 & 8.20 & 6.47 & 8.38 & 6.86 \\
\hline
  \addstackgap{Helmholtz} & 24.0 & 14.3 & 24.4 & 14.8 & 24.5 & 17.3 \\
\hline \addstackgap{\shortstack{Navier\\ Stokes}} &  5.55 & 5.32 &  5.69 &  5.65 &  5.62 &  5.58 \\
\hline
\addstackgap{\shortstack{Structural\\ Mechanics}} & 33.0 & 7.53 & 39.0 & 11.4 & 39.2 & 8.18 \\
\hline
\end{tabular}
\caption{\textbf{Comparison of RFF-FEM and RRFF-FEM}: Average relative test errors (\%) of RFF-FEM and RRFF-FEM for $R_\mathcal{V} \circ \hat{f}$ over 20 trials.}
\label{R_v_f_hat_table}
\end{table}

\section{Conclusion}
\label{discussion}

In this work, we introduced the regularized random Fourier feature approach for operator learning in order to improve robustness and efficiency of kernel operator learning in the presence of noisy data. By incorporating frequency-weighted Tikhonov regularization and feature weights drawn from multivariate Student’s $t$ distributions, our approach generalizes prior random feature methods while ensuring well-conditioned feature matrices with high probability. The theoretical analysis establishes concentration bounds on the singular values of the random feature matrix and supports stable estimation when the number of features $N$ scales like $m \log m$, where $m$ is the number of training samples. Extensive numerical experiments on a range of benchmark PDEs, including advection, Burgers’, Darcy flow, Helmholtz, Navier–Stokes, and structural mechanics problems, demonstrated that RRFF and RRFF–FEM consistently outperform their unregularized counterparts. Comparisons against kernel-based and neural operator baselines showed that RRFF methods achieve comparable accuracy at substantially reduced computational cost. A potential future direction is to consider adaptive schemes for selecting random feature frequencies based on data-driven criteria.

\section*{Acknowledgments}
XY and HS were partially supported by the National Science Foundation under the grant DMS-2331033.

\FloatBarrier
\bibliographystyle{plain}
\bibliography{citations}
\section*{Appendix A: Noise}
\label{appendix:noise}
In the test cases in Section \ref{numerical_experiments}, we set the relative noise to $p\times100\%$ by enforcing:
\begin{equation*}
    \tilde{\mathbf{u}}_j = \mathbf{u}_j + p\dfrac{\|\mathbf{u}_j\|_2}{\|\boldsymbol{\epsilon}_{\mathbf{u}_j}\|_2}\boldsymbol{\epsilon}_{\mathbf{u}_j} \text{ and } \tilde{\mathbf{v}}_j = \mathbf{v}_j + p\dfrac{\|\mathbf{v}_j\|_2}{\|\boldsymbol{\epsilon}_{\mathbf{v}_j}\|_2}\boldsymbol{\epsilon}_{\mathbf{v}_j},
\end{equation*}
where $\boldsymbol{\epsilon}_{\mathbf{u}_j} \sim N(0,\mathbf{I}_n)$ and $\boldsymbol{\epsilon}_{\mathbf{v}_j} \sim N(0,\mathbf{I}_m)$ for all $j \in [M]$.

\section*{Appendix B: Numerical Experiments}
We provide additional numerical experiments for the following PDEs.

\subsection*{Advection Equations \RomanNumeral{1}, \RomanNumeral{2}, and \RomanNumeral{3}}

For the RFF and RRFF methods for Advection \RomanNumeral{1}, the grid is uniformly discretized to 40 points, and we use 1000 samples each for training and testing the models. In Figures \ref{Fig:AD1_test_pred_Gaussian_Student2} and \ref{Fig:AD1_test_pred_Student3}, test examples, RFF-$\infty$, RRFF-$\infty$, RFF-2, RRFF-2, RFF-3, and RRFF-3 predictions of $\hat{f}$, along with their pointwise errors are shown. For Advection \RomanNumeral{2}, RFF and RRFF use a uniform grid with 40 gridpoints, and the sizes of the training and test datasets are both 1000. Figures \ref{Fig:AD2_test_pred_Gaussian_Student2} and \ref{Fig:AD2_test_pred_Student3} show test examples, RFF-$\infty$, RRFF-$\infty$, RFF-2, RRFF-2, RFF-3, and RRFF-3 appoximations of $\hat{f}$, plus their pointwise errors. In Advection \RomanNumeral{3}, we use a uniform grid of size 200 for the RFF and RRFF methods, 1000 samples for training the models, and another 1000 samples for testing the models. In Figures \ref{Fig:AD3_test_pred_Gaussian_Student2} and \ref{Fig:AD3_test_pred_Student3}, plots of test examples, RFF-$\infty$, RRFF-$\infty$, RFF-2, RRFF-2, RFF-3, and RRFF-3 predictions of $\hat{f}$, and their pointwise errors are presented.

\begin{figure}[H]
\centering
\subfigure{\includegraphics[width=36mm]{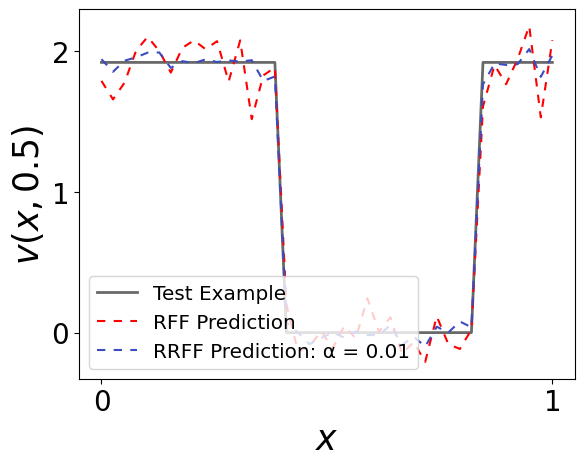}}
\subfigure{\includegraphics[width=37mm]{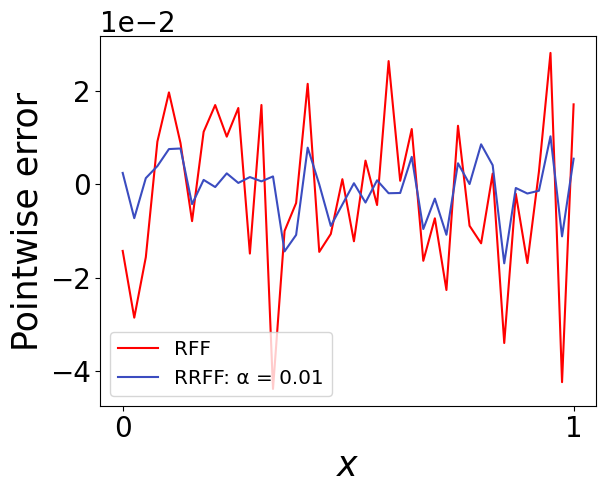}}
\subfigure{\includegraphics[width=36mm]{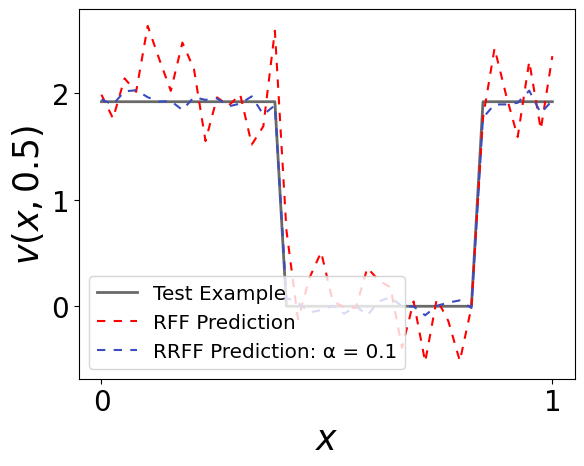}}
\subfigure{\includegraphics[width=39mm]{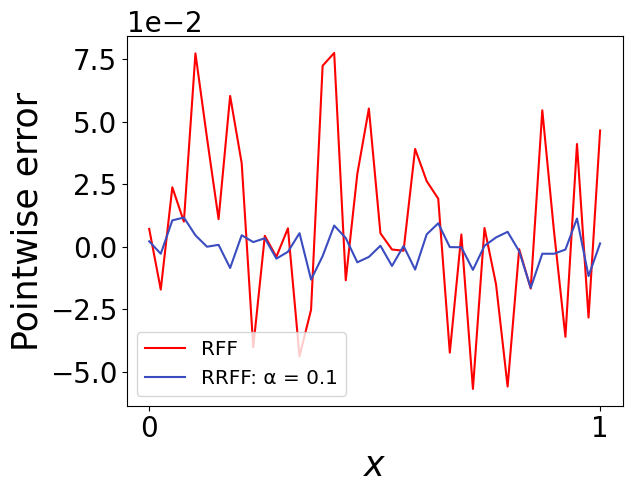}}
\caption{\textbf{Advection \RomanNumeral{1}}: (left to right) (a) test example and RFF-$\infty$ and RRFF-$\infty$ approximations of $\hat{f}$, (b) pointwise errors for RFF-$\infty$ and RRFF-$\infty$ approximations of $\hat{f}$, (c) test example and RFF-2 and RRFF-2 approximations of $\hat{f}$, (d) pointwise errors for RFF-2 and RRFF-2 approximations of $\hat{f}$.}
\label{Fig:AD1_test_pred_Gaussian_Student2}
\end{figure}

\begin{figure}[H]
\centering     
\subfigure{\includegraphics[width=36mm]{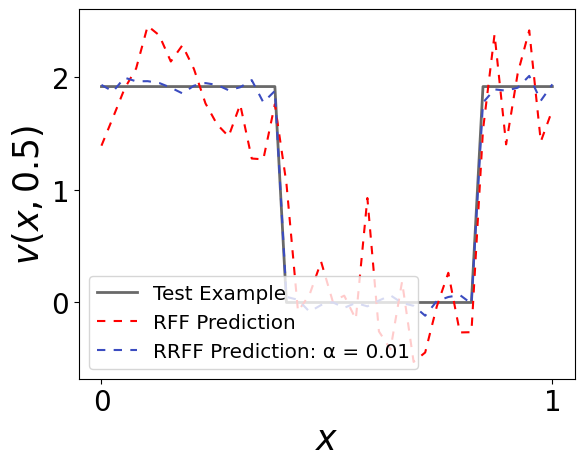}}
\subfigure{\includegraphics[width=39mm]{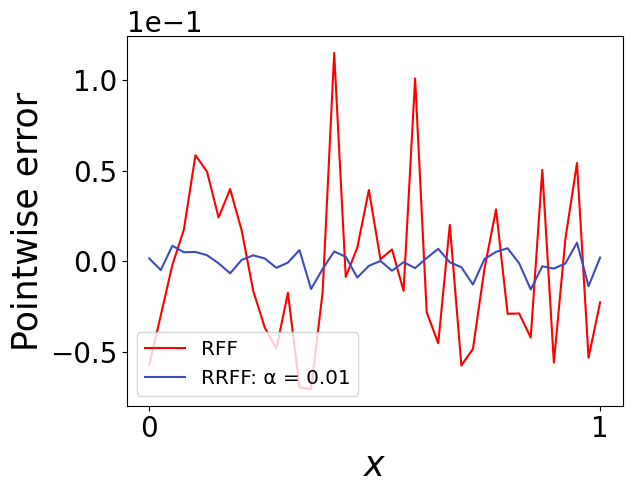}}
\caption{\textbf{Advection \RomanNumeral{1}}: (left to right) (a) test example and RFF-3 and RRFF-3 approximations of $\hat{f}$, (b) pointwise errors for RFF-3 and RRFF-3 approximations of $\hat{f}$.}
\label{Fig:AD1_test_pred_Student3}
\end{figure}

\begin{figure}[H]
\centering
\subfigure{\includegraphics[width=36mm]{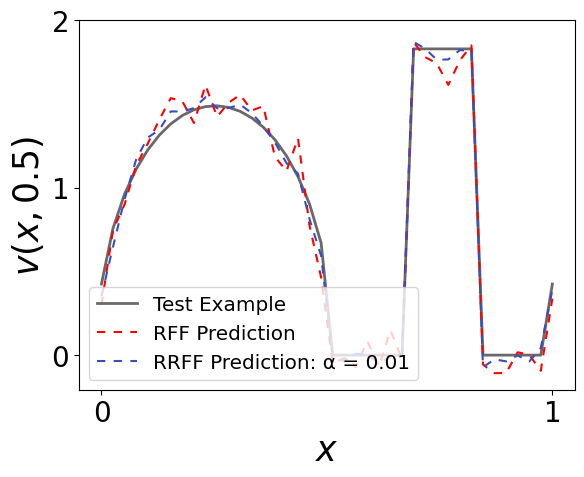}}
\subfigure{\includegraphics[width=37mm]{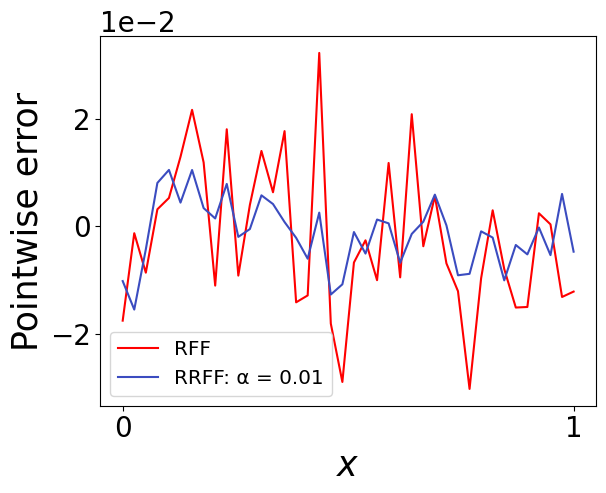}}
\subfigure{\includegraphics[width=36mm]{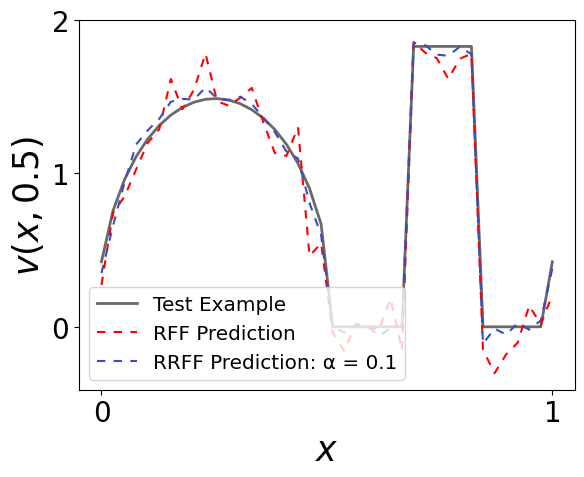}}
\subfigure{\includegraphics[width=37mm]{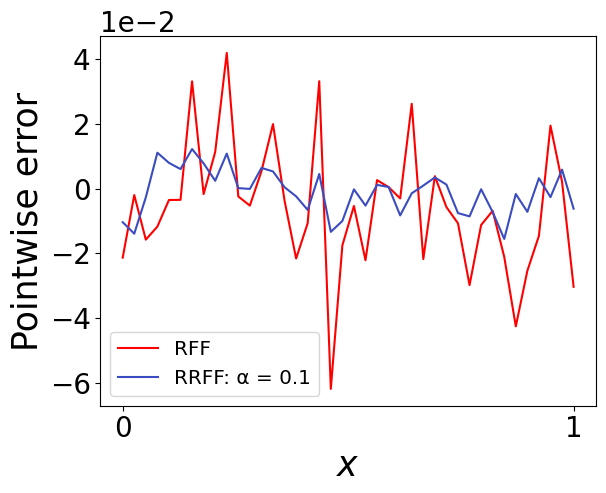}}
\caption{\textbf{Advection \RomanNumeral{2}}: (left to right) (a) test example and RFF-$\infty$ and RRFF-$\infty$ approximations of $\hat{f}$, (b) pointwise errors for RFF-$\infty$ and RRFF-$\infty$ approximations of $\hat{f}$, (c) test example and RFF-2 and RRFF-2 approximations of $\hat{f}$, (d) pointwise errors for RFF-2 and RRFF-2 approximations of $\hat{f}$.}
\label{Fig:AD2_test_pred_Gaussian_Student2}
\end{figure}

\begin{figure}[H]
\centering     
\subfigure{\includegraphics[width=36mm]{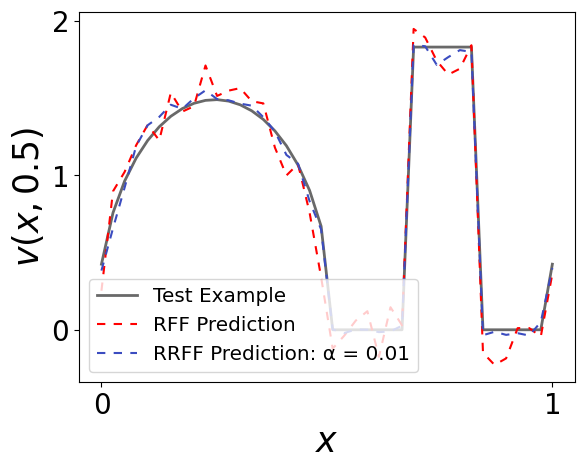}}
\subfigure{\includegraphics[width=37mm]{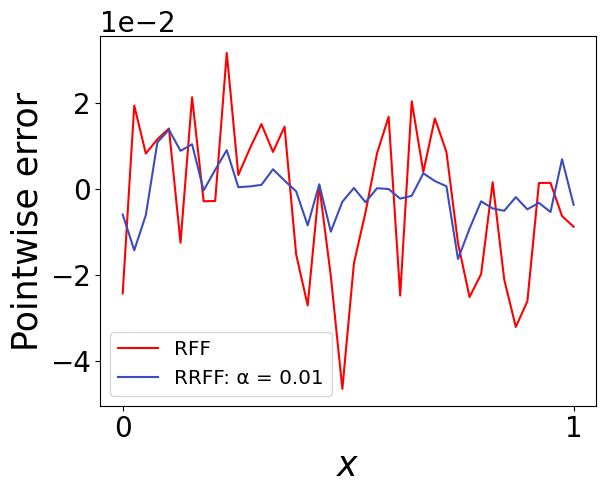}}
\caption{\textbf{Advection \RomanNumeral{2}}: (left to right) (a) test example and RFF-3 and RRFF-3 approximations of $\hat{f}$, (b) pointwise errors for RFF-3 and RRFF-3 approximations of $\hat{f}$.}
\label{Fig:AD2_test_pred_Student3}
\end{figure}

\begin{figure}[H]
\centering
\subfigure{\includegraphics[width=36mm]{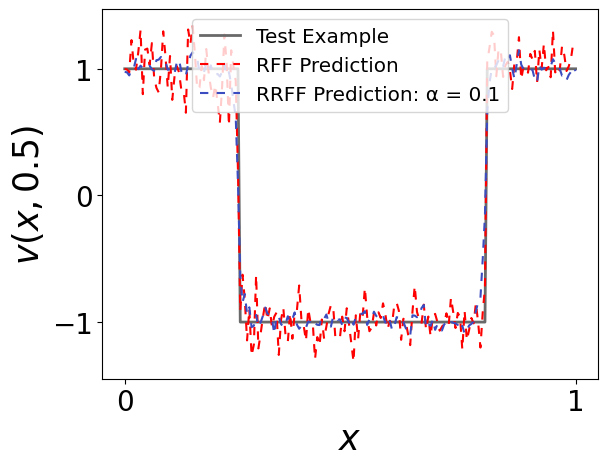}}
\subfigure{\includegraphics[width=37mm]{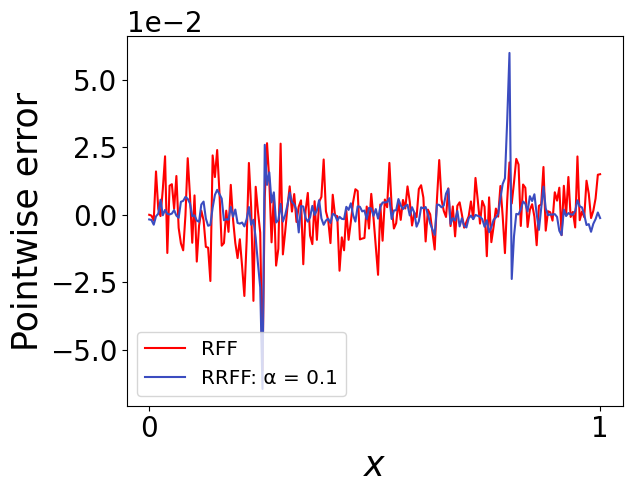}}
\subfigure{\includegraphics[width=36mm]{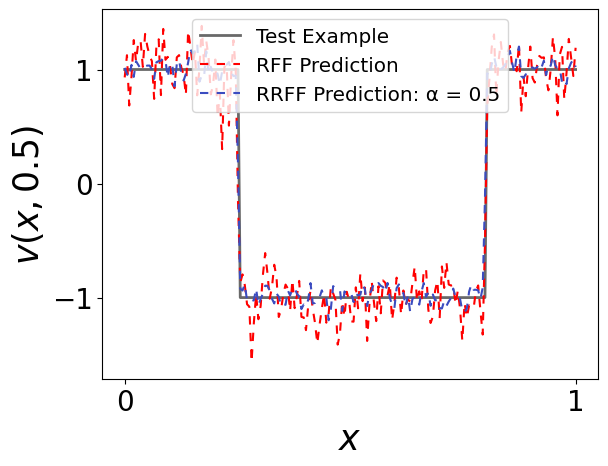}}
\subfigure{\includegraphics[width=36mm]{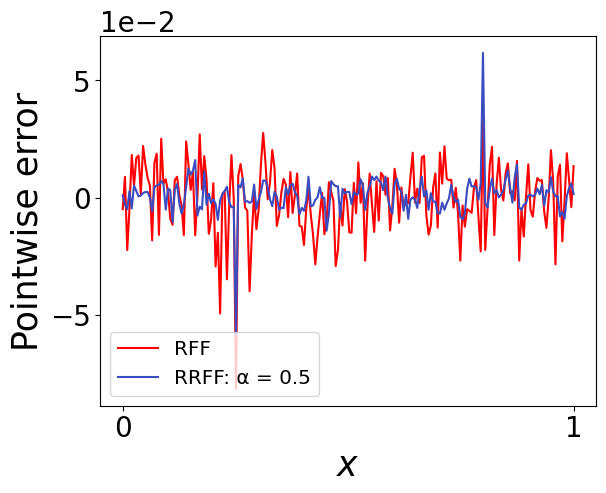}}
\caption{\textbf{Advection \RomanNumeral{3}}: (left to right) (a) test example and RFF-$\infty$ and RRFF-$\infty$ approximations of $\hat{f}$, (b) pointwise errors for RFF-$\infty$ and RRFF-$\infty$ approximations of $\hat{f}$, (c) test example and RFF-2 and RRFF-2 approximations of $\hat{f}$, (d) pointwise errors for RFF-2 and RRFF-2 approximations of $\hat{f}$.}
\label{Fig:AD3_test_pred_Gaussian_Student2}
\end{figure}

\begin{figure}[H]
\centering
\subfigure{\includegraphics[width=36mm]{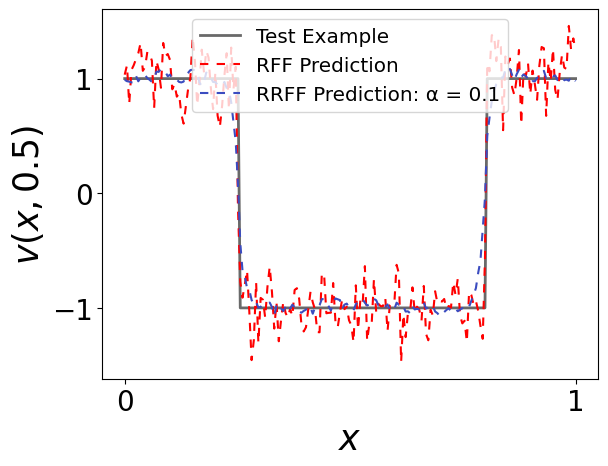}}
\subfigure{\includegraphics[width=36mm]{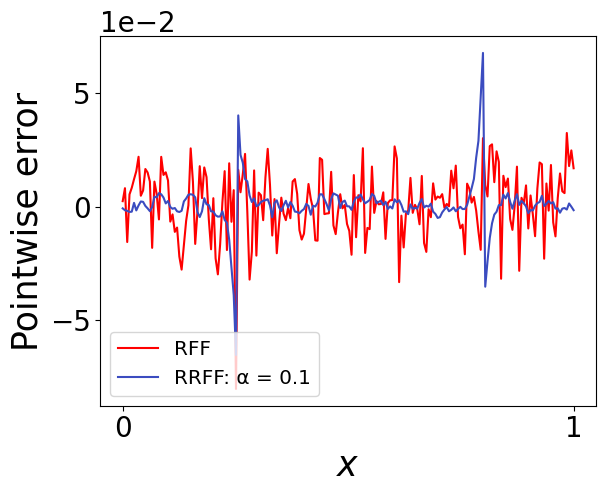}}
\caption{\textbf{Advection \RomanNumeral{3}}: (left to right) (a) test example and RFF-3 and RRFF-3 approximations of $\hat{f}$, (b) pointwise errors for RFF-3 and RRFF-3 approximations of $\hat{f}$.}
\label{Fig:AD3_test_pred_Student3}
\end{figure}

\subsection*{Helmholtz Equation}
We use 10k instances to train the RFF and RRFF models and 30k instances to test the performance. Recall from Section \ref{helmholtz_eq} that a discretized grid of size $101 \times 101$ is used for $u$ and $v$. Figures \ref{Fig:Helmholtz_test_pred_Gaussian}, \ref{Fig:Helmholtz_test_pred_Student2}, and \ref{Fig:Helmholtz_test_pred_Student3} show test examples, RFF-$\infty$, RRFF-$\infty$, RFF-2, RRFF-2, RFF-3, and RRFF-3 predictions of $\hat{f}$, alongside their pointwise errors.

\begin{figure}[H]
\centering     
\subfigure{\includegraphics[width=100.5mm]{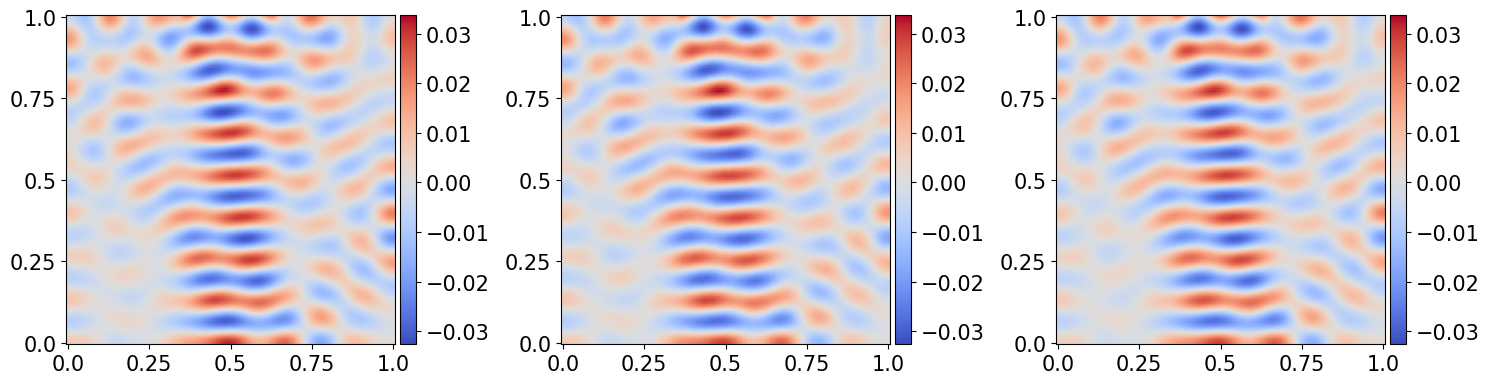}}
\subfigure{\includegraphics[width=62.5mm]{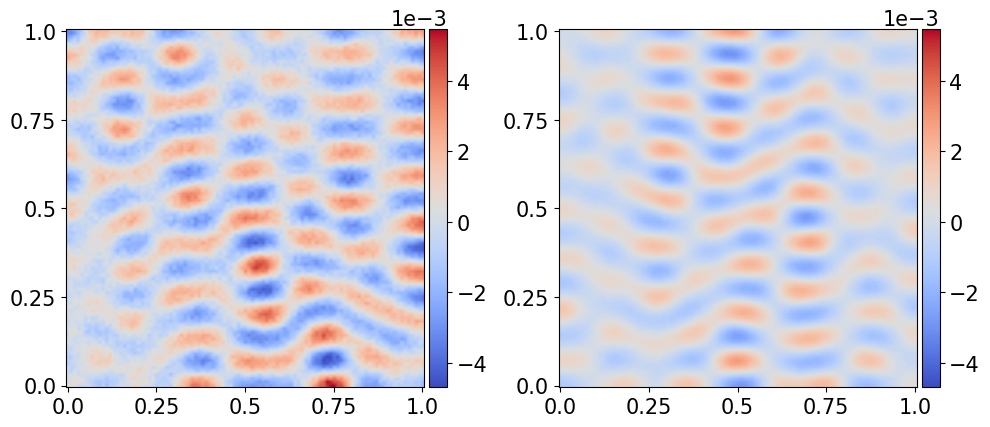}}
\caption{\textbf{Helmholtz Equation}: (left to right) (a) test example, (b) RFF-$\infty$ approximation of $\hat{f}$, (c) RRFF-$\infty$ approximation of $\hat{f}$, (d) pointwise error for RFF-$\infty$ approximation of $\hat{f}$, (e) pointwise error for RRFF-$\infty$ approximation of $\hat{f}$.}
\label{Fig:Helmholtz_test_pred_Gaussian}
\end{figure}

\begin{figure}[H]
\centering     
\subfigure{\includegraphics[width=99.5mm]{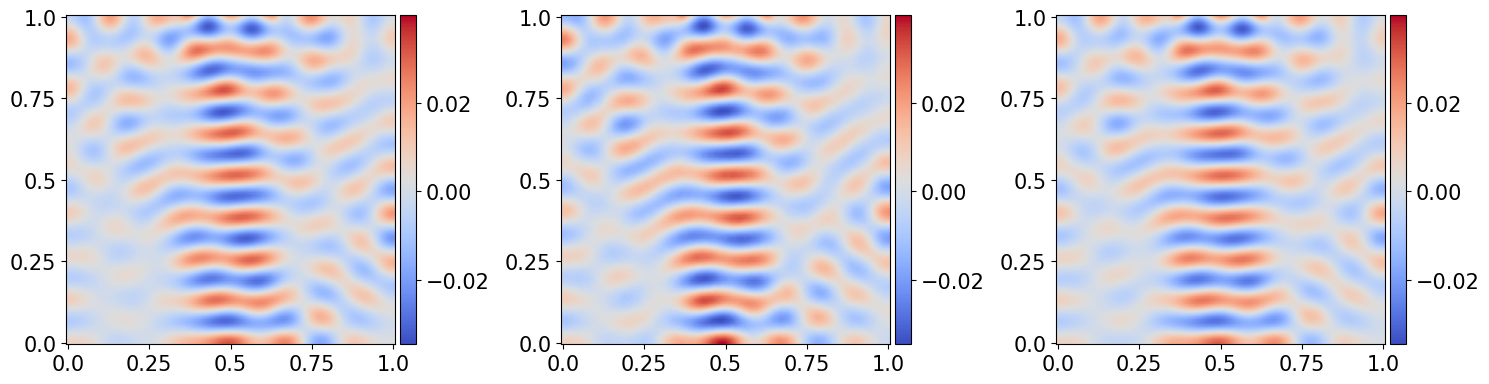}}
\subfigure{\includegraphics[width=63.5mm]{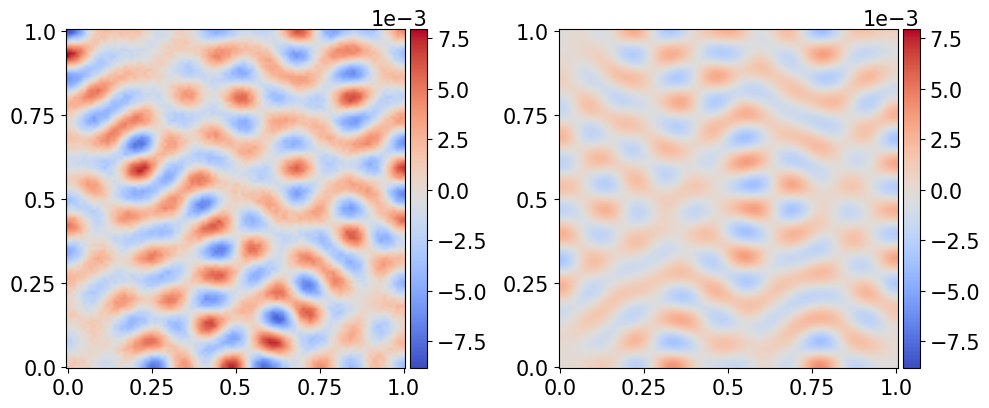}}
\caption{\textbf{Helmholtz Equation}: (left to right) (a) test example, (b) RFF-2 approximation of $\hat{f}$, (c) RRFF-2 approximation of $\hat{f}$, (d) pointwise error for RFF-2 approximation of $\hat{f}$, (e) pointwise error for RRFF-2 approximation of $\hat{f}$.}
\label{Fig:Helmholtz_test_pred_Student2}
\end{figure}

\begin{figure}[H]
\centering     
\subfigure{\includegraphics[width=99mm]{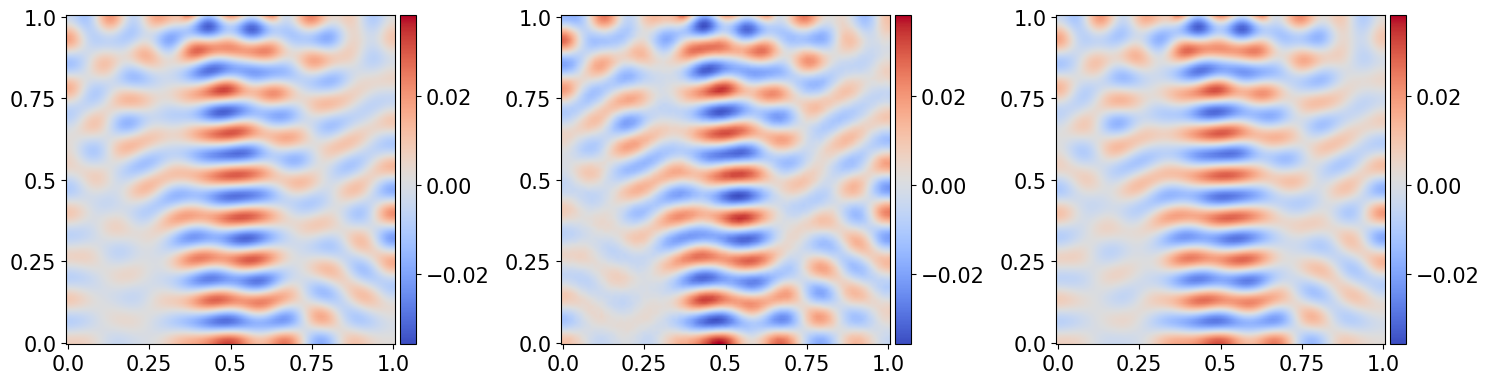}}
\subfigure{\includegraphics[width=64mm]{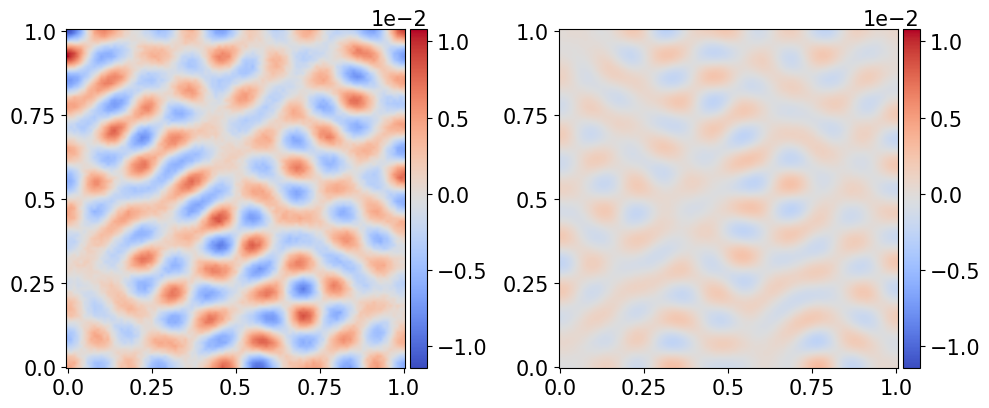}}
\caption{\textbf{Helmholtz Equation}: (left to right) (a) test example, (b) RFF-3 approximation of $\hat{f}$, (c) RRFF-3 approximation of $\hat{f}$, (d) pointwise error for RFF-3 approximation of $\hat{f}$, (e) pointwise error for RRFF-3 approximation of $\hat{f}$.}
\label{Fig:Helmholtz_test_pred_Student3}
\end{figure}

\subsection*{Navier-Stokes Equation}
In the numerical experiments for RFF and RRFF, we use 10k training samples and 30k test samples. From Section \ref{NS_section}, recall that a grid on $[0,2\pi]^2$ of size $64\times 64$ is used. In Figures \ref{Fig:NS_test_pred_Gaussian}, \ref{Fig:NS_test_pred_Student2}, and \ref{Fig:NS_test_pred_Student3}, plots of test examples, RFF-$\infty$, RRFF-$\infty$, RFF-2, RRFF-2, RFF-3, and RRFF-3 approximations of $\hat{f}$, and their pointwise errors are shown.
\begin{figure}[H]
\centering     
\subfigure{\includegraphics[width=100mm]{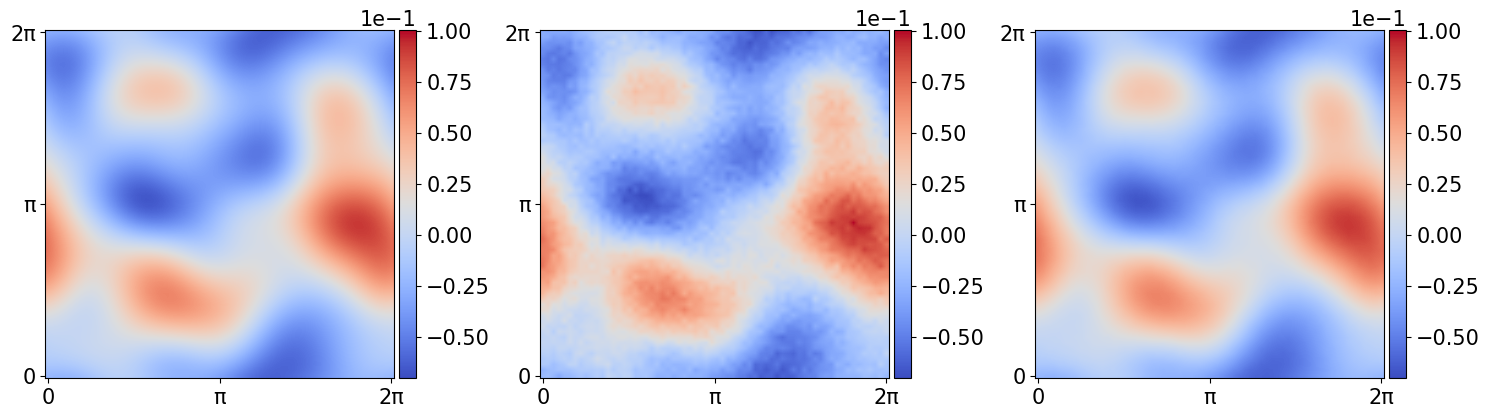}}
\subfigure{\includegraphics[width=63mm]{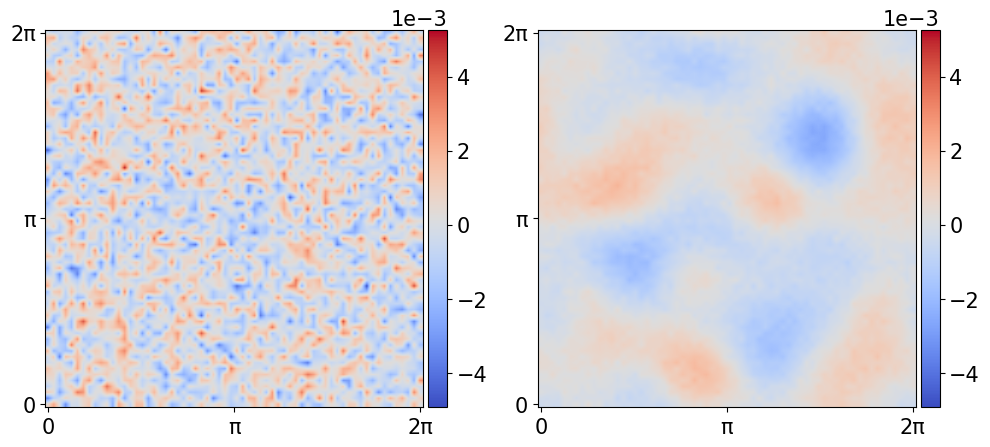}}
\caption{\textbf{Navier-Stokes Equation}: (left to right) (a) test example, (b) RFF-$\infty$ approximation of $\hat{f}$, (c) RRFF-$\infty$ approximation of $\hat{f}$, (d) pointwise error for RFF-$\infty$ approximation of $\hat{f}$, (e) pointwise error for RRFF-$\infty$ approximation of $\hat{f}$.}
\label{Fig:NS_test_pred_Gaussian}
\end{figure}

\begin{figure}[H]
\centering     
\subfigure{\includegraphics[width=99mm]{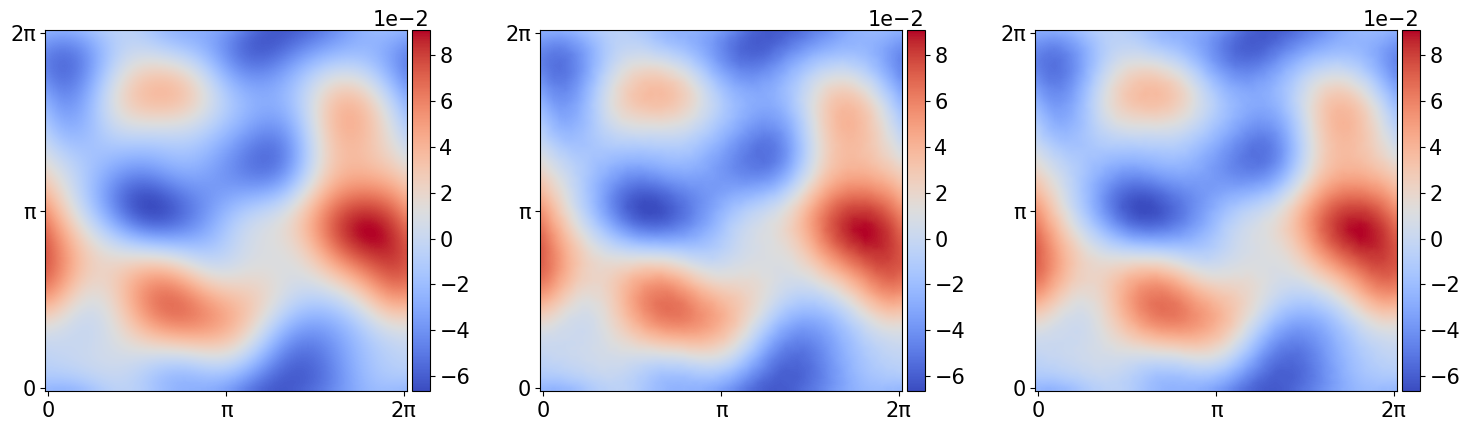}}
\subfigure{\includegraphics[width=64mm]{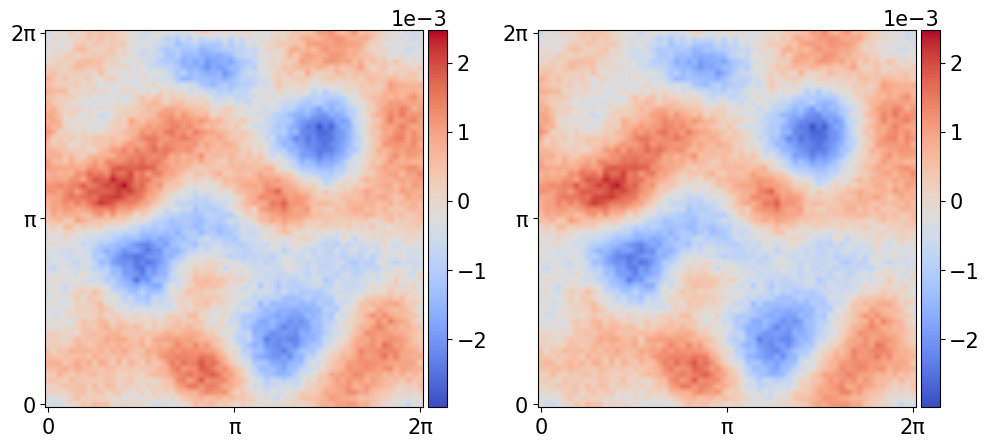}}
\caption{\textbf{Navier-Stokes Equation}: (left to right) (a) test example, (b) RFF-2 approximation of $\hat{f}$, (c) RRFF-2 approximation of $\hat{f}$, (d) pointwise error for RFF-2 approximation of $\hat{f}$, (e) pointwise error for RRFF-2 approximation of $\hat{f}$.}
\label{Fig:NS_test_pred_Student2}
\end{figure}

\begin{figure}[H]
\centering     
\subfigure{\includegraphics[width=99mm]{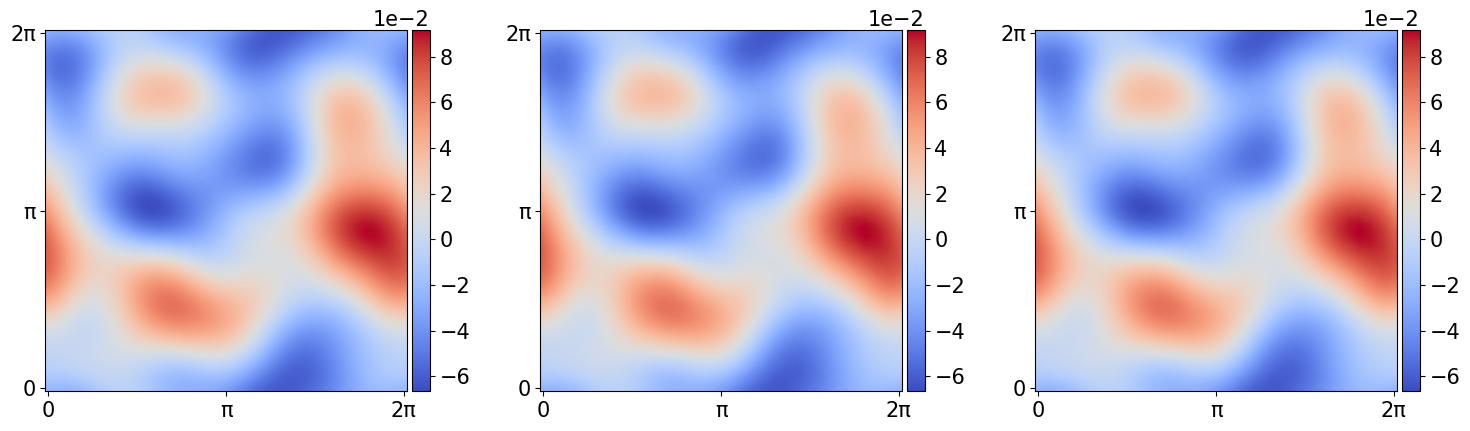}}
\subfigure{\includegraphics[width=64mm]{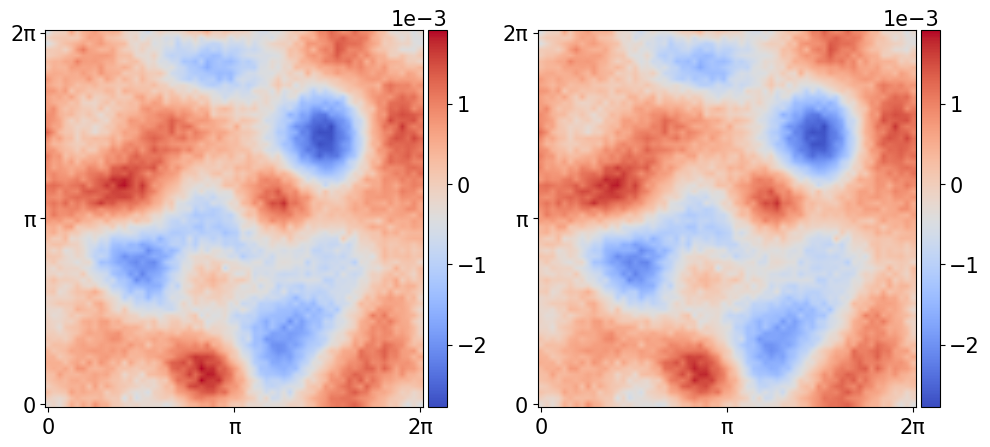}}
\caption{\textbf{Navier-Stokes Equation}: (left to right) (a) test example, (b) RFF-3 approximation of $\hat{f}$, (c) RRFF-3 approximation of $\hat{f}$, (d) pointwise error for RFF-3 approximation of $\hat{f}$, (e) pointwise error for RRFF-3 approximation of $\hat{f}$.}
\label{Fig:NS_test_pred_Student3}
\end{figure}

\subsection*{Structural Mechanics}
The one-dimensional load $u$ is discretized on a grid of size 41, and the two-dimensional von Mises stress field $v$ is discretized on a grid of size $41\times 41$ using radial basis function interpolation. We use 20k samples for training RFF and RRFF and 20k samples for testing its performance. Figures \ref{Fig:Structural_test_pred_Gaussian}, \ref{Fig:Structural_test_pred_Student2}, and \ref{Fig:Structural_test_pred_Student3} present examples of test functions, RFF-$\infty$, RRFF-$\infty$, RFF-2, RRFF-2, RFF-3, and RRFF-3 predictions of $\hat{f}$, plus their pointwise errors.
\begin{figure}[H]
\centering     
\subfigure{\includegraphics[width=98mm]{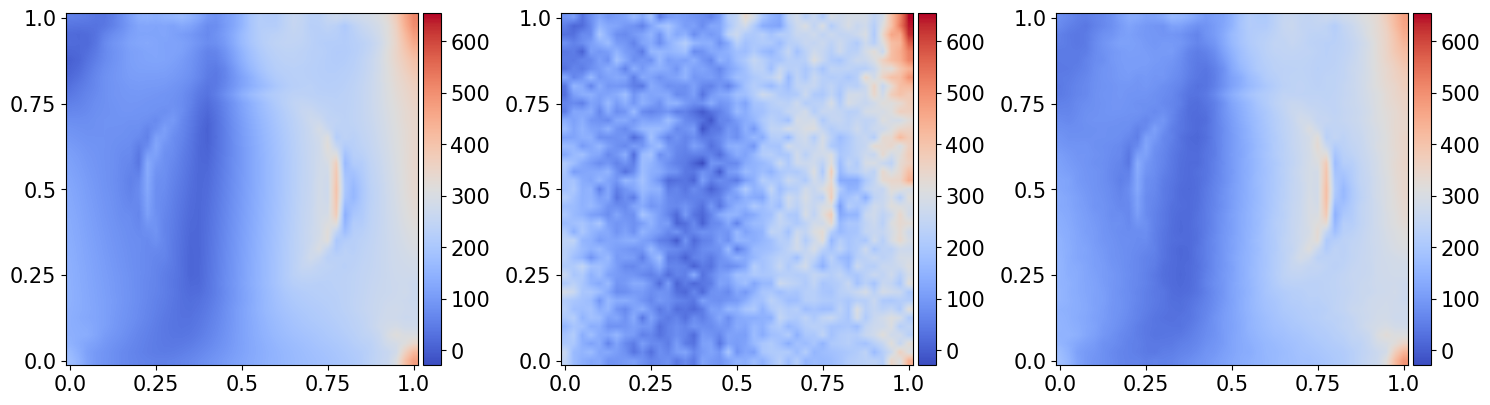}}
\subfigure{\includegraphics[width=65mm]{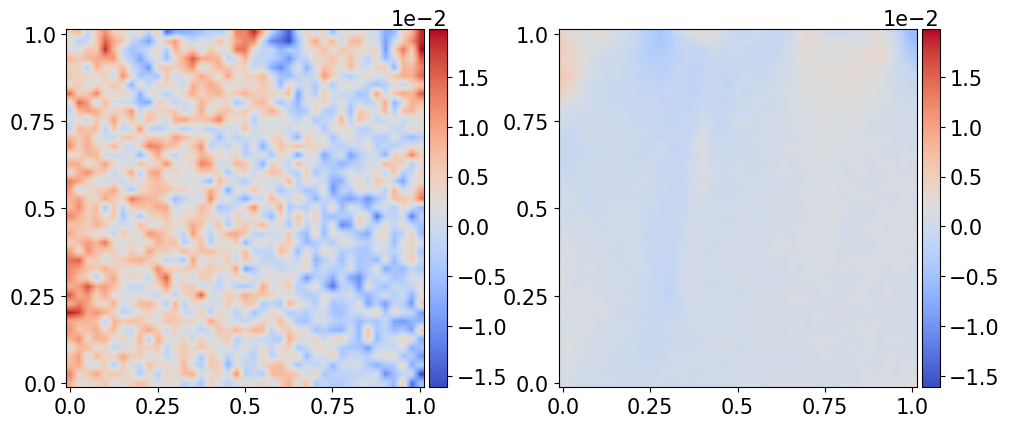}}
\caption{\textbf{Structural Mechanics}: (left to right) (a) test example, (b) RFF-$\infty$ approximation of $\hat{f}$, (c) RRFF-$\infty$ approximation of $\hat{f}$, (d) pointwise error for RFF-$\infty$ approximation of $\hat{f}$, (e) pointwise error for RRFF-$\infty$ approximation of $\hat{f}$.}
\label{Fig:Structural_test_pred_Gaussian}
\end{figure}

\begin{figure}[H]
\centering     
\subfigure{\includegraphics[width=98.5mm]{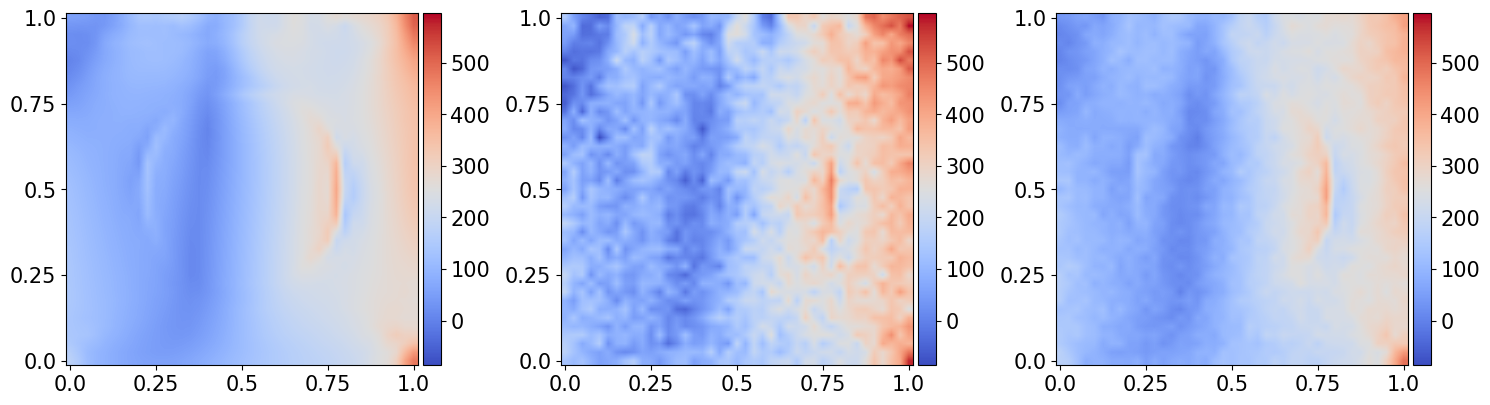}}
\subfigure{\includegraphics[width=64.5mm]{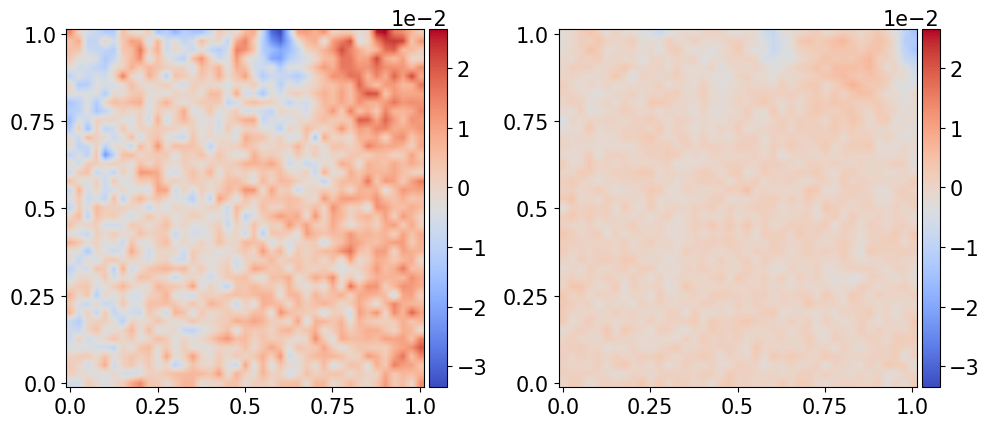}}
\caption{\textbf{Structural Mechanics}: (left to right) (a) test example, (b) RFF-2 approximation of $\hat{f}$, (c) RRFF-2 approximation of $\hat{f}$, (d) pointwise error for RFF-2 approximation of $\hat{f}$, (e) pointwise error for RRFF-2 approximation of $\hat{f}$.}
\label{Fig:Structural_test_pred_Student2}
\end{figure}

\begin{figure}[H]
\centering     
\subfigure{\includegraphics[width=98.5mm]{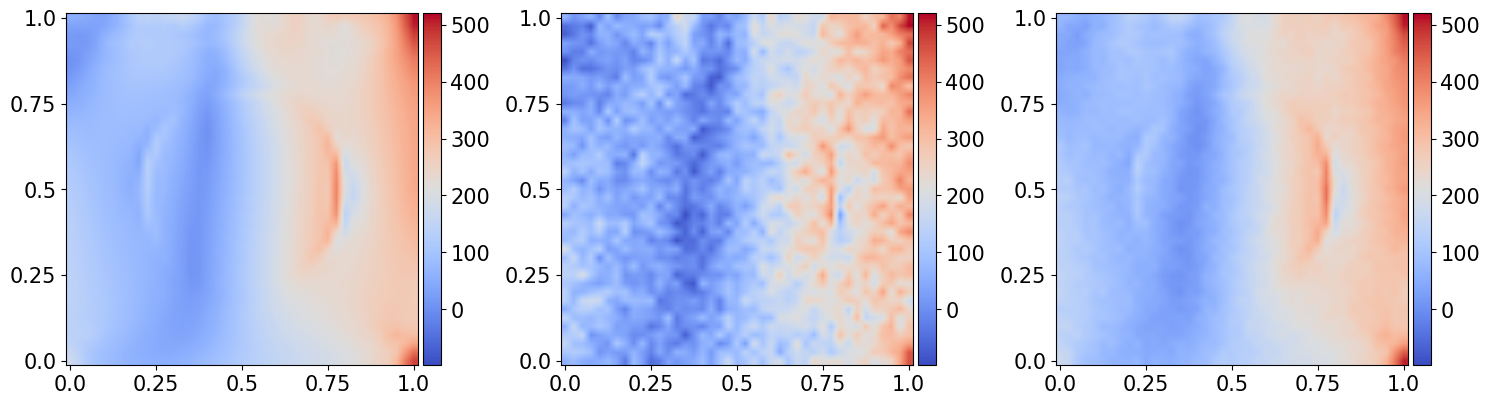}}
\subfigure{\includegraphics[width=64.5mm]{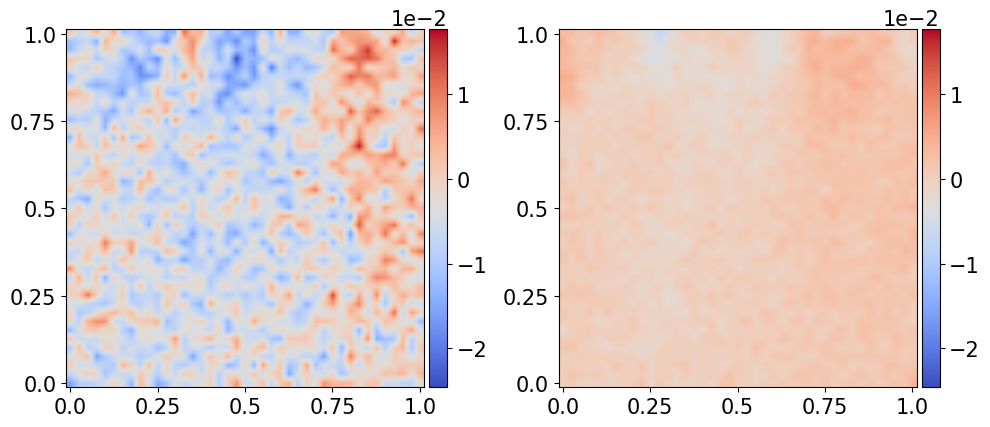}}
\caption{\textbf{Structural Mechanics}: (left to right) (a) test example, (b) RFF-3 approximation of $\hat{f}$, (c) RRFF-3 approximation of $\hat{f}$, (d) pointwise error for RFF-3 approximation of $\hat{f}$, (e) pointwise error for RRFF-3 approximation of $\hat{f}$.}
\label{Fig:Structural_test_pred_Student3}
\end{figure}

\end{document}